\documentclass[12pt]{article}
\usepackage{arxiv}

\usepackage[utf8]{inputenc} 
\usepackage[T1]{fontenc}    
\usepackage[colorlinks,citecolor=blue,urlcolor=blue]{hyperref}
\usepackage{url}            
\usepackage{booktabs}       
\usepackage{amsfonts}       
\usepackage{nicefrac}       
\usepackage{microtype}      
\usepackage{lipsum}		
\usepackage{graphicx}
\usepackage{doi}
\usepackage{float}

\newcommand{\blind}{0}

\newcommand{\online}{1}

\usepackage[colorlinks,citecolor=blue,urlcolor=blue]{hyperref}
\usepackage{graphicx}
\usepackage{amsmath,amssymb,amsthm}
\usepackage[numbers]{natbib}
\usepackage{float}
\usepackage[autostyle]{csquotes}
\usepackage{helvet}
\usepackage{times}
\usepackage{tabulary,booktabs}
\usepackage{enumerate}
\usepackage[title]{appendix}
\usepackage{subcaption}

\let\tilde\widetilde
\let\hat\widehat
\def\Exp{{\ensuremath{\mathbb E}}}

\usepackage[ruled,vlined]{algorithm2e}
\SetKwInOut{Input}{$\triangleright$\ Input}
\SetKwInOut{Parameters}{$\triangleright$\ Parameters}
\SetKwInOut{Output}{$\triangleright$\ Output}
\SetKwComment{Comment}{$\triangleright$\ }{}

\newcommand{\mathbbm}[1]{\text{\usefont{U}{bbm}{m}{n}#1}} 

\SetCommentSty{comsty}

\usepackage[flushleft]{threeparttable}
\usepackage[section]{placeins}
\usepackage{enumitem}
\newtheorem{theorem}{Theorem}[section]
\newtheorem{proposition}[theorem]{Proposition}
\newtheorem{lemma}[theorem]{Lemma}

\newtheorem{remark}{Remark}[section]

\title{Deviance Matrix Factorization}

\author{ Liang Wang \\
	Department of Mathematics and Statistics\\
	Boston University\\
	Boston, MA, 02215 USA \\
	\texttt{leonwang@bu.edu} \\
	\And
	Luis Carvalho \\
	Department of Mathematics and Statistics\\
	Boston University\\
	Boston, MA, 02215 USA \\\
	\texttt{lecarval@math.bu.edu} \\
}

\date{}

\renewcommand{\headeright}{}
\renewcommand{\shorttitle}{Deviance Matrix Factorization}

\hypersetup{
pdftitle={A template for the arxiv style},
pdfsubject={q-bio.NC, q-bio.QM},
pdfauthor={David S.~Hippocampus, Elias D.~Striatum},
pdfkeywords={First keyword, Second keyword, More},
}

\begin{document}
\maketitle

\begin{abstract}
We investigate a general matrix factorization for deviance-based data losses,
extending the ubiquitous singular value decomposition beyond squared error
loss. While similar approaches have been explored before, our method leverages
classical statistical methodology from generalized linear models (GLMs) and
provides an efficient algorithm that is flexible enough to allow for
structural zeros via entry weights. Moreover, by adapting results from
GLM theory, we provide support for these decompositions by (i) showing
strong consistency under the GLM setup, (ii) checking the adequacy of a chosen
exponential family via a generalized Hosmer-Lemeshow test, and (iii)
determining the rank of the decomposition via a maximum eigenvalue gap method.
To further support our findings, we conduct simulation studies to assess
robustness to decomposition assumptions and extensive case studies using
benchmark datasets from image face recognition, natural language processing,
network analysis, and biomedical studies. Our theoretical and empirical
results indicate that the proposed decomposition is more flexible, general,
and robust, and can thus provide improved performance when compared to
similar methods. To facilitate applications, an R package with efficient model
fitting and family and rank determination is also provided.
\end{abstract}


\section{Introduction}
Matrix factorization has a long history in statistics. In general, the problem
is that given a $n$-by-$p$ data matrix $X$ where w.l.o.g. $n \geq p$, one aims to
decompose $X$ into the product of two lower rank matrices that jointly
summarize most of the information in $X$. With the rows of $X$ seen as samples
and the columns as sample conditions or experiments, such decompositions hope
to interpret the information in $X$ in \emph{factors}. For instance, consider
the (thin) singular value decomposition~\citep[SVD;][]{golub13} of $X$, $X =
UDV^\top$, where $U$ is $n$-by-$p$ with orthogonal columns, that is, $U$
belongs to the Stiefel manifold $\mathcal{S}_{n,p}$,
$D = \text{Diag}\{d_1, \ldots, d_p\}$ is a diagonal matrix of order $p$ with
$d_1 \geq \cdots \geq d_p$, and $V \in O(p)$ is orthogonal. The Eckart-Young
theorem~\citep{eckart36} then states that, within rank $q$ decompositions
$\mathcal{V}_{n, p}(q) = \mathbb{R}^{n \times q} \times \mathbb{R}^{p \times q}$,
one solution to
\begin{equation}
\label{eq:svd}
\mathop{\text{argmin}}\limits_{(\Lambda, V) \in \mathcal{V}_{n,p}(q)}
\|X - \Lambda V^\top\|_F = \\
\mathop{\text{argmin}}\limits_{(\Lambda, V) \in \mathcal{V}_{n,p}(q)}
\sum_{i=1}^n \sum_{j=1}^p \big(X_{ij} - (\Lambda V^\top)_{ij}\big)^2
\end{equation}
is given by $\hat{\Lambda} = U_{1:q} \text{Diag}\{d_1, \ldots, d_q\}$, where
$U_{1:q}$ has the first $q$ columns of $U$ and $\hat{V} = V_{1:q}$, that is,
by a rank-truncated version of $X$ according to its SVD. This result has broad
applications in Statistics, the closest one being principal component analysis
(PCA). In this case, regarding rows of $X$ as observations, we can center and
decompose $X$ by solving
\[
(\hat{v}_0, \hat{\Lambda}, \hat{V}) =
\mathop{\text{argmin}}\limits_{v_0 \in \mathbb{R}^p,
(\Lambda, V) \in \{\mathcal{V}_{n,p}(q) : \Lambda^\top \mathbbm{1}_n = 0\}}
\|X - \mathbbm{1}_n v_0^\top - \Lambda V^\top\|_F.
\]
It is known that $\hat{v}_0$ contains column means and
$\hat{\Lambda} \hat{V}^\top$ can be analogously obtained from the SVD of $X -
\mathbbm{1}_n \hat{v}_0^\top$. Here, $\hat{\Lambda}$ corresponds to the first
$q$ principal components of $X$.

The formulation in~\eqref{eq:svd} is based on squared error loss, minimizing
of which can be considered equivalent to maximizing a Gaussian
log-likelihood. However, such a model implicitly assumes that
$X \in \mathbb{R}^{n \times p}$ and often leads to spurious results in more
constrained cases, for example, if
$X \in \mathcal{F}^{n \times p}$ for a field $\mathcal{F}$, when
$\mathcal{F} = \mathbb{N}$, $\mathcal{F} = \{0,1\}$, or even
$\mathcal{F} = \mathbb{R}^+$.
In the case that $X$ stores positive counts, that is, when $\mathcal{F} = \mathbb{N}^{+}$,
one proposed improvement is non-negative matrix
factorization~\citep[NMF;][]{1999nmf}, which solves
\begin{equation}
\mathop{\text{argmin}}\limits_{(\Lambda, V) \in \mathcal{V}_{n,p}(q)}
\sum_{i, j} X_{ij} \log \frac{X_{ij}}{(\Lambda V^\top)_{ij}} -
\big(X_{ij} - (\Lambda V^\top)_{ij}\big)
\label{eq:nmf_loss}
\end{equation}
and is thus based on a Kullback-Leibler (KL) loss or a Poisson log-likelihood
with an identity link.

In the same way that the NMF has been successful in computer vision research,
extending the factorization model with more appropriate likelihood (family)
assumptions and link transformations has generated a series of benchmark
methods in various applications. To name a few, in natural language processing,
the Skip-gram model~\citep{mikolov2013efficient, levy2014neural} can be
considered as a multinomial factorization; the Global Vectors for Word
Representation~\citep[GloVe;][]{mikolov2013efficient} model can be considered as
a log transformed PCA model with entry-wise weight specifications; in network
analysis, the random dot product model~\citep{hoff2002latent,young2007random}
can be considered as a Bernoulli factorization. These generalizations, along
with their popularity in follow-up applications, provided empirical evidence
for the importance of more realistic factorization assumptions.
However, these models are tailored to specific distributions and lack
generality, specially with respect to over-dispersed or heavy-tailed data,
limiting their exploration and broader, more suitable use in modern
applications.

\subsection{Related Work: Exponential Family PCA}
A more general factorization model, exponential family
principal component analysis~\citep[EPCA;][]{2001exppca} proposes to minimize the following distance loss:
\begin{equation}
\label{eq:EPCAloss}
\mathcal{B}(X, \Lambda V^\top) = 2 \sum_{i = 1}^n \sum_{j = 1}^p
B\Big(X_{ij}, \,\, (\Lambda V^\top)_{ij}\Big),
\end{equation}
where $B$ is the Bregman distance, uniquely determined
by the specification of the exponential family likelihood.
EPCA has generated significant interest, for example,
\citep{wedel2001factor,2008bexppca,li2010simple,li2013exponential,2015hpf,2020logf}
focus on extending it using hierarchical structures to improve likelihood
fitness and automatic rank determination. These contributions, however, have
the following drawbacks:
\begin{itemize}
    \item The hierarchical extensions tend to suffer from heavy computational
    costs and identifiability issues (e.g. label
    switching~\citep{stephens2000dealing} and posterior
    collapse~\citep{wang2021posterior}).
    \item The properties of the estimator have not yet been investigated.
    Understanding convergence behaviour even under certain conditions
    will provide stronger statistical assurances for applications.
    \item The model provides the freedom of choosing among exponential
    families without providing sufficient justifications on how to choose the
    appropriate factorization family given the data.
    \item An efficient and/or robust implementation to allow entry-wise
    factorization priority has been wanting from practitioners. Assuming prior
    weights on the factorization can greatly improve model performance; as
    examples, see~\citep{zhang2009weighted} for missing data
    and~\citep{he2019fast} for a face-centered NMF.
\end{itemize}

\subsection{Our Contributions}
To address the gaps in EPCA, we here revisit this factorization task under a
generalized linear model~\citep[GLM;][]{mccullagh89} framework:
we assume implicitly that the data entries
$X_{ij} \stackrel{\text{ind}}{\sim} F(\mu_{ij}, \phi_{ij})$ where $F$ belongs
to the exponential family with mean $\Exp (X_{ij})= \mu_{ij}$ and variance
$\text{Var}(X_{ij}) = \phi_{ij} \mathcal{V}(\mu_{ij})$ for a variance function
$\mathcal{V}(\cdot):\mathbb{R} \rightarrow \mathbb{R}$. More specifically, $F$
specifies a density/pmf $f$ of $X_{ij}$ that can be expressed parametrically as
a function of canonical parameters $\theta_{ij}$ and dispersions $\phi_{ij}$ as
\[
\log f(X_{ij}) = \frac{X_{ij} \theta_{ij} - b(\theta_{ij})}{\phi_{ij}} +
c(X_{ij}, \phi_{ij}),
\]
where the cumulant function $b$ is such that $b'(\theta_{ij}) = \mu_{ij}$ and
$b''(\theta_{ij}) = b''({b'}^{-1}(\mu_{ij})) = \mathcal{V}(\mu_{ij})$.
The terms $\phi_{ij}$ are usually specified as $\phi_{ij} = \phi/w_{ij}$ where
$\phi$ is now a single dispersion parameter and $w_{ij}$ are known weights. 
Good references to exponential family, GLM and factor model are available correspondingly in 
the well-known textbook \citep{mccullagh89} and Chapter 12 of\citep{bishop2006pattern}.

As discussed above, we relate the
entry means $\mu_{ij}$ to factors $\Lambda$ and $V$ by defining the linear
predictor $\eta_{ij} = (\Lambda V^\top)_{ij}$ and by adopting a link function $g$
such that $\eta_{ij} = g(\mu_{ij})$. That is, we assume element-wisely
\begin{equation}
    g(\Exp(X)) = \Lambda V^\top.
    \label{eq:review_setup}
\end{equation}
If $\ell(X_{ij}; \eta_{ij})$ is the
log-likelihood under $F$ for data entry $X_{ij}$ and parameter $\eta_{ij}$, we
can then obtain the factorization $(\Lambda, V)$ by minimizing the 
the negative likelihood function $L(X_{ij}, \eta_{ij}) = \ell(X_{ij}; X_{ij}) -
\ell(X_{ij}; \eta_{ij})$, i.e. the half-deviance. With the most commonly used members of the exponential
family provided in Table~\ref{tab:loss}, the overall loss is then
\begin{equation}
\label{eq:loss}
\mathcal{L}(X, \Lambda V^\top) = \sum_{i = 1}^n \sum_{j = 1}^p
w_{ij} L\Big(X_{ij}, \,\, g^{-1}\big((\Lambda V^\top)_{ij}\big)\Big),
\end{equation}
where we made the dependency on dispersion weights explicit.
To emphasize these connections to GLM methodology, we denote our method as
\emph{deviance matrix factorization} (DMF).

While the implicit assumption of a data generative process might seem to be
too restrictive, here it only serves to define the data loss $L$
in~\eqref{eq:loss}. In practice, as we show here, the factorization only
really requires the specification of $L$ up to the second moments of the
log-likelihood, as usual in GLMs under quasi-likelihood
assumptions~\citep{wedderburn1974quasi}. In this way, our proposed approach
inherits robustness to model misspecification.

 While we regard the methods we discuss here as descriptive, similarly to PCA, there are close connections to factor
analysis, which, in turn, aims at inferring covariance structure. We benefited
from recent advances in this field, under a Gaussian error assumption, to
investigate factorization consistency~\citep{2009pcaconsis,2014gpca} and
determine the factorization rank $q$~\citep{2010onrank,2013frank,2020frank}.
Here we extend these results and propose a maximum eigengap rule to specify
$q$. Compared to related EPCA methods, our methods are simpler and more
efficient while achieving similar statistical performance. To summarize, as we show later, our DMF enjoys the following advantages:
\begin{itemize}
    \item To appropriately model a wider variety of real world data, we extended
    PCA and NMF to a factorization under a general exponential family assumption. 
    \item To avoid identifiability issues, we constraint the factors~$\Lambda$
    and~$V$ to be orthogonal frames and thus unique, yielding a more robust
    and efficient algorithm.
    \item Under some extra conditions, we extend GLM
    consistency~\citep{1999glmconsistency} to DMF factors.
    \item We adapt a generalized Hosmer-Lemeshow test~\citep{2020ghl} to
    determine the adequacy of a factorization family/link choice.
    \item By specifying a link function $g$ and entry weights $w_{ij}$, we can
    better accommodate factorizations for various applications such as matrix
    completion and residual boosting\citep{guillaumin2009tagprop,kalayeh2014nmf}.
\end{itemize}

\subsection{Organization of the Paper}
We start in Section~\ref{sec:dmf} by formally introducing the deviance matrix
factorization and discussing how we tackle identifiability issues and
centering. Next, in Section~\ref{sec:estimation}, we devise a flexible and
robust algorithm to produce DMFs.
In Section~\ref{sec:guarantees} we offer theoretical guarantees on issues that
would lead to spurious results if left unanswered: we first show that the DMF
exists and is consistent, and then discuss how to determine the factorization
rank and test the adequacy of family and link choices.
In Section~\ref{sec:studies} we supplement these developments with
extensive simulations and case studies using benchmark datasets from many
fields. As we show, the proposed methods yield good results and better
performance than traditional methods in practice, with real-world datasets
favoring over-dispersed distributions such as the negative binomial.
Finally, Section~\ref{sec:conclusion} concludes with a discussion and offers
directions for future work.

\section{Deviance Matrix Factorization}
\label{sec:dmf}
In this section, we formally introduce our factorization model along with some
details on how we address the identifiability and incorporate prior
decomposition structure.

In a connection to PCA can be considered as a model assuming $\Exp(X) = \Lambda V^\top$ with $X$ following multivariate Gaussian distribution,  our deviance matrix factorization assumes 
$g(\Exp(X)) = \Lambda V^\top$ with $X$ following natural exponential family distribution. 
Borrowing the notations from generalized linear
models~\citep{mccullagh89}, our goal then is to decompose $X$ as an entry-wise
function of $\Lambda V^\top$, with approximation error measured by the loss
in~\eqref{eq:loss}. That is, we seek the DMF
\begin{equation}
\label{eq:dmf}
(\hat{\Lambda}, \hat{V}) =
\mathop{\text{argmin}}\limits_{(\Lambda, V) \in \mathcal{V}_{n,p}(q)}
\mathcal{L}(X, \Lambda V^\top).
\end{equation}
For example, using a Gaussian loss, $L(x,y) = (x - y)^2/2$, and an identity
link we recover~\eqref{eq:svd}, while a Poisson loss, equivalent here to KL
loss, again with an identity link yields a non-negative matrix decomposition~\eqref{eq:nmf_loss}.
Table~\ref{tab:loss} lists more examples with correspondences
between a distribution in the exponential family and their respective
half-deviance loss functions and canonical links. The link function $g(\mu)$ is given
in the canonical form, but in general can be any smooth and invertible function w.r.t field $\mathcal{F}$.
\begin{table}[H]
\resizebox{.95 \textwidth}{!}{
\begin{threeparttable}
\caption{\label{tab:loss} Distributions and their respective canonical links
and half-deviance loss functions.}
\begin{tabular}{lcccc} \toprule
  Distribution $F$  & Domain  &Link $g(\mu)$ & Variance function $\mathcal{V}(\mu)$ & Loss $L(x, \mu)$ \\ \midrule
  Gaussian        &$\mathbb{R}$ & $\mu$               & $1$            & $(x - \mu)^2 / 2$ \\
  Poisson         &$\mathbb{N}_+$ & $\log\mu$           & $\mu$          & $x \log(x / \mu) - (x - \mu)$ \\
  Gamma           &$\mathbb{R}_+$ & $\mu^{-1}$          & $\mu^2$        & $(x - \mu) / \mu - \log(x / \mu)$ \\
  Bernoulli       &$[0,1]$ & $\log[\mu/(1-\mu)]$ & $\mu(1 - \mu)$  & $x \log(x / \mu) +$ \\
               &      &                     &                & $(1 - x) \log[(1 - x)/(1 - \mu)]$ \\
  Negative binomial $(\phi)$ & $\mathbb{N}_+$ & $\log[\mu/(\mu + \phi)]$ & $\mu + \mu^2 / \phi$ & $x\log(x/\mu) -$ \\
                & &                     &                & $\phi \log[(x + \phi)/(\mu + \phi)]$ \\
\bottomrule
\end{tabular}
\end{threeparttable}
}
\end{table}

\subsection{Identifiability}
Denote by $GL(q, \mathbb{R})$ the generalized linear space with rank $q$ and
domain $\mathbb{R}$. If $\hat{\Lambda}$ and $\hat{V}$ are a solution
to~\eqref{eq:dmf} then an arbitrary $M \in GL(q, \mathbb{R})$ yields a new
solution with $\tilde{\Lambda} = \hat{\Lambda} M^{-\top}$ and $\tilde{V} =
\hat{V} M$ since then $\hat{\Lambda} \hat{V}^\top = \tilde{\Lambda}
\tilde{V}^\top$. To guarantee identifiability so that, as in the SVD, our
decomposition is unique, we then require any solution $(\Lambda, V)$ is
such that
\begin{enumerate}
\item[i)] $V$ has orthogonal columns, i.e.
$V^\top V = I_q$ and so $V$ belongs to the Stiefel manifold
$\mathcal{S}_{p, q}$;
\item[ii)] $\Lambda$ has scaled pairwise orthogonal columns, that is, $\Lambda = U D$
with $U \in \mathcal{S}_{n, q}$ and $D = \text{Diag}\{d_1, \ldots, d_q\}$ with
$d_1 \geq \cdots \geq d_q$ so that $\Lambda^\top \Lambda = D^2$. We denote this
space for $\Lambda$ as $\tilde{\mathcal{S}}_{n, q}$.
\end{enumerate}
In this case, $\tilde{V} \in \mathcal{S}_{p, q}$ implies that
$\tilde{V}^\top \tilde{V} = M^\top M = I_q$, that is, $M \in O(q)$ is
orthogonal, and so $M^{-\top} = M$. Moreover, if
$\hat{\Lambda}^\top \hat{\Lambda} = \hat{D}^2$ then
$\tilde{\Lambda}^\top \tilde{\Lambda} = M^\top \hat{D}^2 M \doteq \tilde{D}^2$,
that is, $\hat{D}^2 M = M \tilde{D}^2$. But the only way that $M$ commutes
with diagonal matrices is if $M = I_q$ and $\hat{D} = \tilde{D}$ (up to sign
changes, as in the SVD). Thus, we have just shown that
\[
\min_{(\Lambda, V) \in \mathcal{V}_{n,p}(q)}
\mathcal{L}(X, \Lambda V^\top) =
\min_{\Lambda \in \tilde{\mathcal{S}}_{n, q},
V \in \mathcal{S}_{p, q}} \mathcal{L}(X, \Lambda V^\top),
\]
with the solution to the right-hand side being unique. In practice, we can
always identify a unique solution $(\hat{\Lambda},
\hat{V}) \in \tilde{\mathcal{S}}_{n,q} \times \mathcal{S}_{p,q}$ from a
general solution
$(\tilde{\Lambda}, \tilde{V}) \in \mathcal{V}_{n,p}(q)$ by obtaining the
$q$-rank SVD decomposition $\tilde{\Lambda} \tilde{V}^\top = U D \hat{V}^\top$
and setting $\hat{\Lambda} = U D$.
\subsection{Centering}
We can capture a known prior structure
$\Lambda_0 \in \mathbb{R}^{n \times q_0}$ of rank $q_0$ in the decomposition
by solving
\begin{equation}
\label{eq:dmfcentered}
\min_{V_0 \in \mathbb{R}^{p \times q_0},
(\Lambda, V) \in \{\mathcal{V}_{n,p}(q) \,:\, \Lambda^\top \Lambda_0 = 0\}}
\mathcal{L}(X, \Lambda_0 V_0^\top + \Lambda V^\top),
\end{equation}
where the $q_0$ orthogonality constraints $\Lambda^\top \Lambda_0 = 0$ avoid
redundancies. As an example, the most common use of this setup is when
$\Lambda_0 = \mathbf{1}_n$ and $V_0$ serves as the \emph{center} of $X$,
similar to PCA~\citep{hastie09} under Gaussian loss.

We can obtain a solution $(\hat{V}_0, \hat{\Lambda}, \hat{V})$
to~\eqref{eq:dmfcentered} directly from a solution $(\Lambda, V)$
to~\eqref{eq:dmf}: if
$H_0 = \Lambda_0 (\Lambda_0^\top \Lambda_0)^{-1} \Lambda_0^\top$ is the
projection matrix into the space spanned by the columns of $\Lambda_0$, then
\[
\Lambda V^\top = H_0 \Lambda V^\top + (I_n - H_0) \Lambda V^\top
= \Lambda_0 \underbrace{(\Lambda_0^\top \Lambda_0)^{-1} \Lambda_0^\top
\Lambda V^\top}_{\doteq \hat{V}_0} +
\underbrace{(I_n - H_0)\Lambda V^\top}_{\doteq \hat{\Lambda} \hat{V}^\top},
\]
where, as in the previous section,
$\hat{\Lambda} \in \tilde{\mathcal{S}}_{n, q}$ and
$\hat{V} \in \mathcal{S}_{p,q}$ are computed
from the $q$-rank SVD decomposition of $(I_n - H_0)\Lambda V^\top$.
Note that since
$\Lambda_0^\top \hat{\Lambda} \hat{V}^\top = \Lambda_0^\top (I_n - H_0) \Lambda
V^\top = 0$ and $\hat{V}$ is full rank, we have
$\Lambda_0^\top \hat{\Lambda} = 0$ as required.

\section{Decomposition}
\label{sec:estimation}
From the last section we know that we just need to solve the unconstrained
version in~\eqref{eq:dmf} and post-process the solution to identify it even
after centering. Moreover, this post-process ensures an unique deviance factorization. In this
section we describe a method to deviance decompose the input matrix $X$.
Incorporating the entry-wise factorization weights $w_{ij}$, we modify the GLM
re-weighted least squares~\citep[IRLS;][]{mccullagh89} algorithm with a
normalization step to achieve higher numerical stability and to avoid overflows.
The optimization problem
in~\eqref{eq:dmf} is clearly not jointly convex on both factors $\Lambda$ and
$V$, but, as in the SVD, it is convex on each factor individually, so we
alternate Fisher scoring updates for one factor conditional on the other
factor.

\subsection{Deviance Factorization}
Let us first define the linear predictor
$\eta_{ij} = \sum_{k=1}^q \Lambda_{ik} V_{jk}$. From the distribution $F$ we
have the log-likelihood in terms of the canonical parameter $\theta_{ij}$ and
up to a constant as
\[
-L(X_{ij}, \theta_{ij}) = \frac{w_{ij}}{\phi}
\big(X_{ij} \theta_{ij} - b(\theta_{ij})\big),
\]
where $w_{ij}$ are weights and $b$ and $\phi$ are the cumulant function and
dispersion parameter of $F$, respectively. The cumulant allows us to define
$\mu_{ij} = b'(\theta_{ij})$ as the mean of $X_{ij}$ and $\mathcal{V}(\mu_{ij}) =
b''(\theta_{ij})$ as the variance function. Finally, the means and linear
predictors are related by the link function $g$, $\eta_{ij} = g(\mu_{ij})$.

With $\beta$ being an arbitrary entry of either $\Lambda$ or $V$, we then have
\[
\begin{split}
\frac{\partial L(X_{ij}, \mu_{ij})}{\partial \beta} &=
\frac{\partial \eta_{ij}}{\partial \beta}
\frac{\partial \mu_{ij}}{\partial \eta_{ij}}
\Bigg(\frac{\partial \mu_{ij}}{\partial \theta_{ij}}\Bigg)^{-1}
\frac{\partial L(X_{ij}, \theta_{ij})}{\partial \theta_{ij}} \\
&= -\frac{\partial \eta_{ij}}{\partial \beta}
\frac{{g^{-1}}'(\eta_{ij})}{V(\mu_{ij})}
\frac{w_{ij}}{\phi} (X_{ij} - \mu_{ij}) \doteq
-\frac{\partial \eta_{ij}}{\partial \beta} G_{ij}.
\end{split}
\]
For $k = 1, \ldots, q$, we then have the scores
\[
\frac{\partial \mathcal{L}}{\partial \Lambda_{ik}} = -2\sum_{j=1}^p G_{ij} V_{jk}
\quad \text{and} \quad
\frac{\partial \mathcal{L}}{\partial V_{jk}} = -2\sum_{i=1}^n G_{ij} \Lambda_{ik},
\]
that is, $\partial\mathcal{L}/\partial\Lambda = -2 G V$ and
$\partial\mathcal{L}/\partial V = -2 G^\top \Lambda$.
Similarly, taking both $\beta$ and $\gamma$ to be arbitrary entries in either
$\Lambda$ or $V$, we obtain
\[
\Exp\Bigg[\frac{\partial^2 L(X_{ij}, \mu_{ij})}{\partial\beta \partial\gamma}\Bigg] =
\frac{\partial \eta_{ij}}{\partial \beta}
\frac{\partial \eta_{ij}}{\partial \gamma}
\frac{w_{ij}}{\phi}
\frac{{{g^{-1}}'(\eta_{ij})}^2}{V(\mu_{ij})} \doteq
\frac{\partial \eta_{ij}}{\partial \beta}
\frac{\partial \eta_{ij}}{\partial \gamma}
S_{ij}.
\]
Thus, for example, with $\beta = \Lambda_{ik}$ and $\gamma = \Lambda_{lm}$ we
have
\begin{equation*}
\Exp\Bigg[
\frac{\partial^2\mathcal{L}}{\partial \Lambda_{ik} \partial \Lambda_{lm}}
\Bigg] = 2\sum_{j=1}^p V_{jk} V_{jm} S_{ij} \delta_{il},
\end{equation*}
where $\delta$ is the Kronecker delta.
\subsubsection{$\Lambda$-step}
Now, denoting $\boldsymbol{\lambda} = \text{vec}(\Lambda)$, the Fisher scoring
update for $\Lambda$ at the $t$-th iteration is
\begin{equation}
\label{eq:lambdaupdate0}
\boldsymbol{\lambda}^{(t+1)} = \boldsymbol{\lambda}^{(t)}
-\Exp\Bigg[ \frac{\partial^2 \mathcal{L}}{\partial \boldsymbol{\lambda}
\partial \boldsymbol{\lambda}^\top}(\boldsymbol{\lambda}^{(t)}) \Bigg]^{-1}
\frac{\partial{\mathcal{L}}}{\partial\boldsymbol{\lambda}}
(\boldsymbol{\lambda}^{(t)})
\doteq \boldsymbol{\lambda}^{(t)} +
H_{\lambda}(\boldsymbol{\lambda}^{(t)})^{-1}
\text{vec}(G^{(t)} V^{(t)}).
\end{equation}
The Hessian matrix $H_{\lambda}$ is such that
\[
(H_{\lambda})_{ik, lm} =
\sum_j V_{jk} V_{jm} S_{ij} \delta_{il}.
\]
Thus, since
$(H_{\lambda} \boldsymbol{\lambda})_{ik} = \sum_j \sum_m V_{jk} V_{jm} S_{ij}
\Lambda_{im}$, we have that
\[
H_{\lambda} \boldsymbol{\lambda} =
\text{vec}\Big(\big(S \circ (\Lambda V^\top)\big) V\Big),
\]
where $\circ$ is the Hadamard elementwise product.
After pre-multiplication by $H_{\lambda}(\boldsymbol{\lambda}^{(t)})$
in~\eqref{eq:lambdaupdate0} the update then becomes, in matrix form,
\begin{equation}
\label{eq:lambdaupdate}
\big(S^{(t)} \circ (\Lambda^{(t+1)} {V^{(t)}}^\top)\big) V^{(t)} =
\big(S^{(t)} \circ (\Lambda^{(t)} {V^{(t)}}^\top)\big) V^{(t)} +
G^{(t)} V^{(t)}.
\end{equation}
To simplify the notation, let us drop all $t$ iteration superscripts with the
exception of $\Lambda$ terms and define
$D_{i\cdot} \doteq \text{Diag}\{S_{i\cdot}^{(t)}\}$, where $S_{i\cdot}$
denotes the $i$-th row of $S$. Let us also denote by $\lambda_i$ the $i$-th
row of $\Lambda$; we then want $\Lambda^{(t+1)} = [\lambda_1^{(t+1)} \cdots
\lambda_n^{(t+1)}]^\top$. We write~\eqref{eq:lambdaupdate} as
\[
\big((\Lambda^{(t+1)} V^\top) \circ S\big)V =
\begin{bmatrix}
{\lambda_1^{(t+1)}}^\top V^\top D_{1\cdot} V \\
\vdots \\
{\lambda_n^{(t+1)}}^\top V^\top D_{n\cdot} V
\end{bmatrix} =
\big((\Lambda^{(t)} V^\top) \circ S\big)V + GV.
\]
Taking the $i$-th row, $i = 1, \ldots, n$, and transposing we see that we just
need to solve
\begin{equation}
\label{eq:lambdasolve}
V^\top D_{i\cdot} V \lambda_i^{(t+1)} = V^\top D_{i\cdot}
(V\lambda_i^{(t)} + D_{i\cdot}^{-1} G_{i\cdot}) \doteq
V^\top D_{i\cdot} Z_{i\cdot}^{(t)},
\end{equation}
where $Z^{(t)}$ is the working response,
\begin{equation}
\label{eq:workingresponse}
Z_{ij}^{(t)} = \eta_{ij}^{(t)} + \frac{G_{ij}^{(t)}}{S_{ij}^{(t)}} =
\eta_{ij}^{(t)} + \frac{X_{ij} - \mu_{ij}^{(t)}}{{g^{-1}}'(\eta_{ij}^{(t)})}.
\end{equation}
That is, to obtain each $\lambda_i^{(t+1)}$ we just regress $Z_{i\cdot}^{(t)}$
on $V^{(t)}$ with weights $S_{i\cdot}^{(t)}$. Note that since  $Z^{(t)}$ is
independent of $\phi$, we can take the variance weights
$S_{ij}$ to be independent of $\phi$ when finding $\lambda_i^{(t+1)}$:
\begin{equation}
\label{eq:varweight}
S_{ij}^{(t)} = w_{ij} \frac{{g^{-1}}'(\eta_{ij}^{(t)})^2}{V(\mu_{ij}^{(t)})}.
\end{equation}

\subsubsection{$V$-step}
Similarly, the Fisher scoring update for $V$ at the $t$-th iteration is
\begin{equation}
\label{eq:vupdate}
\big(S^{(t)} \circ (\Lambda^{(t)} {V^{(t+1)}}^\top)\big)^\top \Lambda^{(t)} =
\big(S^{(t)} \circ (\Lambda^{(t)} {V^{(t)}}^\top)\big)^\top \Lambda^{(t)} +
{G^{(t)}}^\top \Lambda^{(t)}.
\end{equation}
Denoting $v_j$ as the $j$-th row of $V$ our goal now is to obtain
$V^{(t+1)} = [v_1^{(t+1)} \cdots v_p^{(t+1)}]^\top$.
We again drop the iteration superscripts and denote
$D_{\cdot j} \doteq \text{Diag}\{S_{\cdot j}^{(t)}\}$, where
$S_{\cdot j}$ denotes the $j$-th column of $S$. The update
in~\eqref{eq:vupdate} is then, after transposition,
\[
\Lambda^\top\big(S \circ (\Lambda {V^{(t+1)}}^\top)\big) =
[\Lambda^\top D_{\cdot 1} \Lambda v_1 \cdots
\Lambda^\top D_{\cdot p} \Lambda v_p]
= \Lambda^\top\big(S \circ (\Lambda {V^{(t)}}^\top)\big) + \Lambda^\top G.
\]
But then we can just solve columnwise, for $j = 1, \ldots, p$,
\begin{equation}
\label{eq:vsolve}
\Lambda^\top D_{\cdot j} \Lambda v_j^{(t+1)} =
\Lambda^\top D_{\cdot j} (\Lambda v_j^{(t)} + D_{\cdot j}^{-1} G_{\cdot j})
\doteq \Lambda^\top D_{\cdot j} Z_{\cdot j}^{(t)},
\end{equation}
that is, to obtain $v_j^{(t+1)}$ we just regress $Z_{\cdot j}^{(t)}$ on
$\Lambda^{(t)}$ with weights $S_{\cdot j}^{(t)}$.

\subsubsection{Algorithm DMF}
Having both scoring updates we just need to alternate between $\Lambda$ and
$V$-steps until convergence of the deviance loss in~\eqref{eq:loss}.
Convergence can be assessed by relative change in $\mathcal{L}$
between iterations. To kickstart the procedure, we can first follow another
common procedure in GLM fitting and initialize $\mu_{ij}^{(0)}$ entrywise as
an $F$-perturbed version of $X_{ij}$---e.g, $\mu_{ij}^{(0)} =
X_{ij} + \epsilon$ with a small $\epsilon > 0$ if $F$ is Poisson.
Then, with $\eta_{ij}^{(0)} = g(\mu_{ij}^{(0)})$, we set $V^{(0)}$ as the
first $q$ columns of the transpose of $[\eta_{ij}^{(0)}]_{i,j}$. Note that we
do not need to initialize $\Lambda^{(0)}$ since $\Lambda^{(1)}$ only needs
$V^{(0)}$.
Finally, because for an arbitrary diagonal matrix $D$ of order $q$ we have
\[
\Lambda V^\top = \Lambda D^{-1} D V^\top = (\Lambda D^{-1}) (V D)^\top,
\]
to achieve higher numerical stability and avoid overflows we fix the scale of
the columns of $\Lambda$ and $V$ by normalizing the columns of $V^{(t)}$ before
a $\Lambda$-step and the columns of $\Lambda^{(t+1)}$ before a $V$-step.

We summarize this specialized version of iteratively reweighted least squares
(IRLS) in Algorithm~\ref{algo:dmf}. Given a continuous and differentiable link
$g$, we can thus see that the algorithm converges to at least a stationary
point by Proposition~2.7.1 in~\citep{1999convex}. The general solution to the
stationary point optimization problem is either running the algorithm multiple
times with random initialization or initializing the parameters
$\hat{\Lambda}, \hat{V}$ with the data $X$ information. The GLM framework has
initialization algorithm leverages the data $X$ information, we here adopted
the GLM initialization as the remedy to the stationary point optimization problem.
A concrete implementation as \textsf{R} package
\texttt{dmf} is available in the supplement.

\begin{algorithm}
\Input{$n$-by-$p$ data matrix $X$.}
\Parameters{Distribution $F$ and link function $g$;
decomposition rank $q$.}
\Output{$(\hat{\Lambda}, \hat{V}) \in \mathcal{V}_{n,p}(q)$ solution
to~\eqref{eq:dmf}.}
\Comment{Initialization:}
\For{$i = 1, \ldots, n$ and $j = 1, \ldots, p$}{
  Set $\mu_{ij}^{(0)}$ as an $F$-perturbed version of $X_{ij}$ and
  $\eta_{ij}^{(0)} = g(\mu_{ij}^{(0)})$\;
}
Set $V^{(0)}$ as the first $q$ columns of
$[\eta_{ij}^{(0)}]_{i=1,\ldots,n;j=1,\ldots,p}^\top$\;
\Comment{Cyclic scoring iteration:}
\For{$t = 0, 1, \ldots$ (until convergence)}{
  Given $\eta^{(t)}$ and $\mu^{(t)}$, compute working responses $Z^{(t)}$ as
  in~\eqref{eq:workingresponse} and variance weights $S^{(t)}$ as
  in~\eqref{eq:varweight}\;
  \Comment{$\Lambda$-step:}
  Normalize the columns of $V^{(t)}$\;
  \For(\hfill $\triangleright$ find $\Lambda^{(t+1)}$ in~\eqref{eq:lambdaupdate}){$i = 1, \ldots, n$}{
    Regress $Z_{i\cdot}^{(t)} \sim V^{(t)}$ with weights $S_{i\cdot}^{(t)}$
    to obtain $\lambda_i^{(t+1)}$ in~\eqref{eq:lambdasolve}\;
  }
  \Comment{$V$-step:}
  Normalize the columns of $\Lambda^{(t + 1)}$\;
  \For(\hfill $\triangleright$ find $V^{(t+1)}$ in~\eqref{eq:vupdate}){$j = 1, \ldots, p$}{
    Regress $Z_{\cdot j}^{(t)} \sim \Lambda^{(t+1)}$ with weights $S_{\cdot j}^{(t)}$
    to obtain $v_j^{(t+1)}$ in~\eqref{eq:vsolve}\;
  }
  \For(\hfill $\triangleright$ update $\eta$ and $\mu$){$i = 1, \ldots, n$ and $j = 1, \ldots, p$}{
    Set $\eta_{ij}^{(t+1)} = \sum_{k=1}^q \Lambda_{ik}^{(t+1)} V_{jk}^{(t+1)}$
    and $\mu_{ij}^{(t+1)} = g^{-1}(\eta_{ij}^{(t+1)})$\;
  }
}
\caption{Deviance matrix factorization (DMF)}
\label{algo:dmf}
\end{algorithm}

\section{Theoretical and Practical Guarantees}
\label{sec:guarantees}
While Algorithm~\ref{algo:dmf} enables us to find deviance matrix
factorizations, there are still two issues to tackle: finding adequate
exponential family distributions in the formulation of the DMF and a suitable,
parsimonious decomposition rank. In this section,we provide solutions for these two problems, along with theoretical
guarantees for the good performance of the DMF. We start by extending the
consistency of GLM estimators to the DMF, and then discuss methods to
determine factorization family and rank next.

\subsection{Strong Consistency of DMF Estimator}
In this section, we justify our DMF estimator $\hat{\eta} = \hat{\Lambda}
\hat{V}^\top$ fitted using Algorithm~\ref{algo:dmf} as a consistent estimator
of $\Lambda V^\top$ under mild conditions. Our setup is motivated by the setup
in~\citep{2010onrank} and \citep{1983factorconsistency} which assume fixed
rank $q$ but diverging $p/n \rightarrow c$ for some constant $c > 0$.
The difference is that these references assume i.i.d. Gaussian errors
$X_{ij} = \Lambda_i^\top V_j +  A_i \epsilon B_j$ for
deterministic matrices $A \in \mathbb{R}^{n \times n}$ and $B \in
\mathbb{R}^{p\times p }$. Here we assume that the error is non-Gaussian and
not necessarily additive on the $\eta \equiv \Lambda V^\top$.
To simplify the notation from now on
we denote $m(\cdot) = g^{-1}(\cdot)$, so $\mu_{ij} = \Exp(X_{ij})= m\big((\Lambda V^\top)_{ij}\big)$.
\begin{theorem}
\label{thm:consistency}
Let $\lambda_{l,n, p}(\mathcal{M})$ to be the $l$-th largest eigenvalue of matrix $\mathcal{M} \in \mathbb{R}^{n\times p}$. 
$\lambda_{l,n, p}(\mathcal{M})$ can take limits of $n,p$ since the matrix input $\mathcal{M}$ is of dimension $n\times p$. We abbreviated the notation as $\lambda_{i}(\mathcal{M}) \equiv \lambda_{i,n, p}(\mathcal{M})$. Under the following conditions:
\begin{enumerate}
    \item[C1.] $\lim\limits_{n,p \rightarrow \infty}\lambda_{q}((\Lambda^\top \Lambda) V^\top V) = \infty$
        \item[C2.] $e_{ij} = X_{ij} - \mu_{ij}$ is defined as the fitted residual. 
    For any random variables $c_{ij}$ satisfying $\sum_{j=1}^p c^2_{ij}<\infty$ for fix $i$ and $\sum_{i=1}^n c^2_{ij}<\infty$  for fix $j$, the fitted residual $c_{ij}e_{ij}$  are such that:
    \begin{itemize}
        \item $\sum_{j=1}^{p} c_{ij} e_{ij}$ converges a.s as $p\rightarrow \infty$ $\forall i = 1, \cdots, n$
        \item $\sum_{i=1}^{n} c_{ij} e_{ij}$ converges a.s as $n\rightarrow \infty$ $\forall j = 1, \cdots, p$
    \end{itemize}
    \item[C3.] With $m'(t) = \frac{d m (t)}{d t} > 0$, $m$ is continuously differentiable.
    \item[C4.] $\sup_{i \leq n}\|\Lambda_{i}\|_2 < \infty$ and
    $\sup_{j \leq p} \|V_{j}\|_2 < \infty$ for every $n$ and $p$
\end{enumerate}
The DMF estimator is strongly consistent in the sense that :
\[
\hat{\Lambda} \xrightarrow{a.s.} \Lambda, \hat{V} \xrightarrow{a.s.} V  \text{ as } n,p\rightarrow \infty. 
\]
where $\Lambda$ and $V$ are defined as infinite dimensional matrices.
\end{theorem}

\begin{remark}
Notice that condition~C1
restricts the significance of the principal components and when we impose
the identifiability of DMF, we have the condition to be simplified into $\lim\limits_{n,p \rightarrow
\infty} \lambda_{q}((\Lambda^\top \Lambda) V^\top V) =
\lim\limits_{n,p \rightarrow \infty}\lambda_{q}({\Lambda}^\top \Lambda) =
\infty$. Similar conditions have been frequently used
in~\citep{1983factorconsistency,2010onrank,2020frank}.
\\
Condition~C2 puts restrictions on the magnitude
of factorization noise. For instance, a similar condition, but with $c_i$
being deterministic, has been assumed in~\citep{1999glmconsistency}.
It essentially implies that we need a careful balance
on the significance of the principal components such that condition~C1 is
satisfied and the variance does not grow unbounded.
\\
Conditions~C3 and~C4 are additional assumptions required to compensate the
non-linearity imposed by GLM type link function $g(\cdot)$. A similar
condition is also assumed in~\citep{1999glmconsistency}.
Compared to condition~C1, condition~C4 is asking for non-dominating principal
component $\Lambda$ and projection matrix $V$.
\end{remark}

\subsection{Factorization family determination}
\label{sub:family}
Due to the setup of matrix factorization, the
number of parameters in the DMF will increase according to the sample size.
This has prevented us from directly applying the asymptotic distribution of
$\chi^2$ or deviance statistics to assess goodness of fit, as in regular GLMs.
With degrees of freedom specified from stratified residual groups,
\citep{1997family} first generalized the Hosmer-Lemeshow test to other GLM
families by deriving its asymptotic distribution, but such a generalized
testing statistic requires determining a kernel bandwidth for variance
estimation. To overcome bandwidth selection and the curse of dimensionality,
\citep{2020ghl} provided a more straightforward generalized Hosmer-Lemeshow
test statistic with an empirical power simulation study to justify its performance.
Although these test statistics are proposed within an usual GLM regression
framework, we have been able to justify their use of a naive generalized
Hosmer-Lemeshow test within the context of our matrix decomposition.
The generalized Hosmer-Lemeshow test is based upon the idea of grouping the
model fitted values $\hat{\eta}_{ij}$ according to their magnitude and
standardizing their summation into a chi-square
statistic~\citep{stute2002model}. Consequently, to
ensure we have enough normalized samples to judge for normality, the general
guideline is choosing the number of groups $G$ such that $G \geq 15$ and to
have each group containing at least ten observations---taking a dataset with 150
observations as an example, we could choose the $k$-th group threshold $u_k$
according to the 15th quantile of $\hat{\eta}$.

We adopt a notation similar to~\citep{stute2002model}, defined below:
\begin{equation}
\label{eq:family_setup}
\begin{aligned}
R_{n,p}^1(u_k) &= \frac{1}{\sqrt{np}}  \sum_{i=1}^n \sum_{j=1}^p
1_{\{\hat{\eta}_{ij} \leq u_k\}} [X_{ij} - m(\hat{\eta}_{ij})],
\text{~for~} k = 1, \ldots, G\\
T_{n,p}^1 &= \{R_{n,p}^1(u_k) - R_{n,p}^1(u_{k-1})\}_{k=2}^G\\
D^1_{n,p}& =  \text{Diag} \Big(\frac{1}{np} \sum_{i=1}^n \sum_{j=1}^p
\mathcal{V}(m(\hat{\eta}_{ij})) 1_{\{u_{k-1} < \hat{\eta}_{ij} \leq u_{k}\} }\Big)\\
s^2_{n,p}(k) &= \sum_{i=1}^n\sum_{j=1}^p 1_{\{u_{k-1}<\eta_{ij}\leq u_{k}\}}
(X_{ij} - m(\eta_{ij}))^2, \text{~for~} k = 2, \ldots, G
\end{aligned}
\end{equation}
Specifically, $R_{n,p}^1(u_k)$ is the empirical residual process given estimator $\hat{\eta}$
and cutoff point $u_k$; based upon $R_{n,p}^1(u_k)$, we define a $G$-long
vector $T_{n,p}^1$ with $k$-th element being the empirical residual w.r.t.
cutoffs $(u_{k-1}, u_k)$; $D^{1}_{n,p}$ is a diagonal matrix with entries being
the empirical variances of $T^1_{n,p}$; $s^2_{n,p}(k)$ is the (summed) variance
for group $k$ under true parameter $\eta$. Moreover, we denote by $R_{n,p},
T_{n,p}, D_{n,p}$ the respective processes under true parameters $\eta_{ij}$,
instead of $\hat{\eta}_{ij}$. $\mathcal{V}(\cdot)$ here is the variance
function as defined in Table \ref{tab:loss}. Based on this notation, we propose Theorem~\ref{thm:family} to assess a
factorization family:
\begin{theorem}
\label{thm:family}
Assume that $F$ is the correct factorization family. Under conditions
(F1)--(F3),
\begin{enumerate}
    \item[F1.] $\hat{\eta} \rightarrow \eta\ a.s.$ element-wise, which can be justified by Theorem~\ref{thm:consistency};
    \item[F2.] For $\delta s_{n,p}(k) >0$, we have $\{T_k\}_{k=2}^G = \{1_{\{u_{k-1}<\eta_{ij} \leq u_k\}}(X_{ij} - m(\eta_{ij}))  \}_{k=2}^G$ satisfies
    $$\frac{1}{s^2_{n,p}(k)} \sum_{i=1}^n \sum_{j=1}^p \Exp[T_k^2 1_{\{|T_k|> \delta s_{n,p}(k)\}}]\xrightarrow{a.s.} 0$$
    for $k = 2, \ldots, G$;
    \item[F3.] $\sigma^2_{T_k}=\phi_{ij} \int_{u_{k-1}}^{u_{k}} \mathcal{V}(X_{ij} |
    \eta_{ij} = t) F(dt) = c_k < \infty$ for $k = 2, \ldots, G$;
\end{enumerate}
and for $u_1 < u_2 < \cdots < u_G$ a partition on the support of $\eta$, we have
\begin{equation} \label{eq:family}
\begin{aligned}
\hat{T}(n,p) &= (T_{n,p}^1)^\top (D^1_{n,p})^{-1} T_{n,p}^1 \xrightarrow[n,p\rightarrow \infty]{d} \chi^2_{G-1}
\end{aligned}
\end{equation}
\end{theorem}
\begin{remark}
Condition~F1 is satisfied by following conditions C1--C4 in our
Theorem~\ref{thm:consistency}, condition~F2 is the Lindeberg CLT condition
required for asymptotic normality, and condition~F3 is the condition required for
a strong law of large numbers, which together with Slutsky's theorem justifies
the convergence in distribution.
\end{remark}

The proof is available in the supplement. In comparison to the Hosmer-Lemeshow
test \citep{hosmer1980goodness} where $\hat{T}$ can be asymptotically approxmaited by
a $\chi^2_{G-2}$ distribution, our test statistic is a cruder estimator with condition~F1 assumed to simplify the
derivation. In comparison to~\citep{2020ghl}, our method coincides with their
naive estimator but we further justify from Theorem~\ref{thm:family} the use
of the naive estimator in large sample settings. This conclusion is also
consistent with the empirical findings in~\citep{2020ghl} that differences
between the naive estimator and the more rigorous generalized Hosmer-Lemeshow
test are only noticeable in rather small samples. Typical datasets for matrix factorization are large
enough so that condition~F1 holds after the convergence of Algorithm~\ref{algo:dmf}.

\subsection{Factorization rank determination}
\label{sub:rank}
Rank determination is a difficult and evolving research field. Similarly
to all other statistical problems, the factorization parameters and
likelihood objective increase with the factorization rank. The solution
typically requires a judicious balance between model complexity
and model fit. Generally, within our DMF framework, there are two main types
of methods we can leverage.

The first method is to use ``out of sample'' cross validation where the ``out
of sample'' concept in achieved by assigning zero weights (and thus
effectively ``holding out'' the observation) in Algorithm~\ref{algo:dmf}.
To improve computational efficiency,
\citep{2009bicv} proposed block-wise cross validation by assigning zero weights
to the partitioned Gaussian data matrix. \citep{2010nmfrank} compared the performance
between element-wise holdout and block-wise holdout. However,
the analysis in~\citep{1993lmcv} concluded that cross validation in linear
regression (Gaussian case) requires a large (asympotically equal to the sample
size) validation set. A more expensive predictive sequential type cross
validation~\citep{1996gquasi} was thus proposed and was proven to be
consistent under GLM setup.
In our experiments, probably due to the consideration of bi-way
correlation~\citep{2014gpca} and a relatively larger hold-out sample size, we
found that the block-wise cross validation method performs better than the
element-wise hold-out method but failed to generalize it to GLM families
except for Gaussian and Poisson.

The second type of rank determination method is based on estimation.
\citep{2013frank} proposed Stein Unbiased Risk Estimators (SURE), which were
proven to be unbiased, under mean squared error, for factor models.
However, the derivation of SURE requires explicit solutions to
$\frac{d \hat{\mu}_{ij}}{d X_{ij}}$, which are not available in closed form
using our iterative Algorithm~\ref{algo:dmf}. Recent results using random matrix
theory have gained deserved attention by providing great empirical and
theoretical justifications on conventional rank determination problems. For
example, \citep{2010onrank} demonstrated that under similar conditions
to~C1--C2 in our Theorem~\ref{thm:consistency}, the conventional ``elbow''
rule is justified by the divergence of maximum eigengap on the true rank
threshold. With additional assumptions, \citep{2020frank} improved the
consistency of this rank determination method by replacing the covariance
eigengap with a correlation eigenvalue threshold. Although we have been able
to justify the use of both~\citep{2010onrank} and~\citep{2020frank}, we hereby
adopt the covariance maximum eigengap method of~\citep{2010onrank} because its
conditions permit a more general setup and are closely related to the
consistency conditions~C1--C2 assumed in our Theorem~\ref{thm:consistency}.
We also compare the ACT method with the maximum eigenvalue gap method with
simulated studies satisfying different setups to verify the generality of the
former method in Section~\ref{sec:studies}.

\begin{proposition}
\label{prop:rank}
In addition to  conditions C1--C2, if the following additional conditions holds
\begin{enumerate}
    \item[R1.] The function $m''(\cdot) (m\circ g'(\cdot))$ is continuous with respect to the function input.
    \item[R2.] The fitting error $e_{ij} = X_{ij} - \mu_{ij}$ has finite fourth
    moment: $\Exp[e^4_{ij}]< \infty$
    \end{enumerate}
then for moderate $n,p$, there exist $A\in \mathbb{R}^{n\times n}$,
$B\in \mathbb{R}^{p\times p}$ and decomposition error
$\epsilon \in \mathbb{R}^{n\times p}$ such that given DMF full rank estimate $\hat{\eta}$:
\[
\hat{\eta}_{ij} =\eta_{ij} + A_i \epsilon B_j
\]
where $A, B ,\epsilon$ satisfy Assumption~2 in~\citep{2010onrank}.
Thus, as a corollary, we have \citep{2010onrank}'s rank determination theorem:
given the maximum rank $q_{\text{max}} \leq \min\{n,p\}$ and for some
$\delta > 0$, the following rank estimation converges in probability to the
true rank,
\begin{equation}
\label{eq:rankstop}
q(\delta) = \max\{j: \tilde{\lambda}_j( \hat{\eta}^\top \hat{\eta}) -
\tilde{\lambda}_{j+1}( \hat{\eta}^\top \hat{\eta}) \geq \delta, j \leq
q_{\text{max}}\}.
\end{equation}
\end{proposition}
\begin{remark}
Conditions~R1-R2 are fairly mild and satisfied by commonly used GLM link
functions. Also notice that the $\epsilon_{ij}$ in this definition is different
from our previously defined fitting error $e_{ij}$. The relationship between
$e_{ij}$ and $\epsilon_{ij}$ is elaborated in Appendix~\ref{appendix:rankproof}.
\end{remark}

In practice, we need an algorithm to determine $\delta$ and thus output
$\hat{q}$. One obvious solution is to plot the eigenvalue up to $q_{max}$ and
look for the maximum eigengap.
A more straightforward solution is to adopt the original algorithm proposed
in Section~IV of~\citep{2010onrank}. The algorithm simply sets
$q_{\text{max}} = p-5$ and iteratively checks the convergence of the slope of
the local linear regression (and thus automatically determine $\delta$) based
upon five neighboring points. The algorithm is by no means rigorous compared
to the theorem, but is provided to automatically estimate the rank by
leveraging on the convergence information of local linear regression.
To guard against numerical artifacts in the $\delta$ calibration algorithm,
plotting the eigenvalue gap of $\text{Cov}(\hat{\eta})$ is always recommended.

\section{Empirical Studies}
\label{sec:studies}
In this section, we first conduct deviance matrix factorization on simulated
datasets to justify the use of family test (Theorem~\ref{thm:family}) and rank
estimation (Proposition~\ref{prop:rank}). Then, after this initial validation,
we perform real-world data analyses on benchmark datasets from computer
vision, natural language processing, network community detection, and
genomics. All these datasets have known labels/classes and the factorization
results are thus compared through the separability of different classes in the
lower dimensional factorized components. To organize our findings in a
more compact way, we refer to Appendix~\ref{appendix:rank} for the eigenvalue gap
plots that justify the factorization ranks for each of the empirical benchmark
datasets.

\subsection{Family Simulation studies}
To empirically validate the family determination theorem, we assume $X_{ij} \stackrel{\text{iid}}{\sim} F(g^{-1}(\eta_{ij}))$ as before and simulate matrix data to check the significance and power of the 
following hypotheses:
\begin{align*}
H_0: & F = F_0 \text{ and } g = g_0 \\
H_1: & F = F_1 \text { and } g = g_1 \text{ with } F_1 \neq F_0 \text{ or } g_1 \neq g_0 \\
\end{align*}
The first simulation study is designed to validate the significance levels of the generalized Hosmer-Lemeshow test in Theorem~\ref{thm:family}, through which we also check the test sensitivity on a dispersion parameter $\phi$ under alternative families in $H_1$.
The second, more extensive simulation study is conducted to investigate type-II error and power.

\subsubsection{Significance Analysis}
The negative binomial distribution can be considered a generalization of the Poisson distribution since it adds a precision parameter to address additional dispersion in count data. Similarly, the gamma distribution better accounts for dispersed data via its shape parameter when compared to the log-Gaussian distribution.
In logistic regression, the complementary log-log link, when compared to the canonical logit link, is
better suited for data characterized by a heavy tail distribution. We here
demonstrate that factorizations based on (over) dispersed families are in
general more robust and better fit the data when compared to (unit) non-dispersed family
factorizations, as expected.

Given a dispersed exponential factorization family $F_0$ with link $g_0$, we first
simulate data $X \sim F_0(g_0^{-1}(\eta))$ conditional on some $\eta = \Lambda_0
V_0^\top + \Lambda V^\top$. Then we conduct DMF estimation to obtain 1) $\hat{\eta}_0$
using the simulation factorization family $F_0$  with link $g_0$ and 2) $\hat{\eta}_1$ using a comparing
factorization family $F_1$ with link $g_1$. Based upon the fitted $\hat{\eta}_0$ and $\hat{\eta}_1$, we compute
 the test statistics in Eq~\eqref{eq:family}.  To maintain numerical stability after link
function transformations, we simulated data with a center (offset)
$\Lambda_0 V_0^\top$.  To validate the generality of our proposed family test,
we varied the null family $F_0$ across gamma ($\phi =1$), negative binomial ($\phi = 5$), and
binomial ($w = 90$) distributions, while also varying link functions from log,
complementary log-log (cloglog), and logit functions.
Here we highlight the results using $q = 5$, $p = 20$, $n = 1000$,
$\Lambda_{ij} \sim N(1, 0.5/q)$ and $V_{ij} \sim N(0, 0.5/q)$. Simulations
with larger sample sizes and various $q/p$ ratios were also conducted but the
results were similar so we do not report them here for brevity.
Table~\ref{tab:family_simu} summarizes the three simulation cases, listing
simulation (null) and comparison (alternative) distributions and links.
\begin{table}[H]
    \centering
    \begin{threeparttable}
    \caption{Family simulation cases.}
    \begin{tabular}{cccc} \toprule
         Factor & Case 1 & Case 2 & Case 3 \\ \midrule
         Simulation family $F_0$ &  gamma & negative binomial & binomial\\
         Simulation link $g_0$ & log & log  & cloglog \\
         Center $\Lambda_0 V_0^\top$ & $\textbf{0.5}_{n\times p}$ & $\textbf{0.5}_{n\times p}$ & $\textbf{0}_{n\times p}$  \\
         Comparison family $F_1$& gaussian & poisson & binomial \\
         Comparison link $g_1$&  identity & log & logit\\
        \bottomrule
    \end{tabular}
    \label{tab:family_simu}
 \end{threeparttable}
\end{table}
For each of the simulation cases, we examine the test statistics conditional
on $\hat{\eta}_0$ and $\hat{\eta}_1$ from Eq~\eqref{eq:family} for both
simulation and the comparison families. More specifically, we examine: (i)
normal QQ plots of the normalized residuals $T^1_{n,p}
(D^1_{n,p})^{-\frac{1}{2}}$; (ii) density estimates of the normalized
residuals $T^1_{n,p} (D^1_{n,p})^{-\frac{1}{2}}$; and (iii) Pearson residuals
defined as $(X_{ij} - \hat{\mu}_{ij})^2/ \mathcal{V}(\hat{\mu}_{ij})$.
Figure~\ref{fig:family} shows these results.
\begin{figure}[htbp]
\centering
\includegraphics[width= \textwidth]{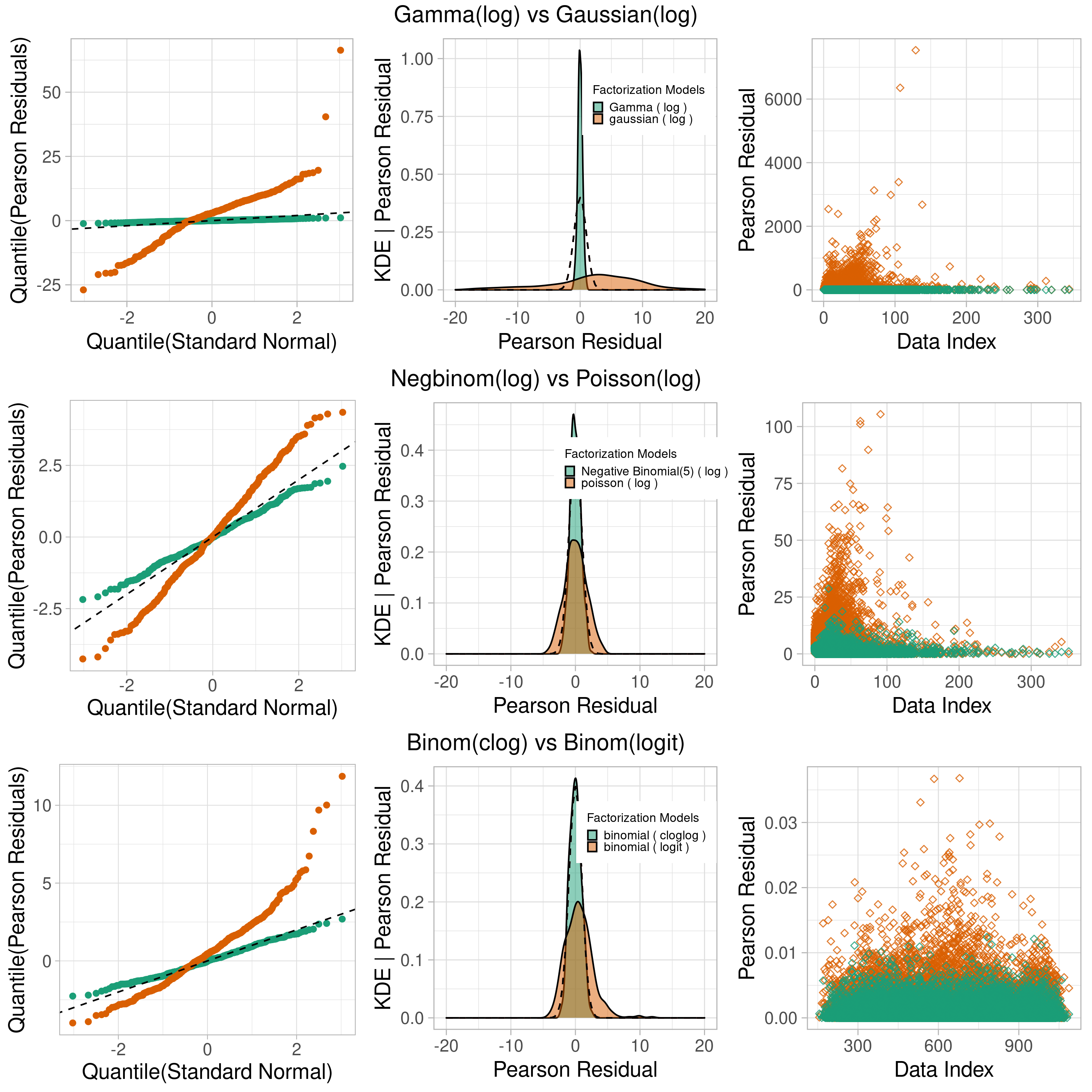}
\caption{\textbf{Goodness of fit tests:} green indicates matched family/link
fits. Columns (from left to right) represent respectively the QQ, density, and
Pearson residual plots.}
\label{fig:family}
\end{figure}

By comparing Pearson residuals (third column in Figure~\ref{fig:family}), we
can conclude that a DMF under correct family and link parameterizations
indeed yields better results. Since we generate the data from the given family (in
green), it is not surprising to observe that all $\chi^2_{G-1}$ test
statistics from correct family/link specifications have p-values equal to
\textbf{1} and all $\chi^2_{G-1}$ statistics from wrong family/link
assumptions have p-values equal to \textbf{0}.

Notice that for the family test in Figure~\ref{fig:family} we employed the
gamma and negative binomial families with true dispersion parameter $\phi$,
which is unknown in real world datasets. We here suggest that a naive
moment estimator of the dispersion estimator (e.g.,
$\hat{\phi} = \text{max}\{0.1, \mu_{X^2} / (\sigma_X^2- \mu_X)\}$ for negative
binomial) can still effectively distinguish the true simulation factorization
family from the comparing non-dispersed factorization family.
With the same $q, p, n$ and $\Lambda_0, V_0^\top$ setup, we simulate gamma and
negative binomial data and compare $\chi^2$ family test p-values
among three different factorization families: 1) dispersed family with true
dispersion $\phi$, 2) dispersed family with estimated dispersion $\hat{\phi}$, and 
3) non-dispersed family. We generate data matrices and evaluate family test
p-values for 10,000 simulations. The results from a Poisson-negative binomial example are shown in Figure~\ref{fig:gof:sensitivity}.
\begin{figure}[htbp]
\centering
\includegraphics[width = 0.9\textwidth]{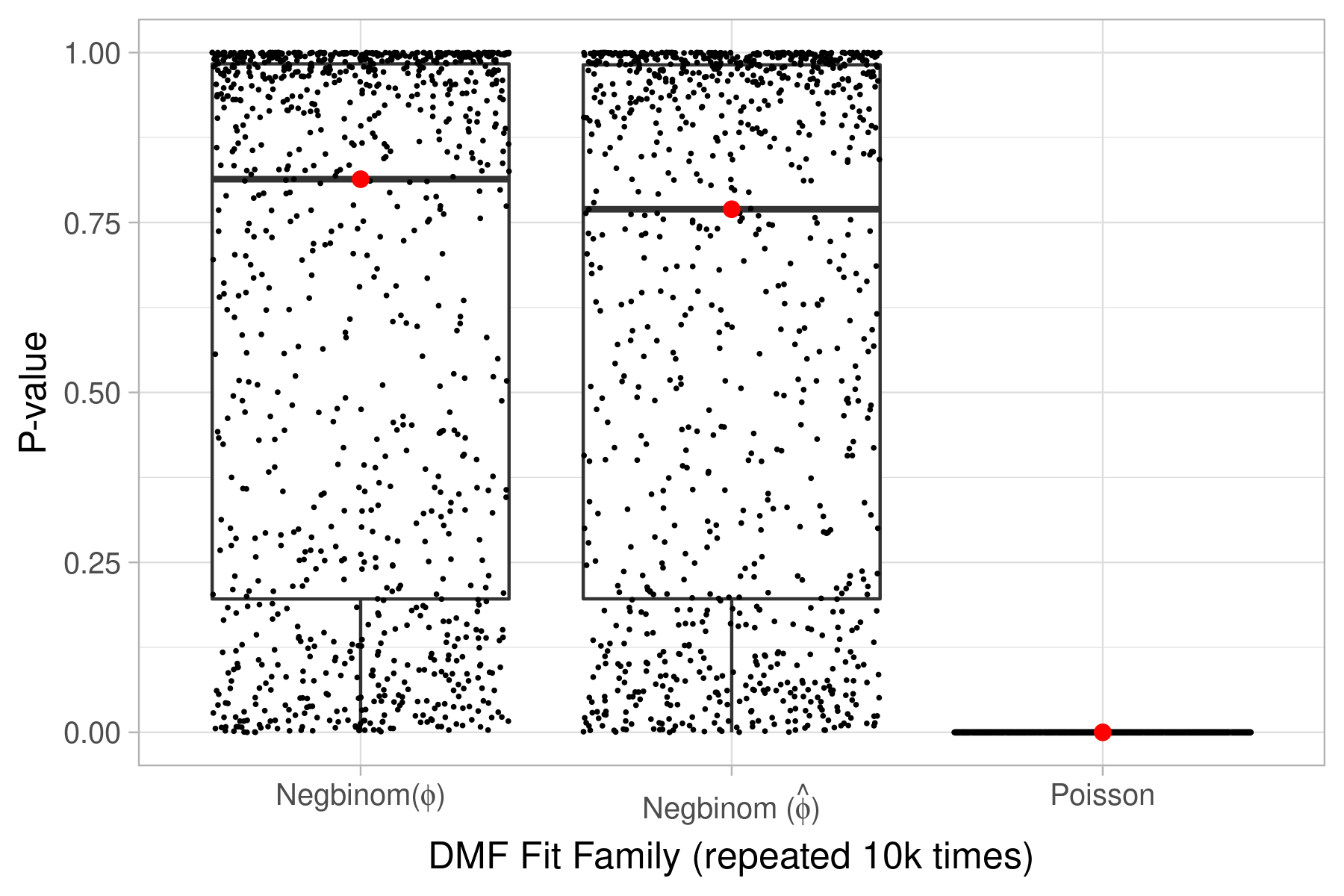}
\caption{\textbf{Sensitivity analysis on $\phi$ estimation:} with red dot
indicating the median, all simulations are under log link.}
\label{fig:gof:sensitivity}
\end{figure}

From Figure~\ref{fig:gof:sensitivity}, we can conclude that the p-values are
not significantly different when we compare between estimated and true
dispersion parameters in factorizations for dispersed families.
Both dispersed factorizations are clearly different from the Poisson
factorization family. The high p-values for the dispersed families
indicate that there is no sufficient evidence to reject the null hypothesis
that the negative binomial (dispersed family) is the true factorization
family. In fact, because the negative binomial distribution is more flexible
and better represents the variance in data entries, it should consistently
show higher p-values when compared to the Poisson distribution even when the data already
show good agreement with the non-dispersed (Poisson) family. We have
empirically confirmed this observation in a separate simulation study.

\subsubsection{Power Analysis}
We investigate the power of our test according to two sets of parameters: $n$ and $p$, specifying the sample size; and $\|\eta\|_F$, controlling the effect size. Power is, in general, sensitive to these parameters, increasing with them. However, in our setup we expect a more complex, combined effect: small data matrices with overall small entry magnitudes can be well explained by multiple families, while larger matrices provide more opportunities to spot discrepant entries under the null, specially with larger entries as they impact the variance function $\mathcal{V}$ more pronouncedly.

Since deriving the power of the test analytically is challenging under the DMF setup, we estimate it empirically under four commonly used factorization families in simulation studies.
The four selected factorization families and links are poisson (log), binomial (logit) with weight $w=90$, gamma (log) with $\phi=1$, and negative binomial (log) with $\phi=5$. For each combination of factorization family and link, we simulate data under $n= 200$, $600$ and vary the dimension according to $p = [0.2n,0.3n,\ldots, 0.7n]$. The magnitude $\|\eta\|_F$ is controlled by a $\sigma$ parameter with $\eta_{ij} = \Lambda_i^\top V_j, \Lambda_{i} \sim N(0, \sigma^2I_q), V_j \sim N(0, \sigma^2 I_q)$. The $\sigma$ parameter is taken to be $\{0.1, 0.2, 0.3, 0.4\}$ for the four families.

We simulate data $X$ according to each of the four families indexed by $j = 1,\ldots, 4$ under each $n,p,\sigma$ setup. The simulation(sampling) is repeated $S = 30$ for each of the family index. Conditional on each of the simulated data $X$, we then choose $G= np / 50$ to ensure each group has $50$ residual observations.
Next, we compute the p-value for the generalized Hosmer-Lemeshow test using each of these four families indexed by $i=1,\ldots, 4$ as the null family, with $i=j$ indicating a DMF fit under the true generating family/link assumption. 
For each combination of simulated hypothesis family $j$ and fitted hypothesis family $i$, we denote as $P_s(i,j)$ the $p$-value for sample $s$. This simulation provides power estimates $1 - \hat{\beta}(i,j) = \frac{1}{S}\sum_{s=1}^S 1(P_s(i,j) \leq \alpha)$ at significance level $\alpha$ whenever $i \neq j$. In the study we set $\alpha = 0.05$. As an initial sanity check, the power of two extreme cases of $n,p,\sigma$ are shown in Figure~\ref{fig:extreme_power}.
\begin{figure}[H]
    \centering
    \includegraphics[width =\textwidth]{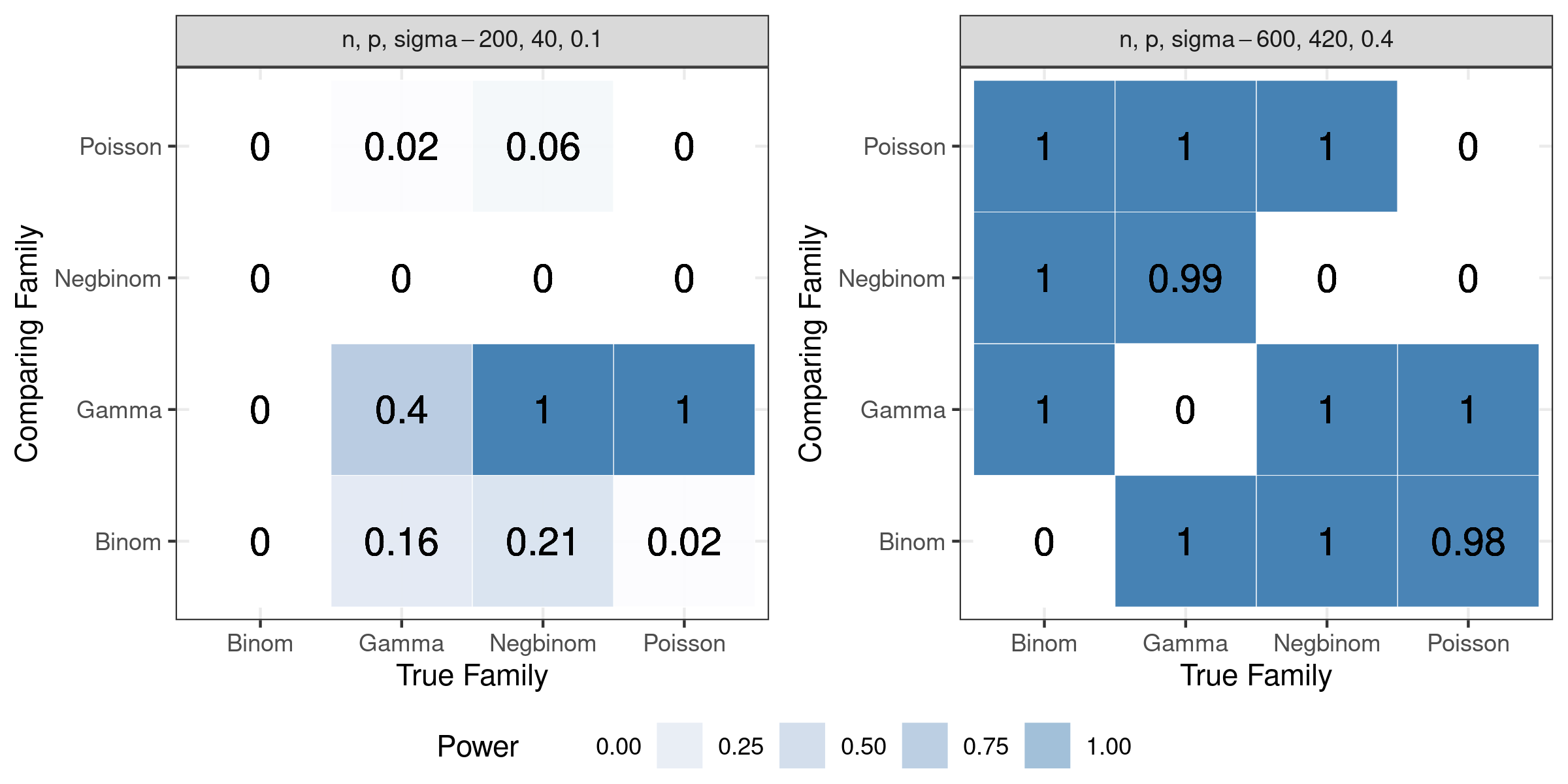}
    \caption{Extreme power comparison.}
    \label{fig:extreme_power}
\end{figure}
As we can see from the left matrix in Figure~\ref{fig:extreme_power}, in general, with small sample sizes ($n=200$, $p=40$) and magnitude ($\sigma=0.1$) we have low power as expected, shown here by having $1 - \hat{\beta}(i,j) \approx 0$ on the off-diagonal elements. True dispersed families such as gamma and negative binomial, however, show better discriminatory power even at these challenging conditions. As we focus on the right figure to increase sample sizes $n,p$ and the effect size $\sigma$, we observe higher power, again as expected. The only exception to $1 - \hat{\beta}(i,j) \approx 1$ in the off-diagonal entries in the right matrix is the true Poisson versus compared negative binomial case, which is also expected since the former is a particular case of the later. This observation additionally confirms the robustness of the negative binomial distribution in terms of explaining variation within Poisson data. 

Despite the fact that the designed statistics is able to distinguish the correct families and links, we observe that the diagonal of the two matrices in Figure~\ref{fig:extreme_power} is not converging to the advertised $\alpha=0.05$ significance level. From our empirical explorations, the desired $\alpha = 0.05$ convergence can be achieved by requiring 1) asymptotically larger $n$, $p$, 2) a carefully chosen number of groups $G$ and 3) a good initialization point for optimization convergence. The consideration of these factors will further depend on the family and link being simulated, which will make the design of the experiment unnecessarily lengthy. We decided to keep this interest open to follow-up research since this experiment has sufficiently provided the first assurance (to the best of our knowledge) for practitioners to access the appropriateness of the DMF family/link. That is, the experiments designed here clearly demonstrate that the test has high sensitivity and specificity with large sample sizes.

Next, to observe the effect of size on power, we estimate it from 10 replicates under different simulation scenarios of increasing $n \times p$ overall size with fixed $\sigma=0.4$. The results are summarized in Figure~\ref{fig:size_power}.
As we can see, the power of our family test generally increases as the sample size $np$ increases while keeping the nominal significance level. As in the previous simulation study, negative binomial fitted models fail to reject Poisson data, as expected. Interestingly, negative binomial generated data is harder to fit under other families, even dispersed ones, while Poisson generated data is more flexible and be well fit even with a binomial family under a low sample size.
\begin{figure}[H]
    \centering
    \includegraphics[width = \textwidth]{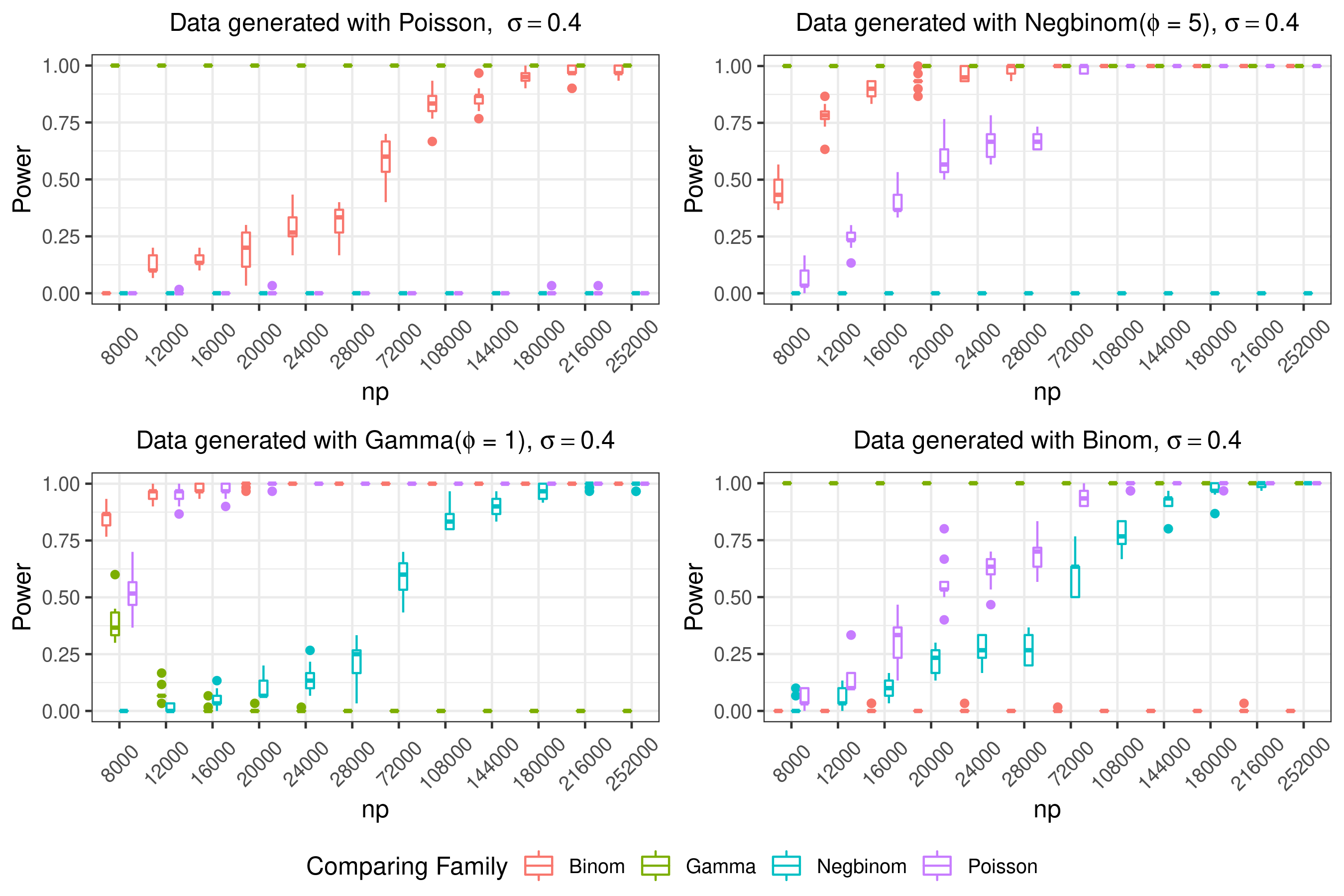}
    \caption{Sample size effect on the power of family test.}
    \label{fig:size_power}
\end{figure}

Lastly, we are also interested in the impact of $\|\eta\|_F$ through the parameter $\sigma$. In Figure~\ref{fig:sigma_power} we plot how power varies as a function of $\sigma$ for fixed $n = 200$ and $p = 80$ over 10 replicates.
As expected, power increases monotonically with the magnitude $\sigma$, indicating a clear effect of $\|\eta\|_F$ in distinguishing the variance structure and thus providing effective family selection.
\begin{figure}[htbp]
    \centering
    \includegraphics[width = \textwidth]{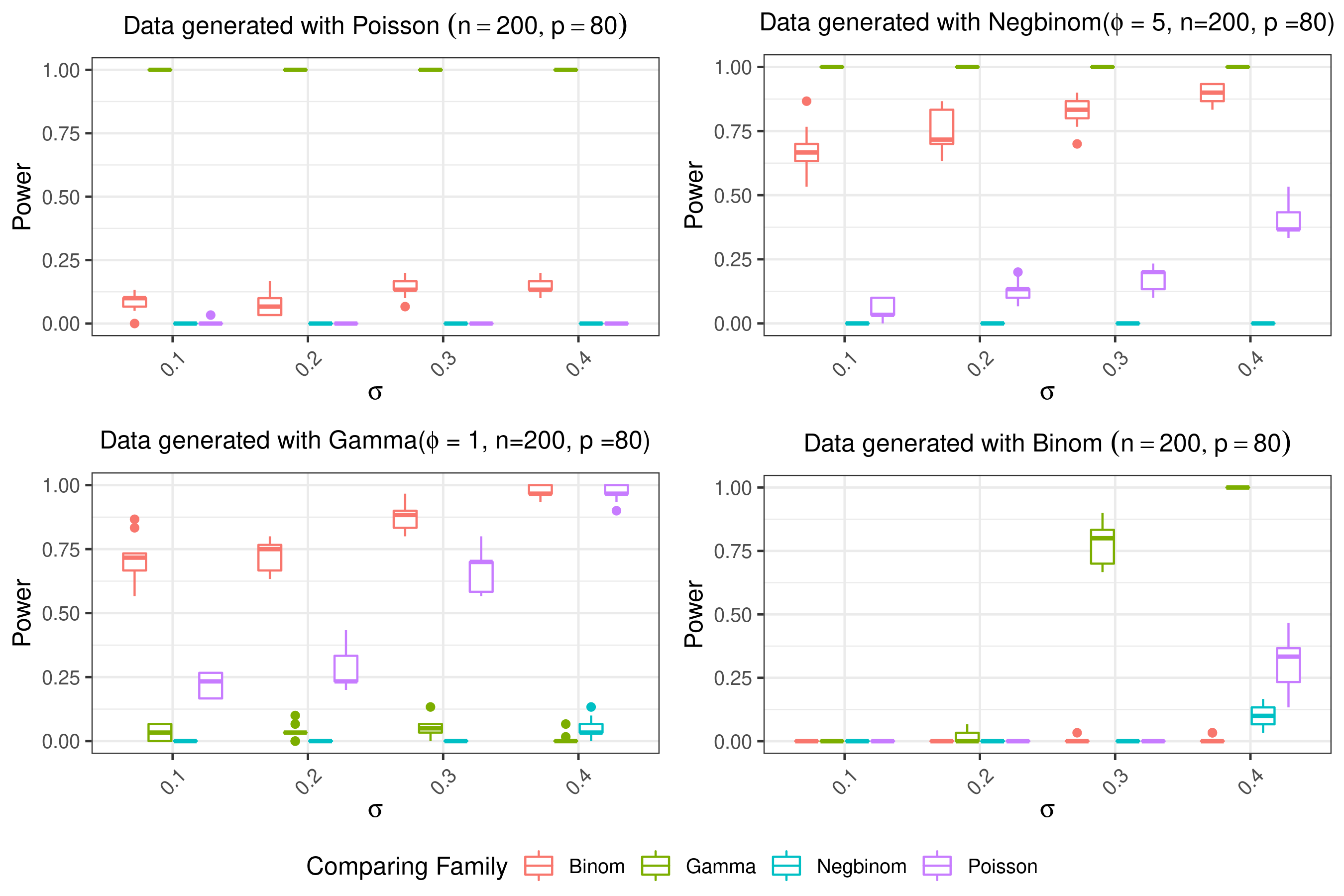}
    \caption{Sigma effect on the power of family test.}
    \label{fig:sigma_power}
\end{figure}

\subsection{Rank Simulation Study}
To determine the effectiveness of factorization rank estimation, we continued to simulate 
data given a specific family and link function $X \sim F(g^{-1}(\eta))$,
where $\eta = \Lambda_0 V_0^\top + \Lambda V^\top$ is of a known true rank $q^*$. To test and
compare the generality and the accuracy of two possible rank estimation
methods~\citep[Adjust Correlation Threshold, ACT;][]{2020frank} and
\citep{2010onrank}'s maximum eigenvalue gap, we generate both random and
deterministic $\Lambda$ and $V$ according to the conditions in our
Theorem~\ref{thm:consistency} and Proposition~\ref{prop:rank}.

For the implementation of the ACT method, we simply take $\hat{\eta} = g(X)$
and use the Stieltjes transformation on the eigenvalues of its correlation
matrix. For the maximum eigenvalue gap method, we simply adopt the $\delta$
calibration algorithm in~\citep{2010onrank} with $q_{\text{max}} = p-5$. For
each simulation scenario, we set $p = 50$, $n = 500$, and vary the true
rank $q^*$. We set $\Lambda_0 V_0^\top = \mathbf{5}_{n \times p}$ to maintain the
linear predictor positive, as required for the gamma, Poisson, and negative
binomial distributions, but set $\Lambda_0 V_0^\top = \mathbf{0}_{n \times p}$ to avoid
perfect separation in the binomial case. Here we highlight a simulation
scenario with $n = 500$ and $p =50$; experiments under different aspect ratios
$p/q \in [0.1, 0.2, \cdots 1]$ are repeated and resulted in similar
conclusions, so they are not shown for brevity. We simulate the rank
determination process 1,000 times under the different scenarios listed in
Table~\ref{tab:simu_cases}.
\begin{table}[h]
\centering
\begin{threeparttable}
\caption{\label{tab:simu_cases} Rank simulation cases.}
\begin{tabular}{cccc}\toprule
       & $q^*$ & $\Lambda$ & $V$ \\\midrule
Case 1 & 6      & $\Lambda_{ij} \sim N(0, 1)$                           & $V_{ij}\sim N(0, 1)$   \\
Case 2 & 6      & $\Lambda_{ij} \sim \text{Unif}(-n/50, n/50)$          & $V = [\text{Diag}(1)_{q \times q}|0_{q \times (p-q)}]$ \\
Case 3 & 15     & $\Lambda_{ij} \sim N(0, 1)$                           & $V_{ij}\sim N(0, 1)$   \\
Case 4 & 15     & $\Lambda_{ij} \sim \text{Unif}(-n/50, n/50)$          & $V = [\text{Diag}(1)_{q \times q}|0_{q \times (p-q)}]$ \\
Case 5 & 6      & $\Lambda_{ij} \sim N(0, \sigma = 0.1)$                & $V_{ij}\sim N(0, \sigma =0.1)$   \\
Case 6 & 6      & $\Lambda_{ij} \sim \text{Unif}(-0.1, 0.1)$          & $V = [\text{Diag}(1)_{q \times q}|0_{q \times (p-q)}]$ \\
\bottomrule
\end{tabular}
\end{threeparttable}
\end{table}
Case~1 was designed with both the conditions of ACT and maximum eigenvalue
gap method in our Proposition~\ref{prop:rank} satisfied. Case~2 was conducted
with the true $V$ matrix being deterministic, which is permitted in
Proposition~\ref{prop:rank} but violates the assumption in the ACT
method. Cases~3 and~4 were designed with $q^*$ relatively large compared to
data dimension $p$, which violates the assumption of $q^* = O(p^{\delta})$ in
ACT methods for $\delta$ large. Cases~5 and~6 are designed to demonstrate the
failure of both methods in the case of weak signal (condition C1 in
Theorem~\ref{thm:consistency} is violated). The simulation results are shown in Figure~\ref{fig:rank_simu}. 

\begin{figure}[htbp]
\centering
\includegraphics[width=\textwidth]{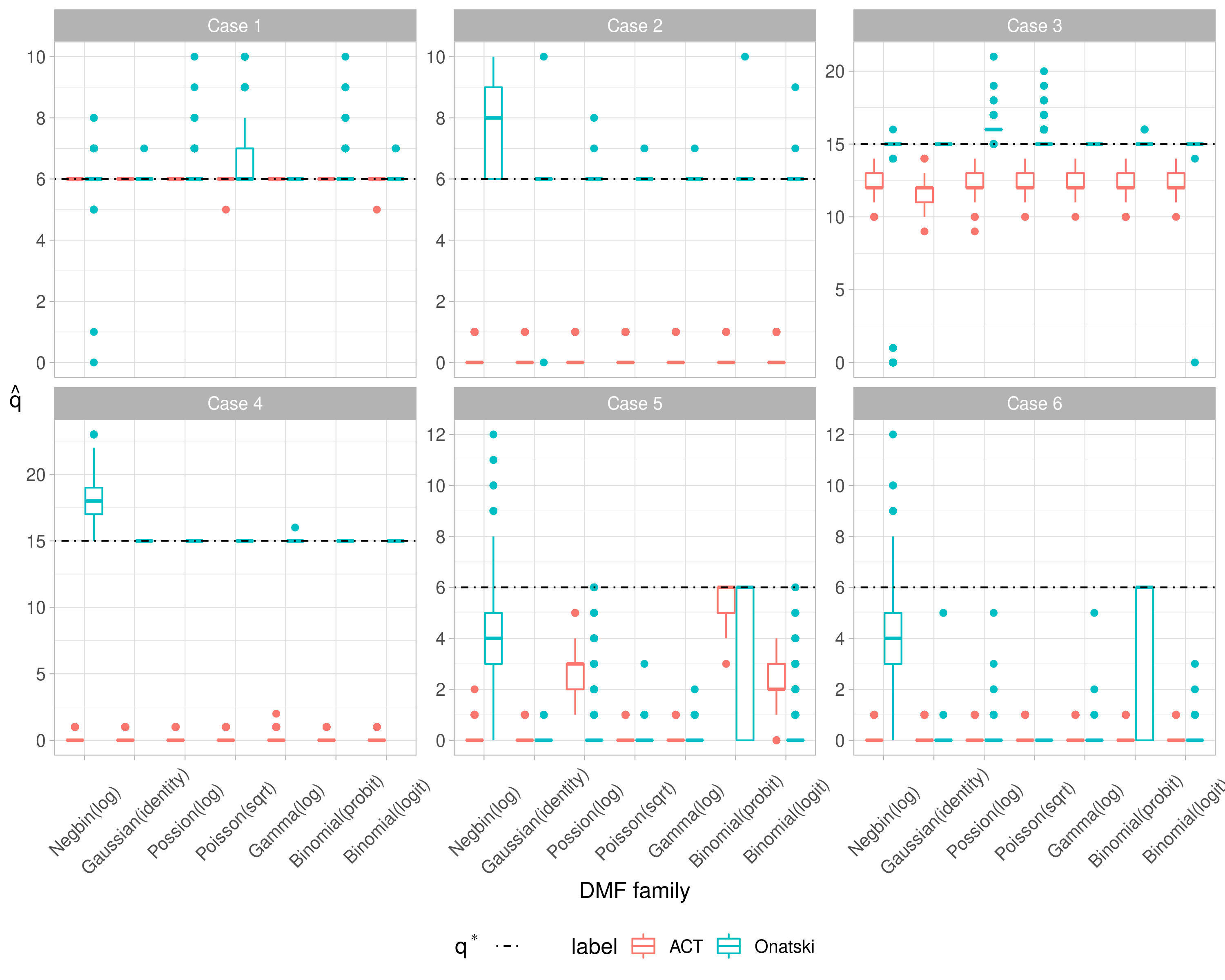}
\caption{\textbf{Rank simulation results:} ACT (red), maximum eigenvalue gap (blue).}
\label{fig:rank_simu}
\end{figure}

We can conclude that although the ACT method works almost perfectly when the
conditions in~\citep{2020frank} are satisfied, it also consistently fails to
estimate the true rank when the conditions are violated (e.g., in case~2 with
$V$ being deterministic, case~3 when $q^*$ is relatively large compared to
$p$). On the other hand, the maximum eigenvalue gap method estimates the true
rank $q^*$ with few mistakes for the first four simulation scenarios.

If we take a closer look at the maximum eigenvalue gap plot we notice that
most, if not all, of the over estimations resulting from the maximum
eigenvalue gap method are caused by the crude $\delta$ calibration algorithm
in~\citep{2010onrank}. That is, the algorithm attempts to capture the maximum
eigenvalue gap through estimating the slope locally up to convergence, which
might suffer from high variance under a limited local data sample. A typical
overestimating effect is shown in Figure~\ref{fig:rank_simufail}, which
clearly indicates a rank 6 and 15 but the $\delta$ calibration algorithm
outputted a rank of 9 and 17 for cases~2 and~4.
\begin{figure}[H]
\centering
\includegraphics[width=.49\textwidth]{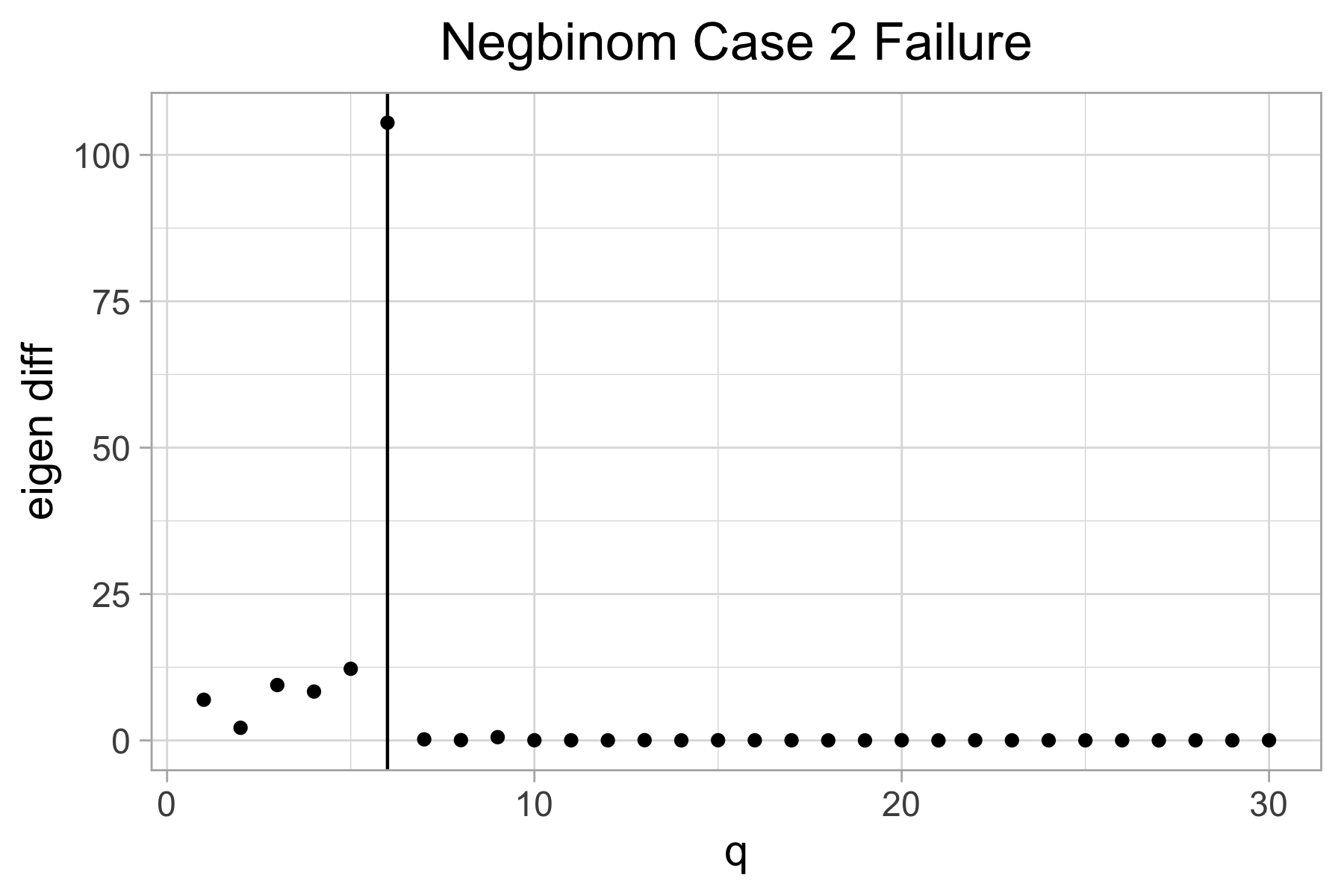}\hfill
\includegraphics[width=.49\textwidth]{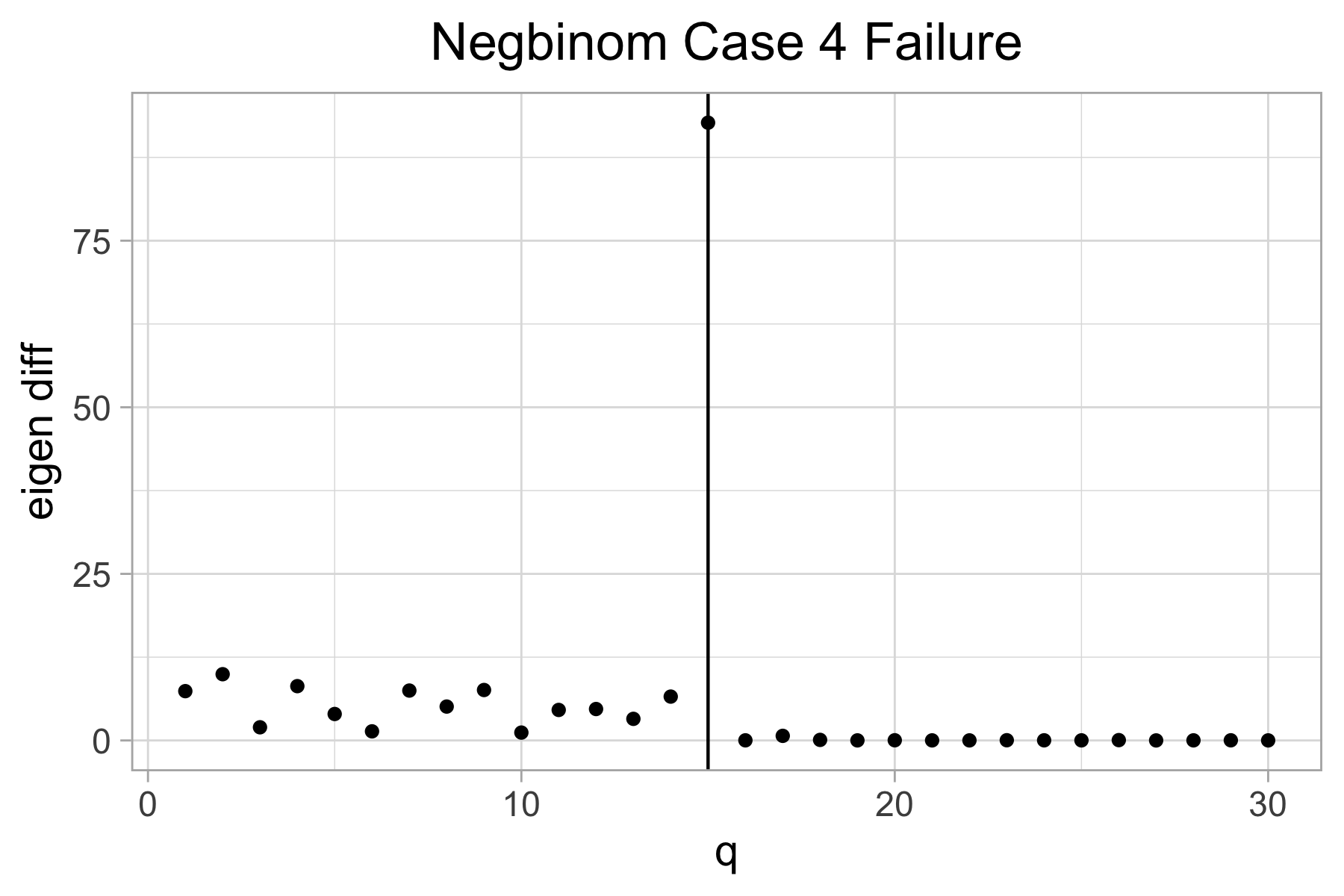}
\caption{\textbf{Negative binomial rank overestimation}.}
\label{fig:rank_simufail}
\end{figure}
We thus encourage the user to always plot the
eigenvalue gap to supervise the determination of factorization rank. We also
refer to Onatski's $\delta$ calibration algorithm for a potential
improvement on the automatic algorithm robustness. In conclusion, we adopt the
maximum eigenvalue gap method for rank determination in our empirical data
study as it is more flexible and robust.

\subsection{Case Studies -- Visualization}
\subsubsection{Face Recognition}
Since pattern recognition is closely related to matrix factorization, we
tried our method on a computer vision benchmark dataset.
The CBCL dataset is a two class dataset with 2,429 of them being human faces
and 4,548 images being non-face images.  Enriched with followup research, the
CBCL dataset has been extensively employed in computer vision
research~\citep{CHHH07}.
Containing the first comparison among PCA, NMF and factor model on the CBCL
dataset, the work of~\citep{1999nmf} brought considerable attention to NMF in
computer vision.

Here we compare factorization results of the CBCL dataset with four
potential factorization families: 1) the constrained Poisson identity link (NMF),
2) the unconstrained Poisson log, 3) the negative binomial log with MoM-estimated
$\hat{\phi} = 0.462$, and 4) the binomial family probit link with weights as 255
(max grayscale value in dataset images). The binomial factorization with weight 255 is particularly
interesting because it offers the statistical view of computer images as binomial allocations.

We first determine the factorization rank by fitting $\hat{\eta}$ with a
full rank ($q = 361$) and by performing the maximum eigenvalue gap method
(with $q_{max} = 30$). The results indicate a rank of three for negative
binomial, Poisson (log link), and binomial, and a rank of four for Poisson
(identity link, NMF). Eigenvalue plots can be found in
Appendix~\ref{appendix:rank}.
We then perform the family test with $G = 165$ with results shown in Figure~\ref{fig:cv_family}.
\begin{figure}[H]
\centering
\includegraphics[width = \textwidth]{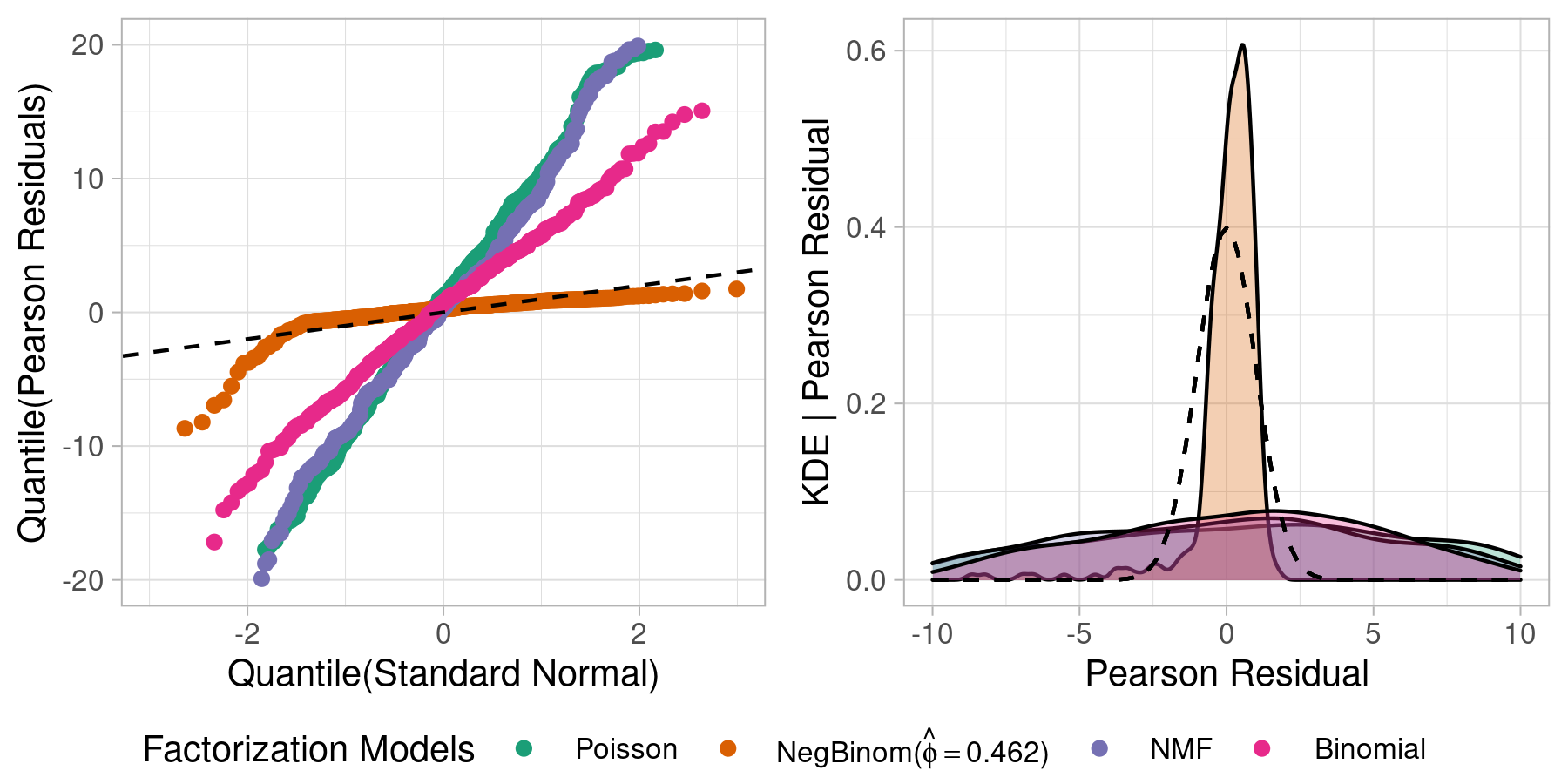}
\caption{\textbf{CBCL factorization family test} with $G = 165$. QQ-plot (left), density plot (right). Both are overlaid with
a standard normal density (black dashed line).}
\label{fig:cv_family}
\end{figure}
As can be seen, the negative binomial family is closest to normality on its components 
$T_{n,p}^1 (D_{n,p}^1)^{\frac{1}{2}}$ and thus the highest p-value in $\chi^{2}_{G-1}$ statistics. It is then 
followed by the binomial with probit link, NMF, and log-link Poisson. Note that even though the negative binomial factorization 
family is the best among the four potential families, we still observe tail
deviations from the $45\deg$ line in the QQ-plot. We attribute this deviation
to the unusual presence of zeros as the background color. The tail deviation thus
points out to a potential factorization research direction with the generation
of background pixels being modeled by a separate specification as in
zero-inflated models.

Next, we compare the projected three dimensional $\hat{\Lambda}$ for these four
families, as shown in Figure~\ref{fig:cv_pc}. Even though the negative
binomial factorization family is the best among the four in the family test,
the projected three dimensional plot indicates better separability between the
face and non-face classes under the NMF factorization. This result suggests a
better fit of negative binomial and binomial factorization to the image data
by better accommodating the dispersion induced from different classes, but a
less clear separability between faces and non-faces. This observation is
related to the positivity constraint imposed by the NMF method, which has been
shown in~\citep{1999nmf} to promote partial face learning. A potential
research direction thus would be to fit DMFs under a classification scenario, for
instance, a negative binomial or binomial family factorization for faces and
non-faces, with different factors $\Lambda$ and $V$ for each class.
\begin{figure}[H]
\centering
\includegraphics[width=.49\textwidth,height =2in]{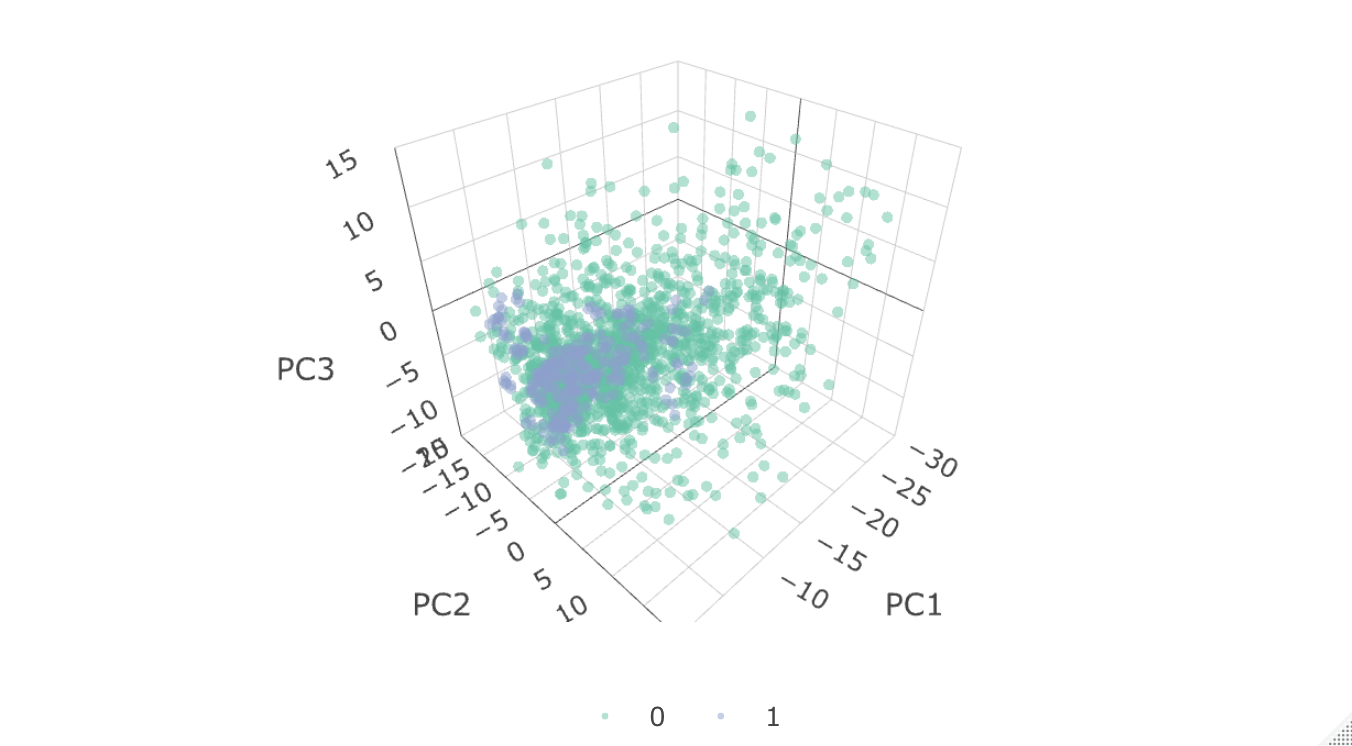}\hfill
\includegraphics[width=.5\textwidth, height=2in]{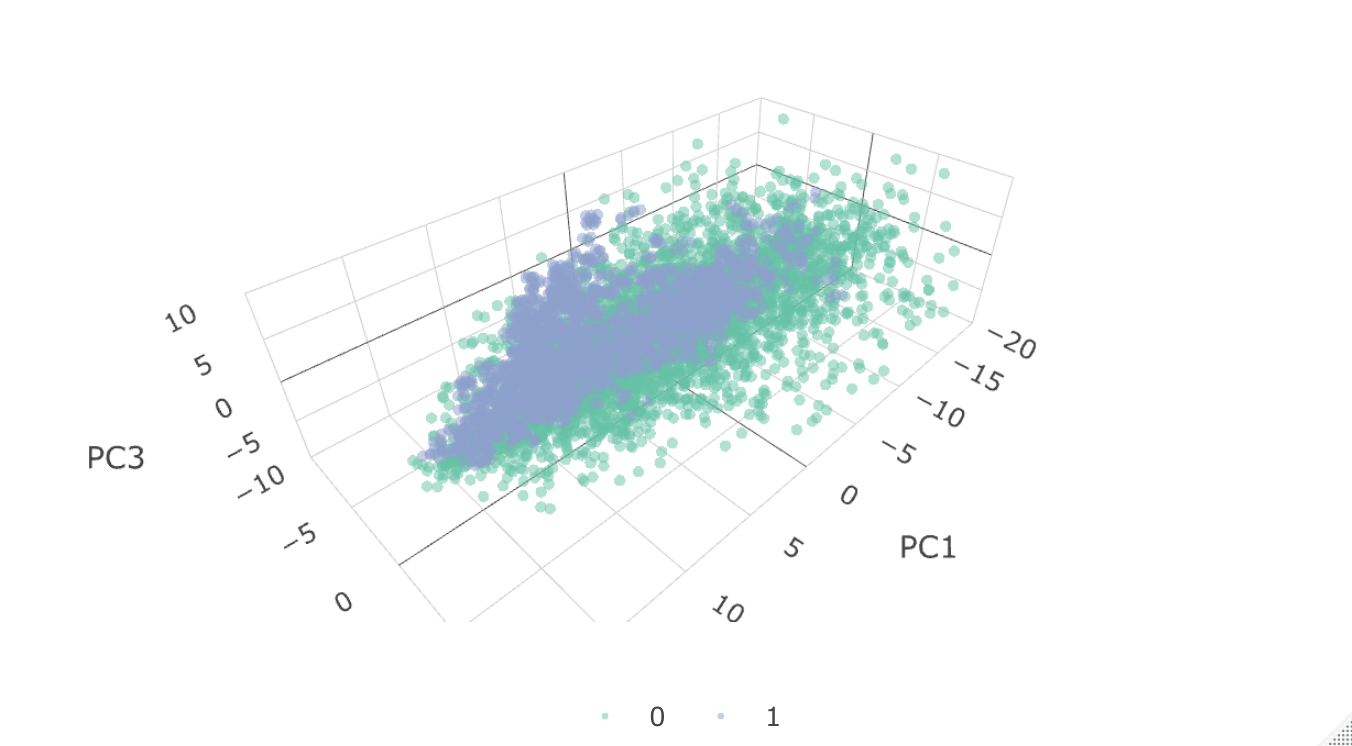}\hfill
\includegraphics[width=.49\textwidth, height=2in]{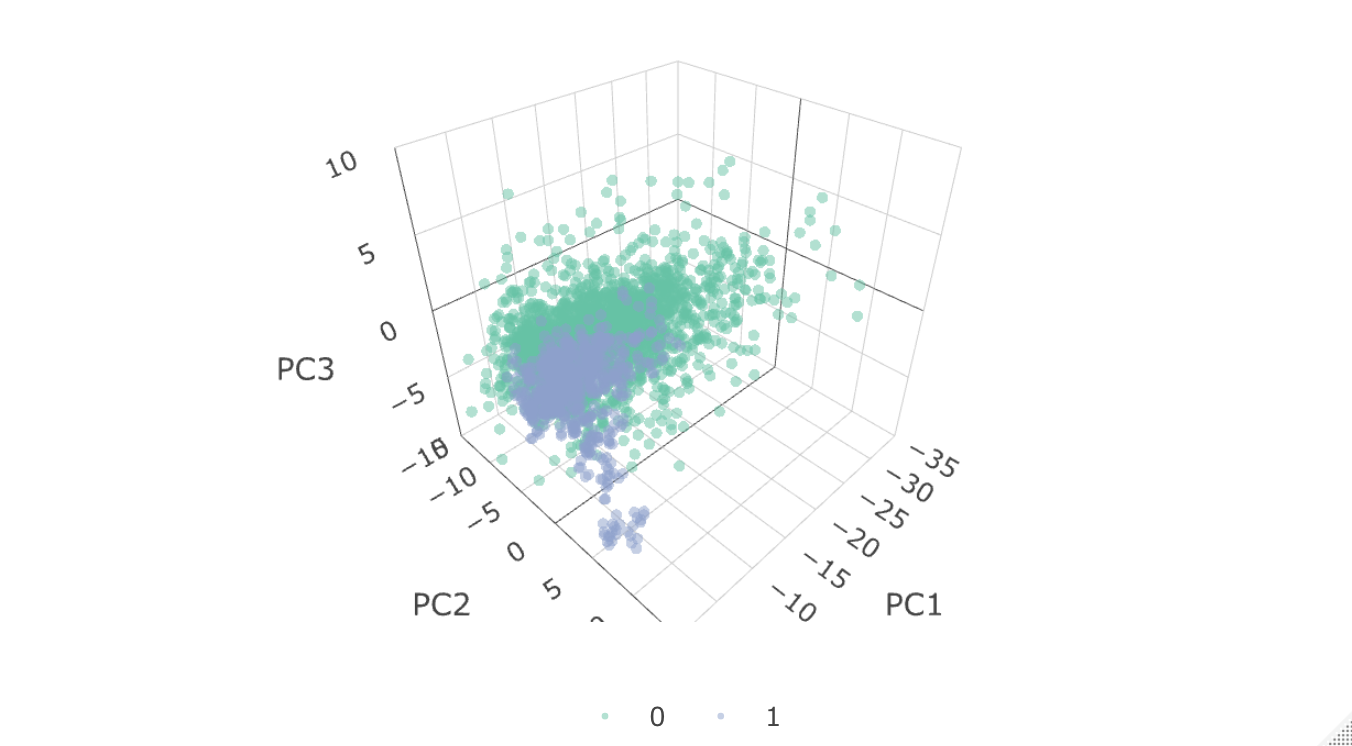}\hfill
\includegraphics[width=.5\textwidth, height=2in]{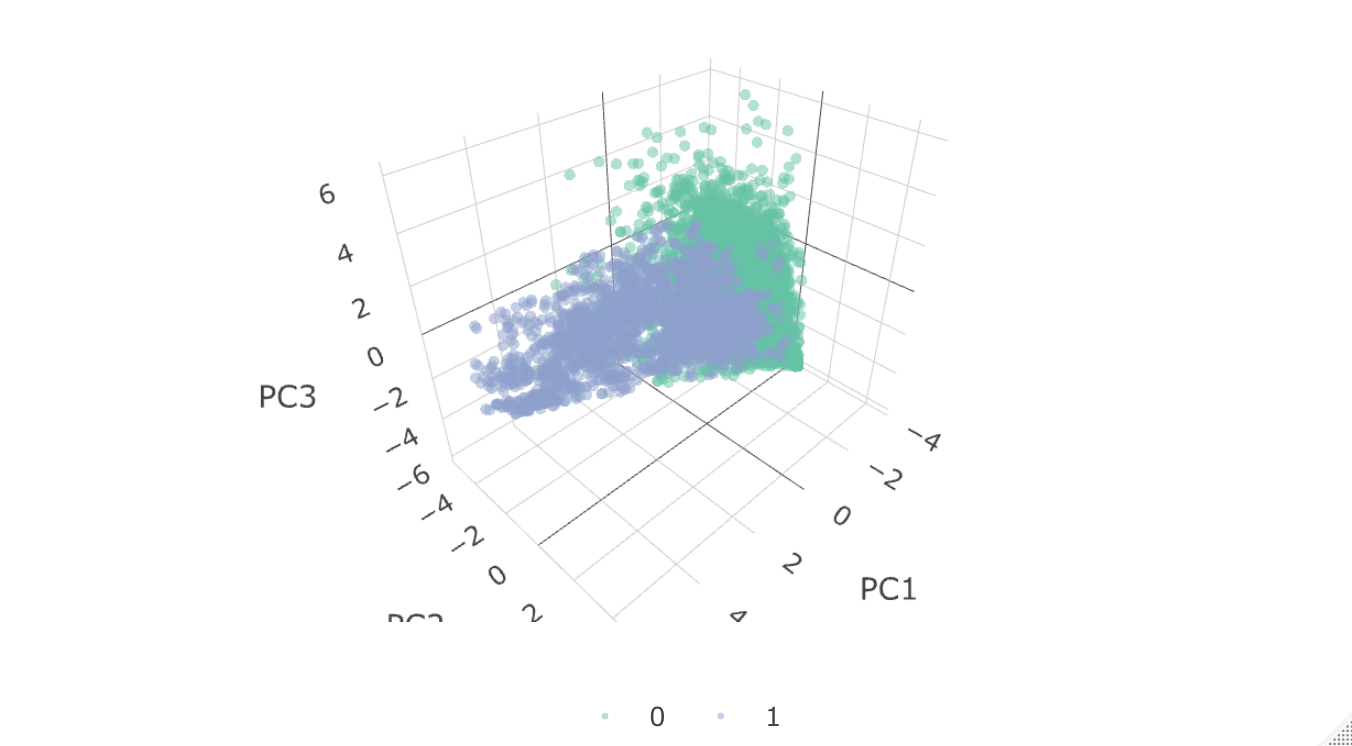}
\caption{\textbf{CBCL factorization projections:} negative binomial (top
left), Poisson (top right), binomial (bottom left), and NMF (bottom right).}
\label{fig:cv_pc}
\end{figure}

\subsubsection{Natural Language Processing}
Besides applications in computer vision, natural language processing (NLP) has
also been a very successful application area for matrix factorization. The
factorization is usually based upon a keyword co-occurrence matrix, then
termed as latent semantic analysis~\citep[LSA;][]{1990lsa}, or
keyword-document matrix, known as hyperspace analogue to
language~\citep[HAL;][]{1996hal}. One recent breakthrough is
GloVe~\citep{2014glove}, which motivates the use of a log link on the
probability of word co-occurrence matrix for word analogy task.

We hereby compare the NMF factorization against our DMF
factorization with negative binomial and Poisson log links by experimenting on
one of the NLP benchmark datasets known as ``20Newsgroup''\citep{1995newsgroup}.
The dataset contains 20,000 articles with each of them associated with one of
20 news topics. For visualization convenience, we focus on four topics---auto,
electronics, basketball, mac.hardware---that are different enough and can adopt
similar preprocessing methods to~\citep{2009nlp}. Specifically, to avoid
overfitting on specific documents, we filtered out the headers and footers and
quote information within the news while tokenizing specific keywords using
English stop words. For the tokenization of the keywords, we also filter out
words that show up in 95\% of the articles and keywords that appear in only
one article to promote the effectiveness of keyword identification. After
preprocessing, we end up with 3,784 documents and 1,000 keywords. We use the
frequency matrix as the input of our matrix factorization with $n = 3,784$ and
$p = 1,000$.

Since the data is a positive count matrix, we here
compare the factorization results of NMF, Poisson log and negative binomial
with MoM-estimated $\hat{\phi} = 0.06$. For both the Poisson log and negative
binomial log family we have found a rank of $\hat{q} = 5$, while for the NMF
we have $\hat{q} = 3$. To better visualize factorization results, we adopt
a factorization with $\hat{q} =3$ for all the families to ensure a fair
comparison to NMF. We then conduct a family test with $G= 20$ in Figure~\ref{fig:nlp_family}.
\begin{figure}[H]
\centering
\includegraphics[width = \textwidth]{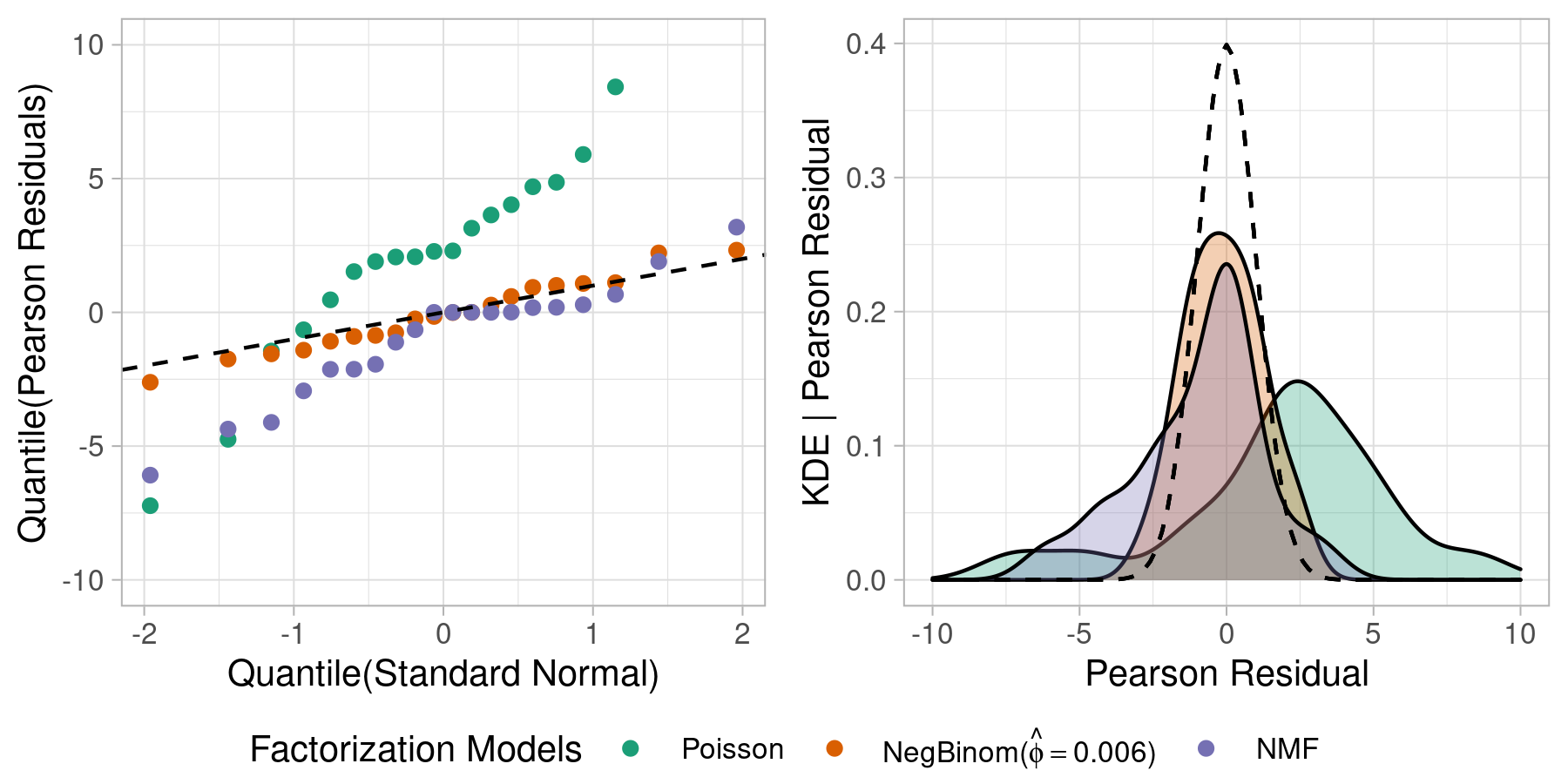}
\caption{\textbf{Keyword frequency factorization family test:} $G = 20$,QQ-plot (left), density plot (right). Both are overlaid with
a standard normal density (black dashed line).}
\label{fig:nlp_family}
\end{figure}
With a \textbf{unit} p-value, we take the negative binomial family with log link and
MoM-estimated dispersion ($\hat{\phi} = 0.06$) as an adequate distribution
while rejecting the Poisson family with log link and NMF with Poisson identity
link with both of them having zero p-value.

Lastly, we adopt~\eqref{eq:dmfcentered} for all the three factorization to
better visualize the projected three dimensional plot in
Figure~\ref{fig:nlp_pc}.
We can see after projecting the 3,784 documents onto the three dimensional
factor space, the negative binomial factorization more clearly separates
the four different topic classes. Compared to the negative binomial
factorization, the Poisson log factorization can still separate the four
different class albeit in a more compact way, which suggests the
indispensability of dispersion to account for keyword frequency variation.
The NMF with Poisson identity link perform the worst, with four topics
overlapping in the center, at the origin.

\begin{figure}[H]
\centering
\includegraphics[width=\columnwidth]{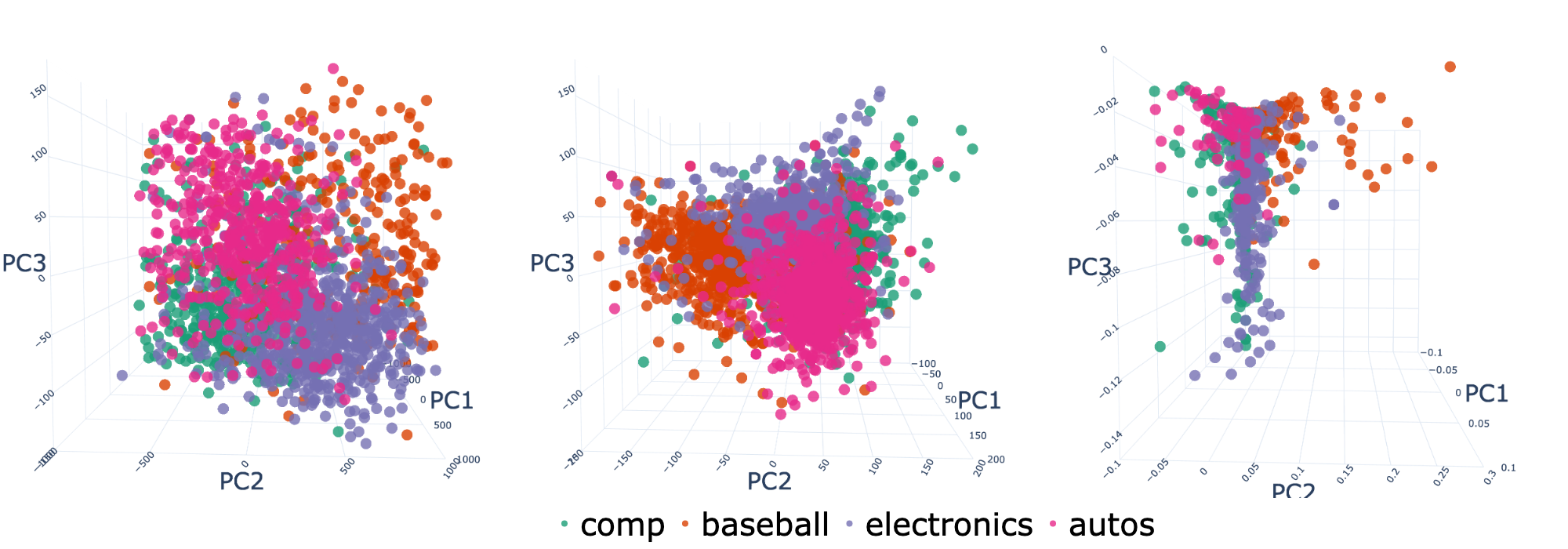}
\caption{\textbf{Keyword frequency factorization projection:} negative binomial (left),
Poisson (center), and NMF (right).}
\label{fig:nlp_pc}
\end{figure}

\subsubsection{Network Analysis}
Another important application of our deviance matrix factorization is network
data analysis, and here we focus on community detection~\citep{2002community}.
Since the adjacency matrix of a network has binary entries, we can just adopt
a Bernoulli family for the factorization. To demonstrate the effectiveness
of our DMF method on network data, we use the political blogs and Zachary's
karate club datasets as examples, both of which have been benchmark datasets
for network community detection since~\citep{2002community}
and~\citep{2005blog}.

We thus just compare the binomial logit link factorization with the commonly
used Laplacian spectral clustering method and use the separation of projected
principal components in~$\hat{\Lambda}$ to assign communities.

As before, we first employ maximum eigenvalue gap method to determine the
factorization rank. With $q_{max} = 200$ and $q_{\text{max}} = 29$, we have
found rank three to be true factorization rank for political blogs and a rank two
for karate club network (see Figure~\ref{fig:network:rank}).
The projected two dimensional plot in Figure~\ref{fig:club_pc} suggests
good separation between two groups, with the leaders being far from the
centroid of their respective groups. Notice that both methods can identify two
communities with a single principal component, but the DMF method provides more
information about group variations in the second principal component.
\begin{figure}[H]
\centering
 \includegraphics[width=\textwidth]{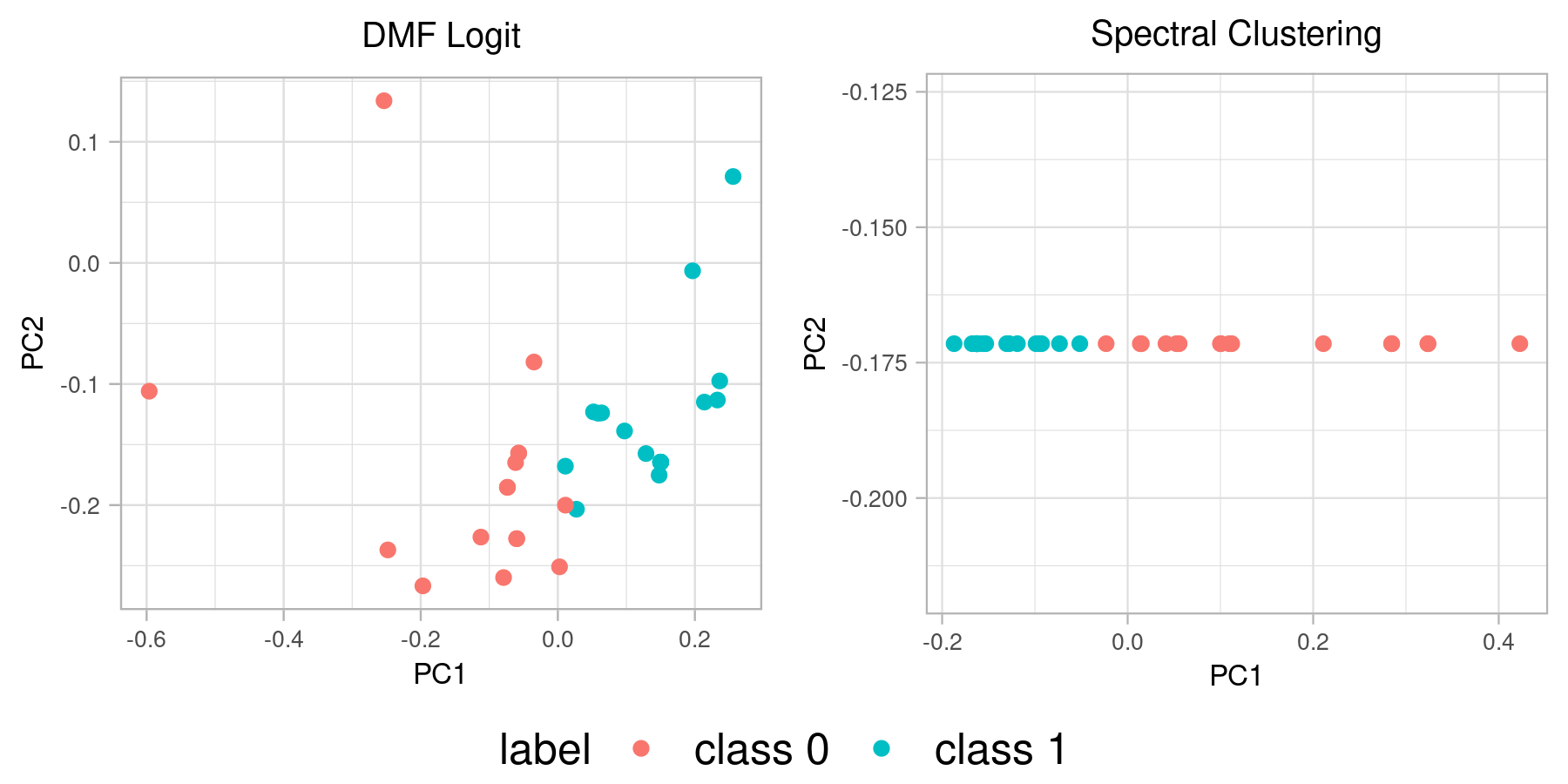}
\caption{\textbf{Karate club factorization projections}.}
\label{fig:club_pc}
\baselineskip=10pt
\end{figure}

Although our method can take asymmetric and undirected network data as inputs
for the factorization, the Laplacian clustering method depends heavily on the
positive definite property of a connected symmetric graph. To ensure a fair
comparison between our method and Laplacian clustering, we pre-process
the data by focusing on the largest connected component, which consists of
1,222 nodes. As the projected three dimensional plot in
Figure~\ref{fig:polblog_pc} indicates, compared to the spectral clustering
method our method has a much clearer separation across the two
known communities and better captures within group variation by
exploiting all three factorized principal components.
\begin{figure}[htbp]
\centering
\includegraphics[width = \textwidth]{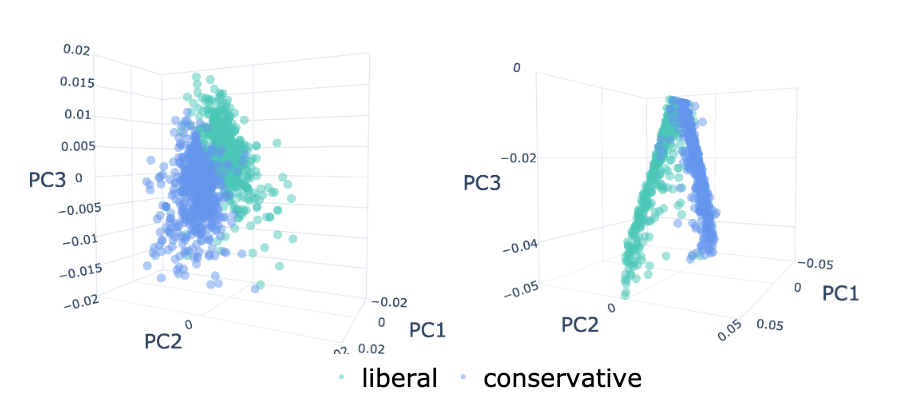}
\caption{\textbf{Polblog factorization projections}: binomial logit (left), spectral clustering (right)}
\label{fig:polblog_pc}
\baselineskip=10pt
\end{figure}

\subsubsection{Biomedical studies}
The leukemia dataset reported in~\citep{1999gene}
has been a staple benchmark for gene classification. The dataset contains
5,000 features and 38 observations with two known labels of classes: acute
myelogenous leukemia (AML) and acute lymphoblastic leukemia (ALL). Recent
researches~\citep{2004gene,2009gene} have explored factorizations based on PCA
and NMF with ranks ranging from 2 to 5.

For similar dispersion reasons, we
compare factorization results among NMF, Poisson with log link, and negative
binomial with $\hat{\phi} = 1.93$. Similar to the previous experiments, we
first employed the maximum eigenvalue gap method with a full rank fit to
determine the factorization rank. The results indicates a rank two
factorization for negative binomial with log link and a rank one for the
Poisson log-link model. For the NMF rank determination, we adopt the similar
procedure by fitting the NMF with full rank and then plot the maximum
eigenvalue gap. The results indicate possible ranks of two and possible rank
of four for the NMF (Poisson with identity link) factorization.

\begin{figure}[H]
\centering
\includegraphics[width = \textwidth]{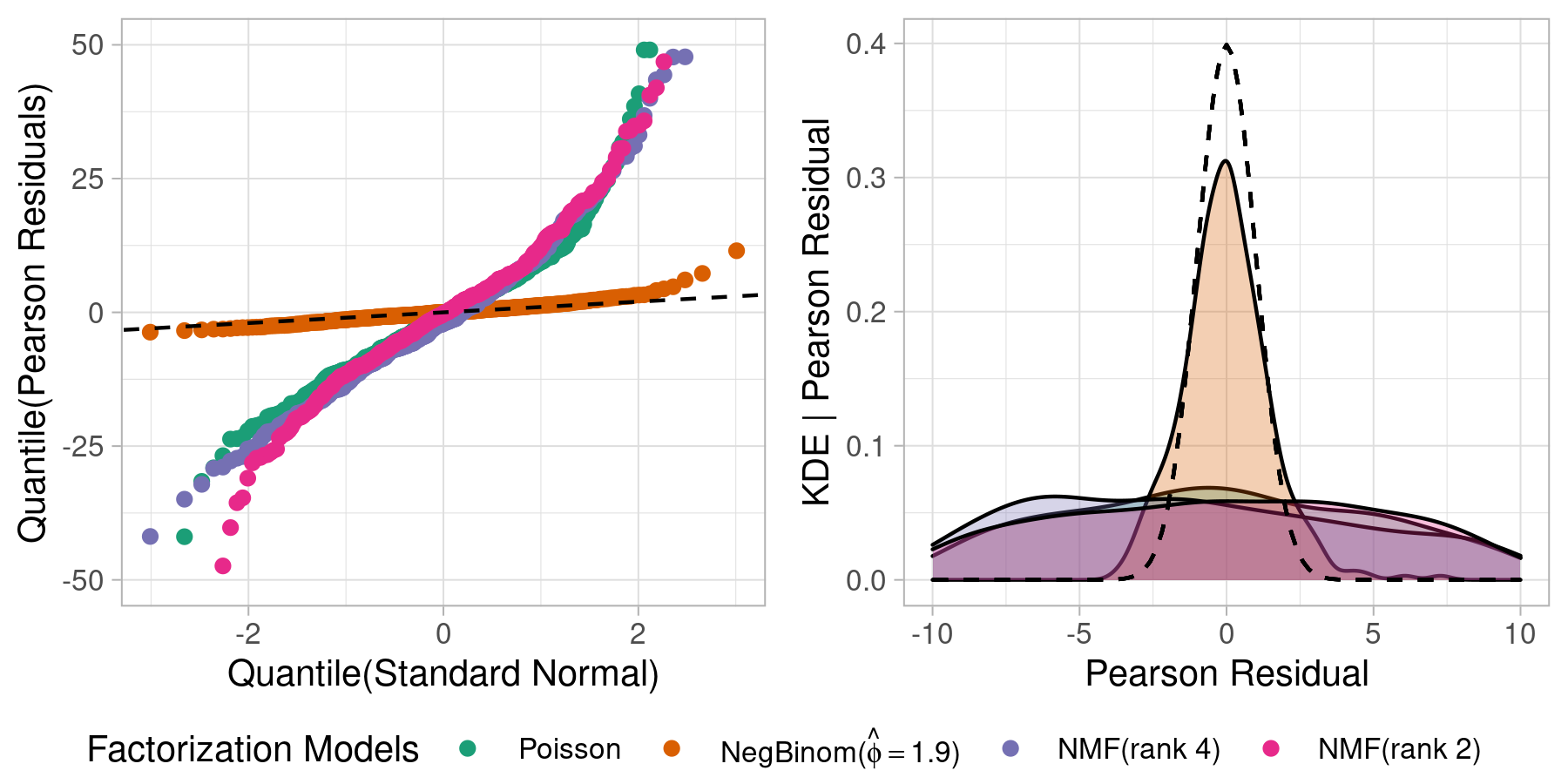}
\caption{\textbf{Leukemia factorization family test}: with $G = 380$, QQ-plot (left), density plot (right). Both are overlaid with
a standard normal density (black dashed line).}
\label{fig:leukemia_family}
\baselineskip=10pt
\end{figure}
We then conduct the family test to determine the appropriate factorization
family, as shown in Figure~\ref{fig:leukemia_family}. We can see that
Poisson (log), NMF, and PCA (identity) all failed the normality test, while the
negative binomial family with log link demonstrates significant better normality.
This result justifies a negative binomial family with
log link as a reasonable factorization family.
%
\hfill
\begin{figure}[H]
\centering
\includegraphics[width=\textwidth]{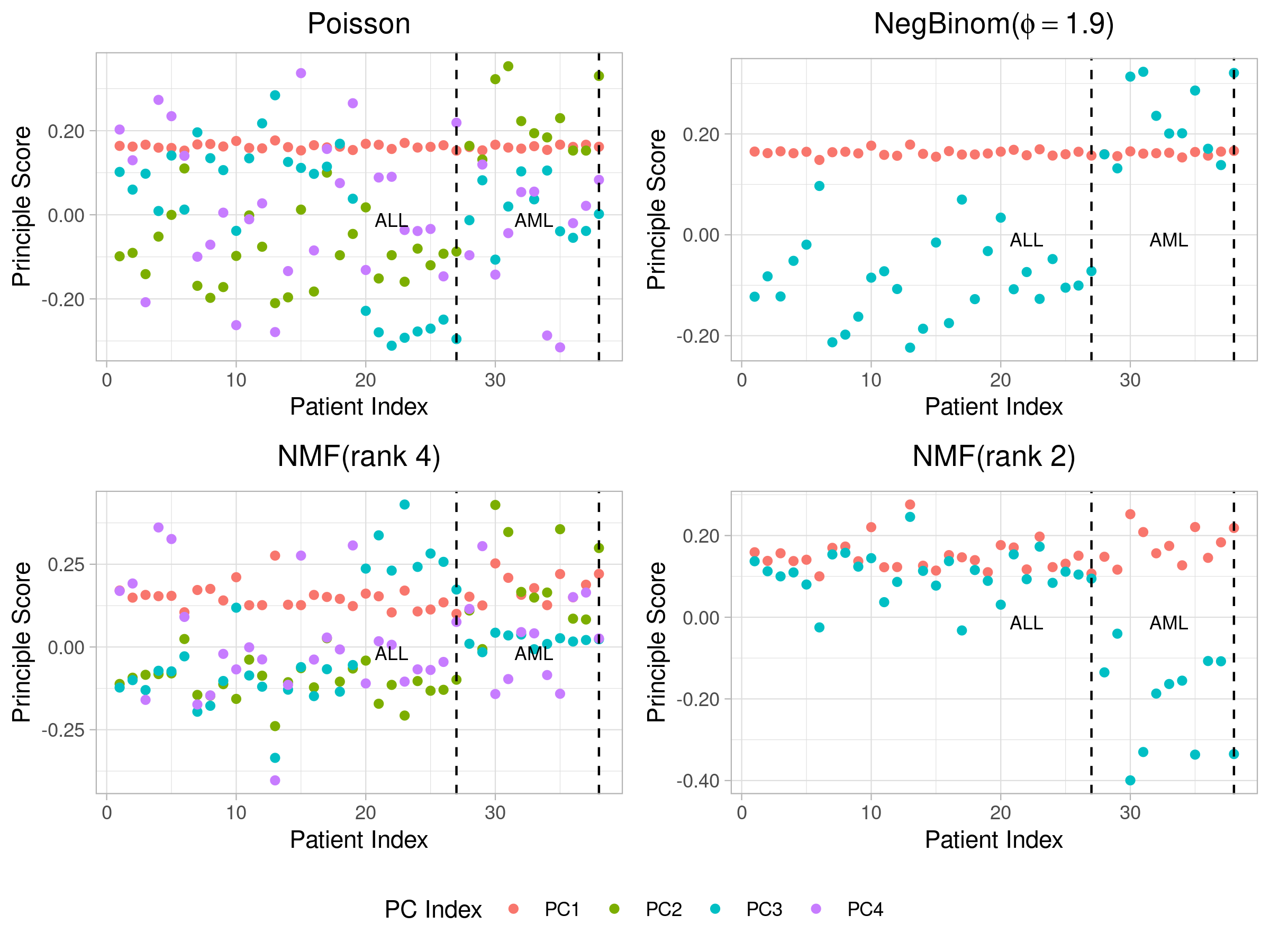}
\caption{\textbf{Leukemia dataset factorization projections.}}
\label{fig:leukemia_pc}
\end{figure}

Lastly, we compare the factorized projected matrix $\hat{V}$ in
Figure~\ref{fig:leukemia_pc}. Since in this particular factorization the three
different families all have different ranks, we plot the projected $\hat{V}$
along sample indices with leukemia types (ALL versus AML) separated by a
vertical line. As Figure~\ref{fig:leukemia_pc} shows, the negative binomial
family with log link can correctly distinguish the two types by using just
two principal components. In contrast, the non-negative matrix factorization
gives spurious results even with a higher factorization rank. The rank two
factorization result for NMF is better compared to NMF rank four
factorization, but the sample types are not as well separated as in
the negative binomial factorization, especially for samples~6, 17, and~29.
Thus we conclude that the NMF (Poisson with identity link) and Poisson family
(with log link) are not capturing the essential information of the matrix
completely; for instance, the factorized principal scores could not effectively
distinguish the correct types. By better taking data dispersion into
consideration, the negative binomial factorization is in general better suited
for positive count gene data factorizations.

\subsection{Case Studies -- Classification}
In the previous case studies we observe notable separability even from DMFs of
low ranks of $2$ and $3$ when compared to other benchmark methods. Here we aim
to further quantify this separability with objective downstream
classification tasks. We conducted these comparisons under three inherently
different classification models---tree, multinomial regression and k-nearest
neighbors (KNN) with $k=9$---using 20 factorized components. For each
classification model we replicated training/testing for 10 runs with equal
split on training (50\%) and on testing (50\%). Performances are summarized
using boxplots on the 10 reproduced multi-class AUC~\citep{hand2001simple},
which has been shown to be invariant to class-skewness.

\begin{figure}[H]
\centering
\includegraphics[width = \textwidth]{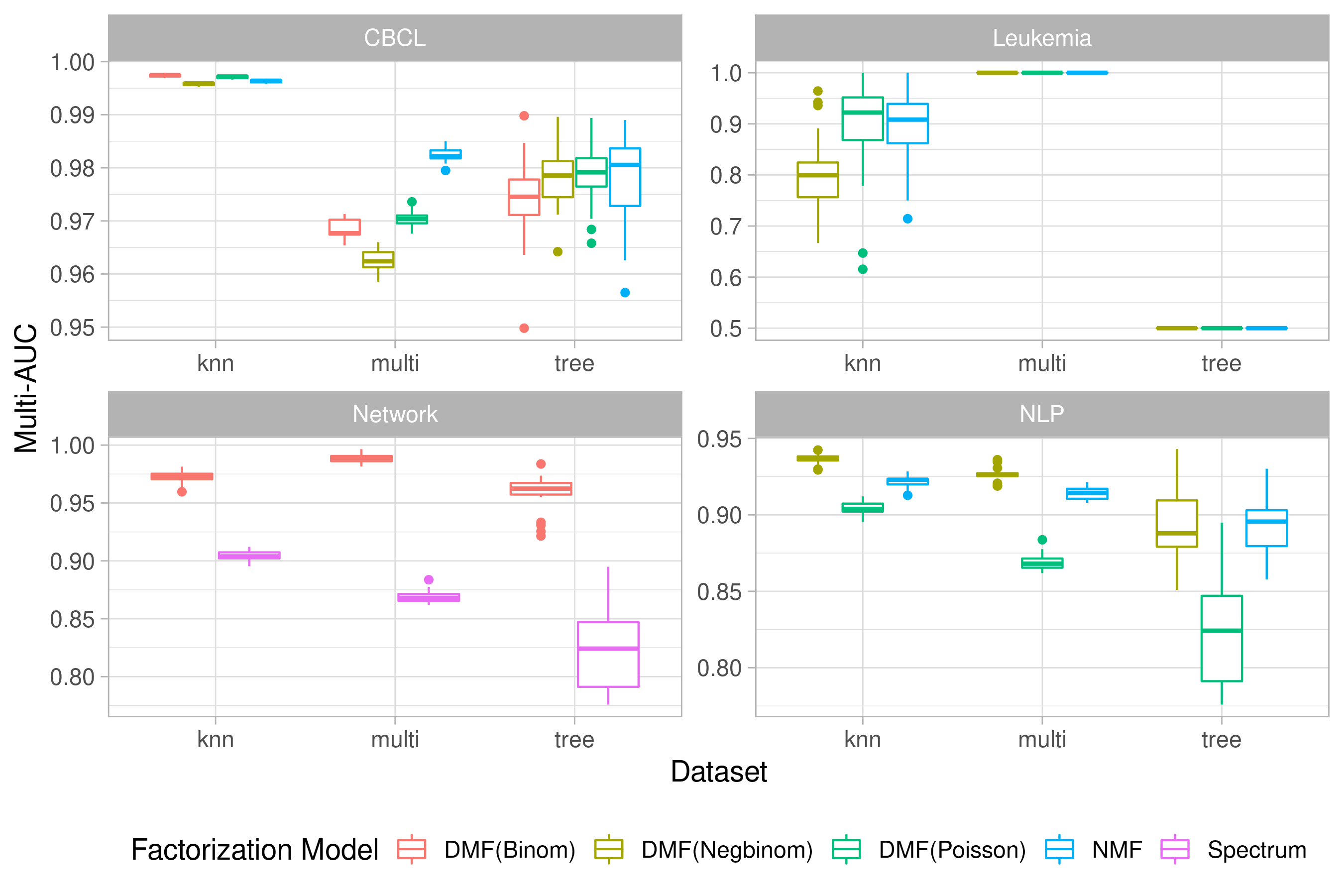}
\caption{\textbf{Multi-class AUC comparison.}}
\label{fig:auc}
\end{figure}

As we can see from Figure~\ref{fig:auc}, the DMF model has better performance
in all the three classification models for the NLP dataset and for the polblog
dataset. Perhaps due to a limited sample size (38 labels for leukemia), the
performance for the leukemia dataset varies significantly on the classification
model being used but overall we see that the DMF and NMF perform similarly
with 20 PCs. For the CBCL face/non-face dataset, we see that the NMF model
performs the best using multinomial regression and tree classification.
However, both of the two methods are less-desired compared to KNN, achieving a
nearly perfect score for the binomial DMF.

\section{Conclusion}
\label{sec:conclusion}
We have proposed a new approach for matrix factorization. The factorization
method generalizes the commonly applied PCA and NMF methods to accommodate
more general statistical families that take data features such as
over dispersion into consideration. To this end, we have generalized a fitting
algorithm with convergence to at least a stationary point being justified, and
have provided methods, supported by theoretical guarantees, to determine
factorization family and rank.

We also obtained encouraging experimental results in our simulated data
studies and in many case studies using benchmark datasets from different
fields. As we show in Section~\ref{sec:studies}, our method can better
explain sparsity patterns in keyword frequency and gene expression
matrices through over-dispersed families such as the negative binomial.
It also provides an alternative factorization perspective to image
recognition by changing the original problem to an optimal binomial weights
allocation problem. It further enriches community detection research in
network analysis by allowing a general factorization method that is applicable
to asymmetric and weighted networks.

Our work offers an important perspective to matrix statistical inference by
extending the probability concept from Gaussian and Poisson to more general
exponential families with different link functions. In fact, we plan to
explore zero inflated models to improve fit and interpretability in face of
matrices with excess zeros. For future research, we intend to enrich our
deviance matrix factorization with a hierarchical structure to model matrix
mixtures to effectively incorporate row or column classification and to
investigate more efficient optimization procedures for very large datasets.

\section{Acknowledgement}
We thank two anonymous referees and the associate editor who provided many
valuable comments that significantly improved the paper.
\if0\blind
{
We also thank Weixuan Xia who provided suggestions and validation on our
consistency proof.
}
\fi

\appendix
%
\section{Factorization Ranks}
\label{appendix:rank}
\begin{figure}[H]
\centering
\includegraphics[width=.45\textwidth]{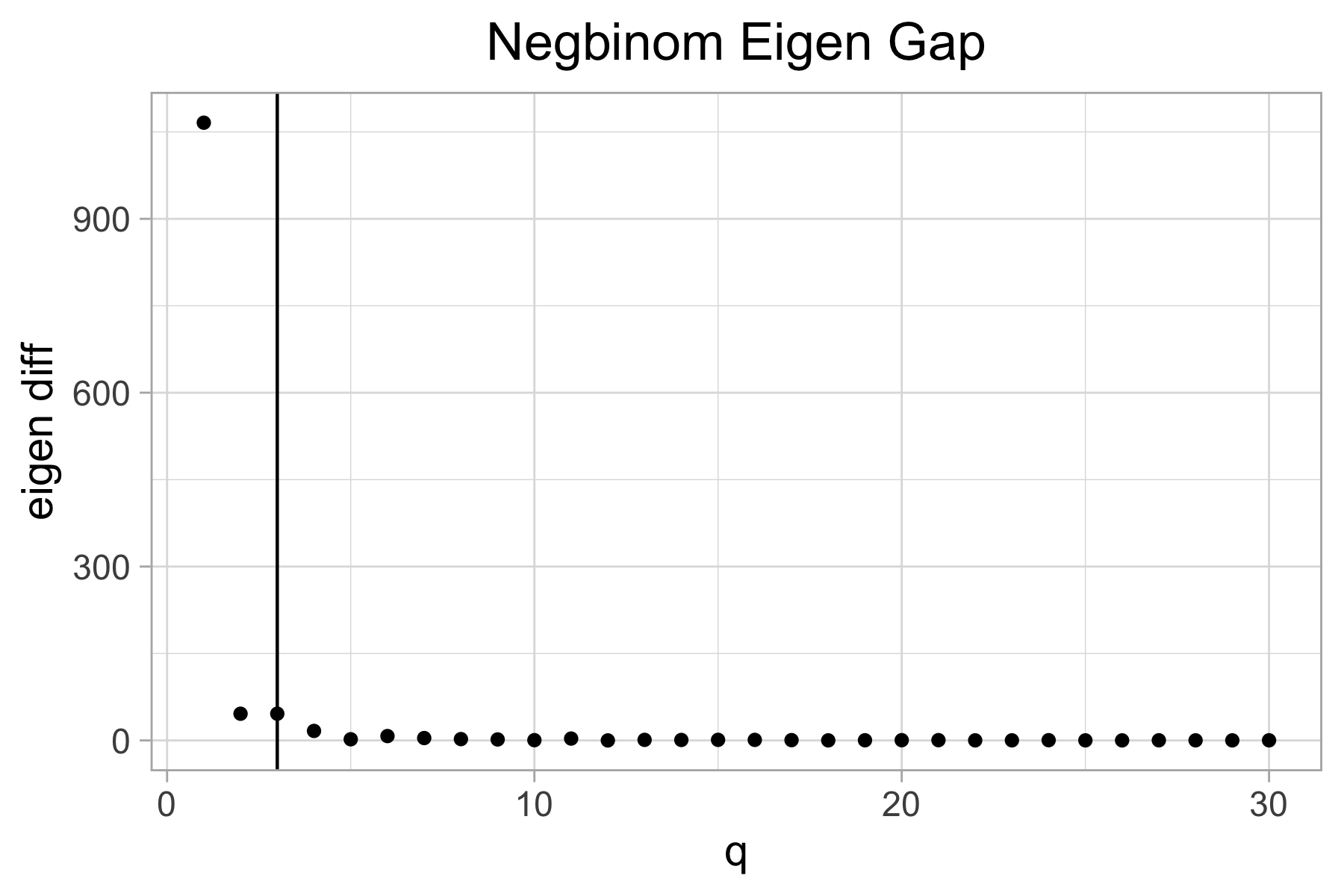}
\hspace{0.5cm}
\includegraphics[width=.45\textwidth]{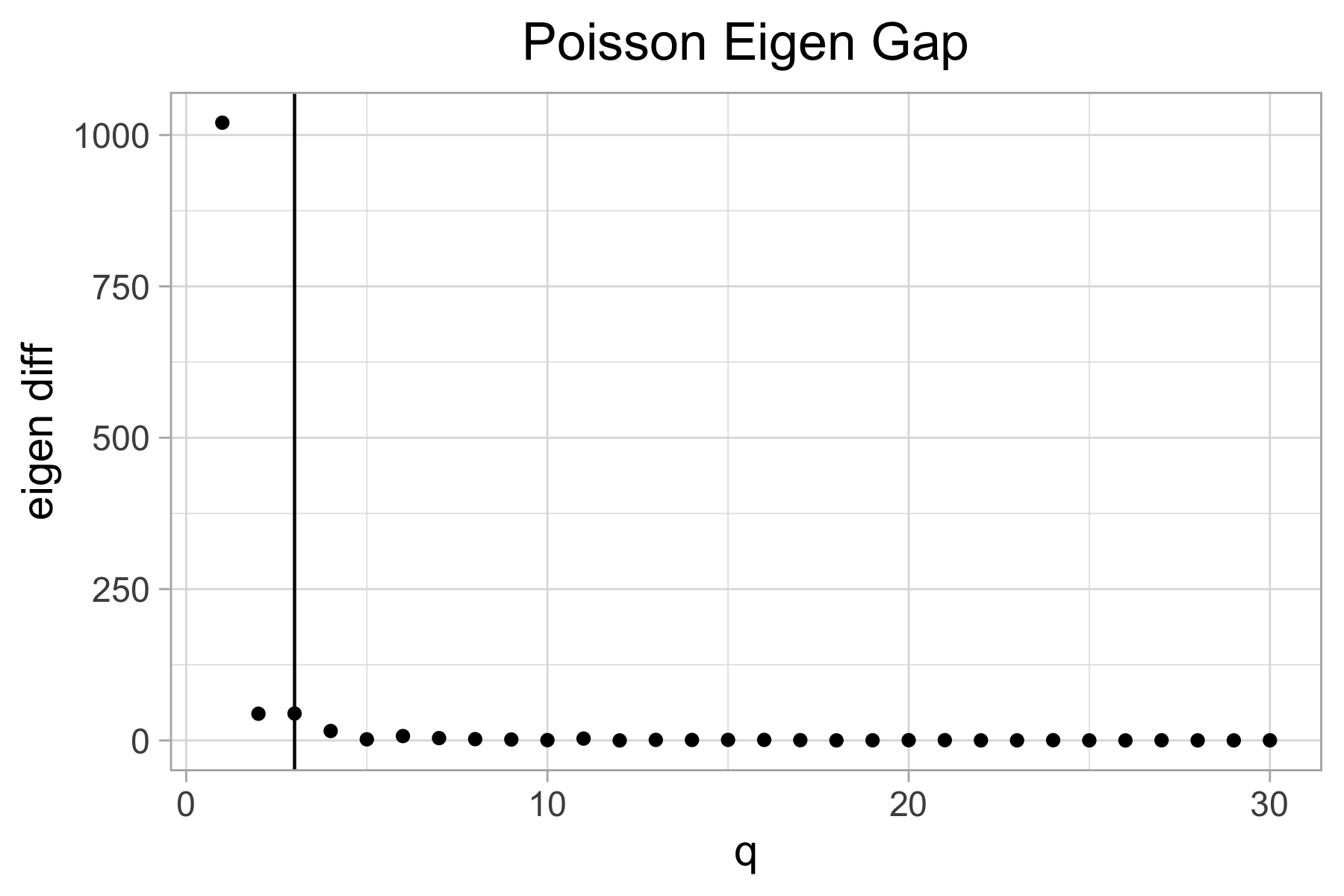}
\includegraphics[width=.45\textwidth]{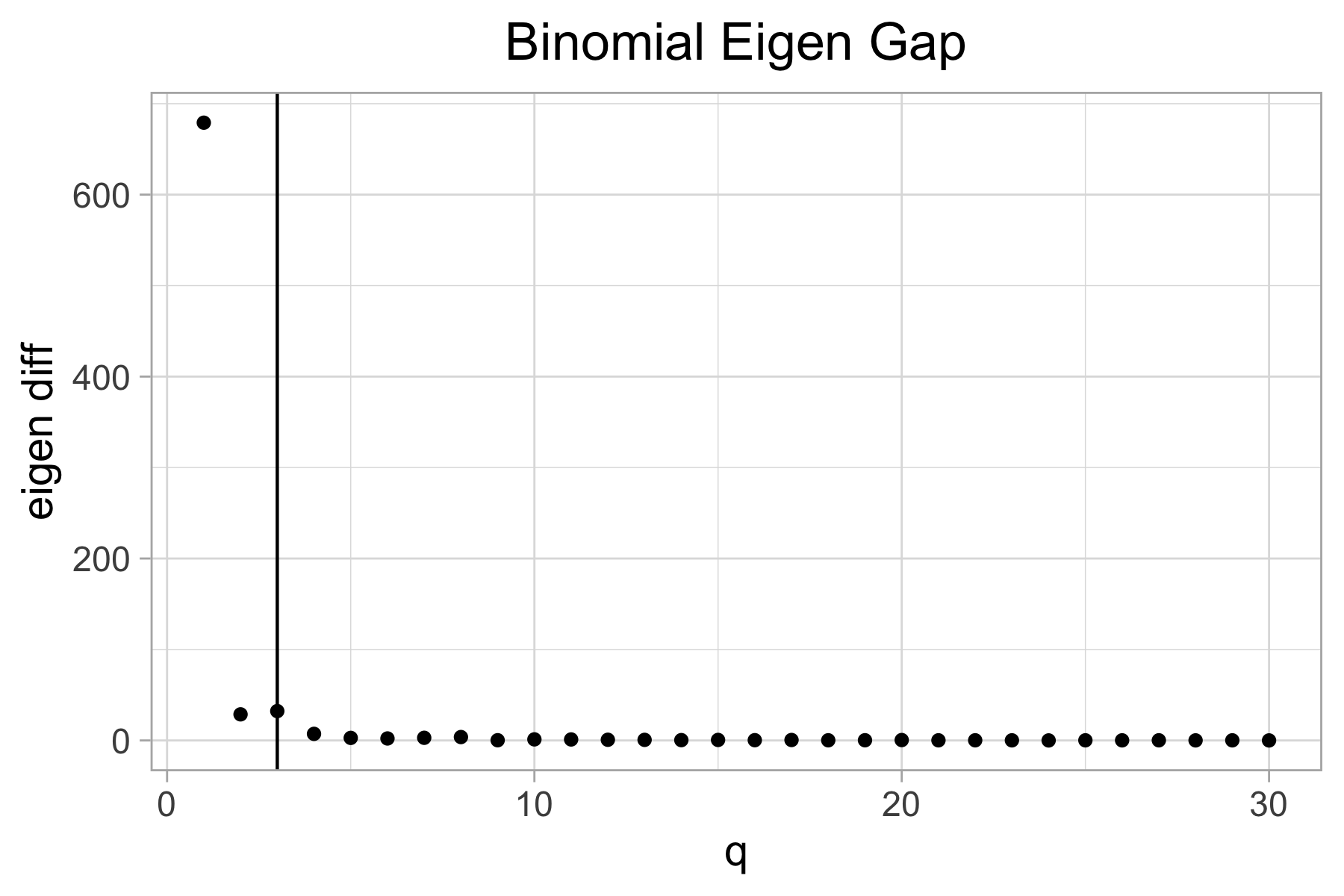}
\hspace{0.5cm}
\includegraphics[width=.45\textwidth]{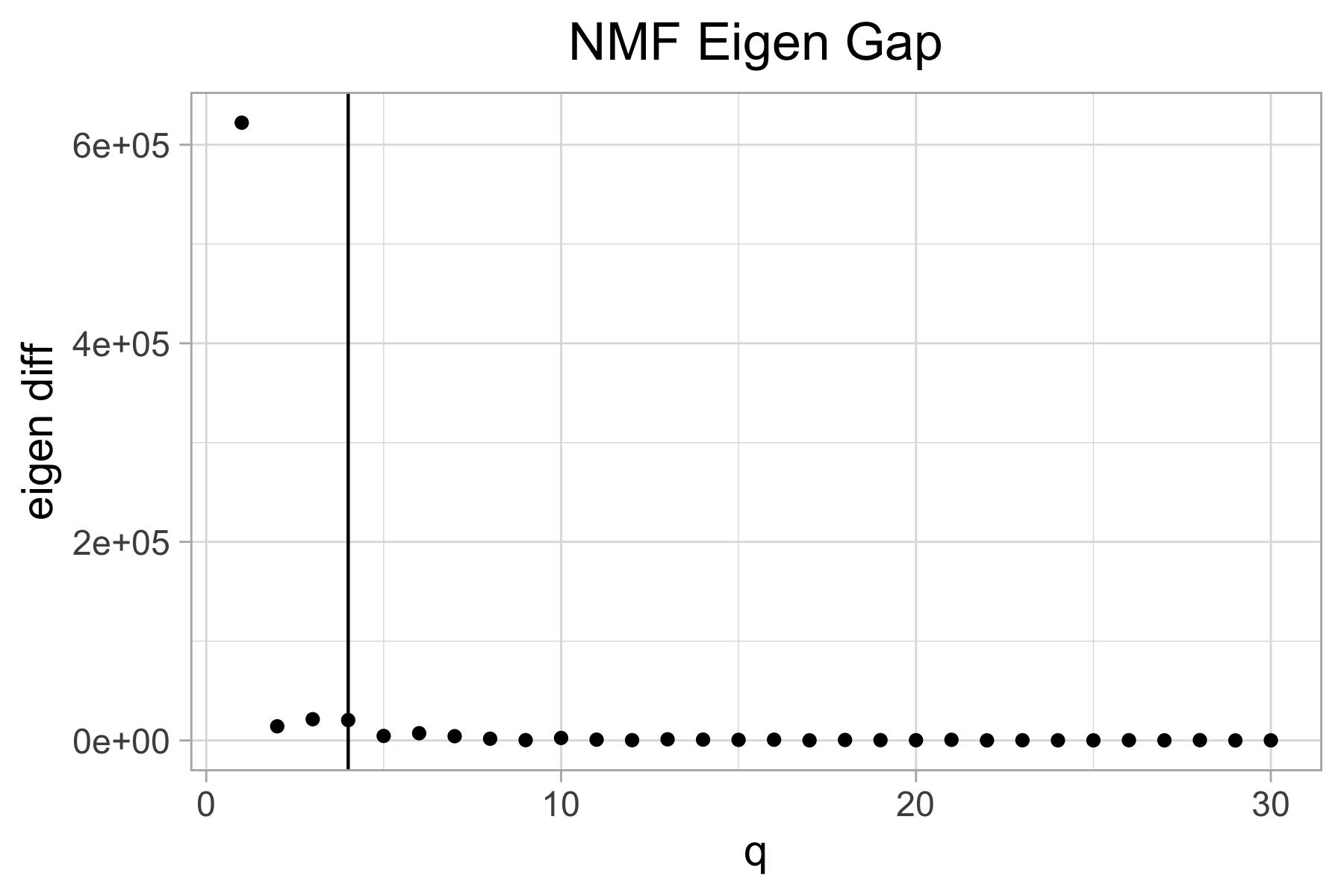}
\caption{\textbf{CBCL factorization eigengap}.}
\label{fig:cbcl_rank}
\end{figure}
\hfill
\begin{figure}[H]
\centering
\includegraphics[width=.33\textwidth,height=1.6in]{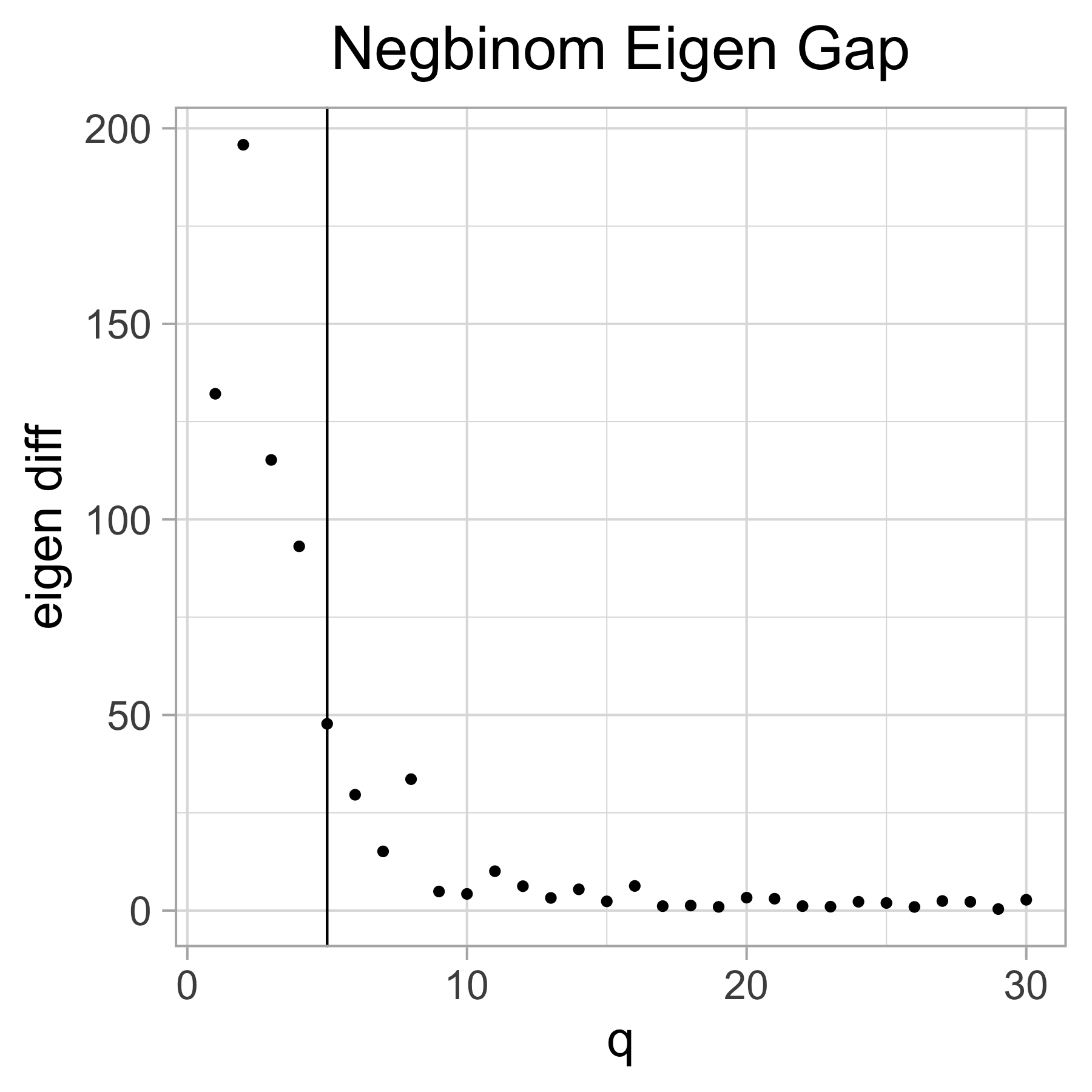}\hfill
\includegraphics[width=.33\textwidth,height=1.6in]{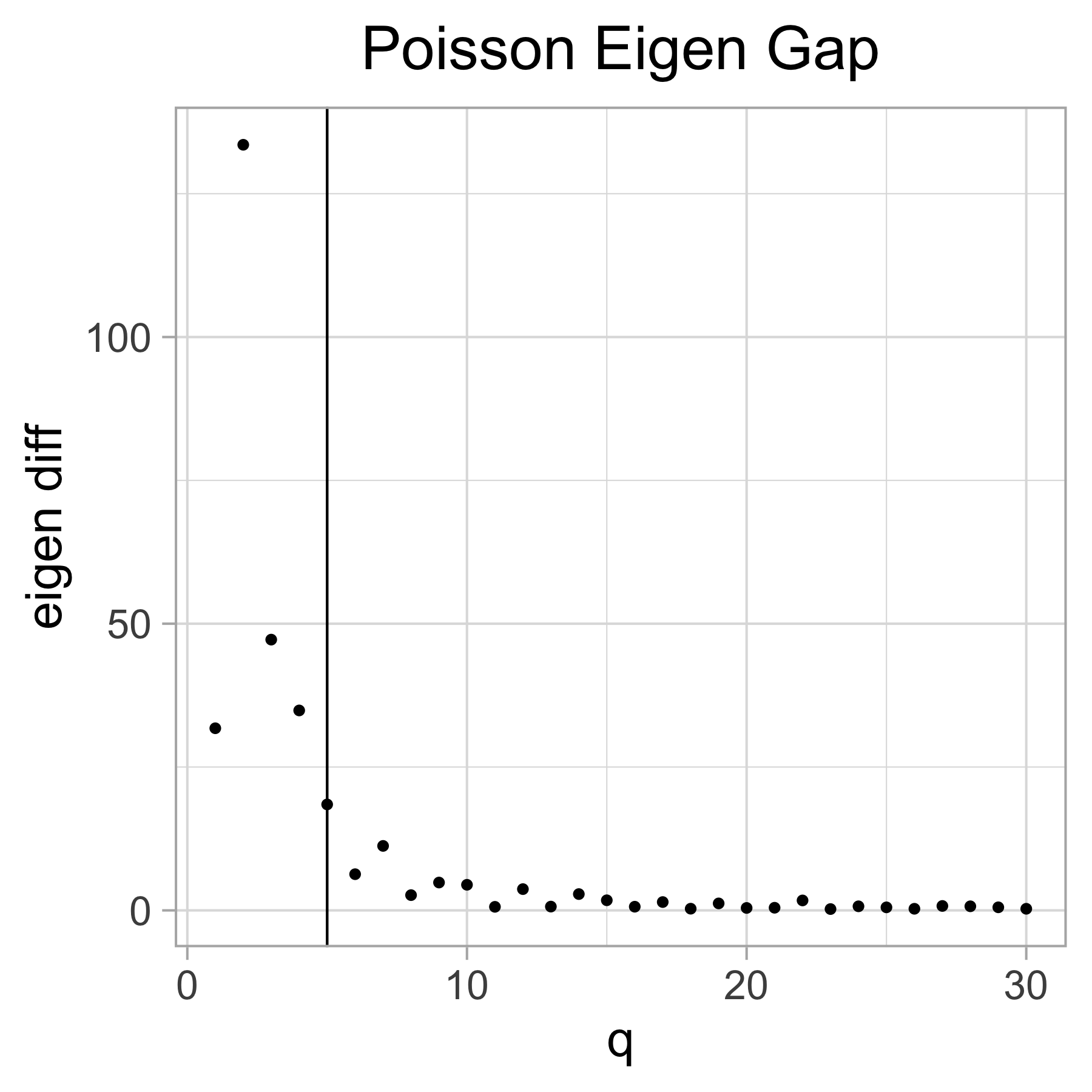}\hfill
\includegraphics[width=.33\textwidth, height=1.6in]{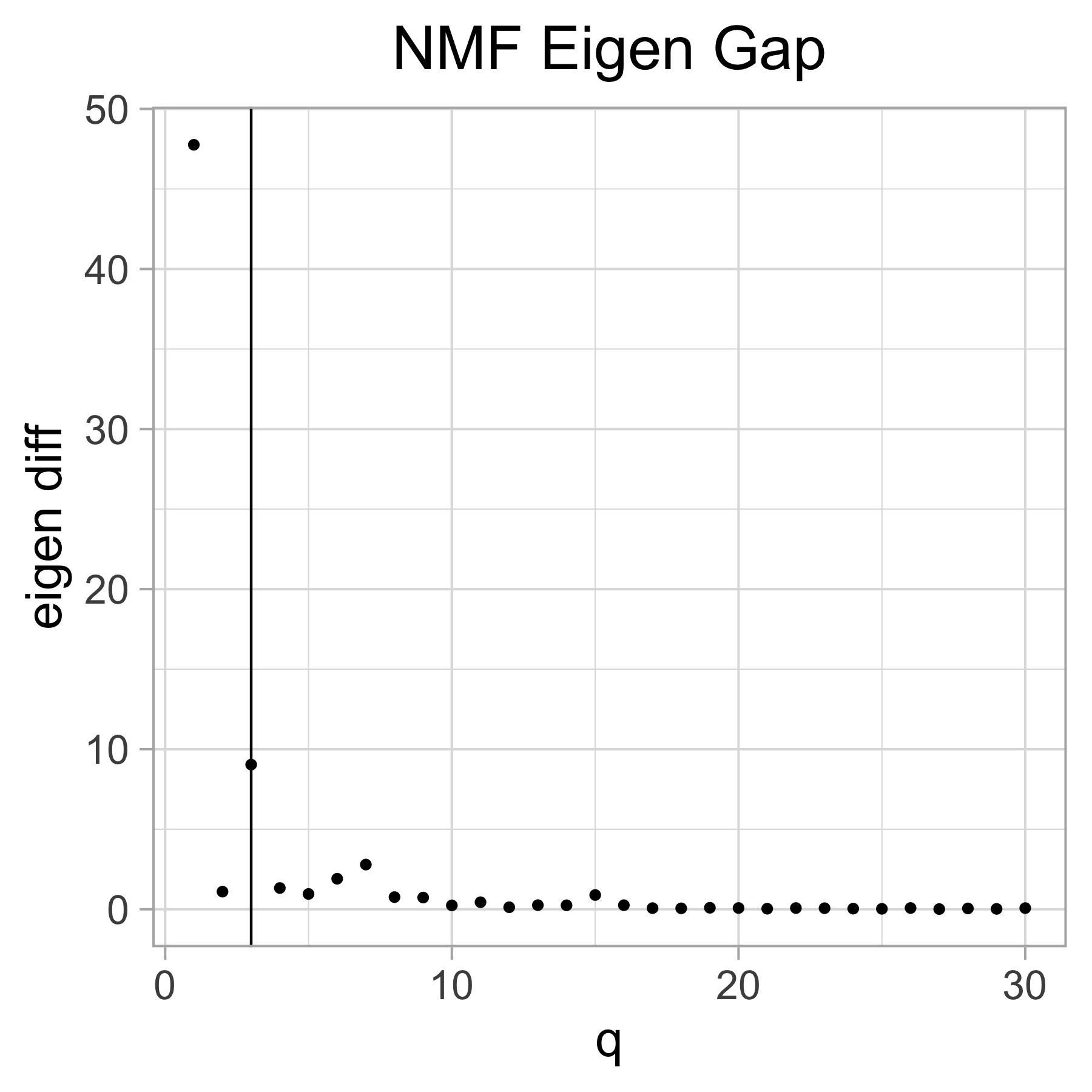}
\caption{\textbf{Keyword frequency factorization eigengap}.}
\label{fig:nlp_rank}
\end{figure}
\hfill
\begin{figure}[H]
\centering
\includegraphics[width=.49\textwidth]{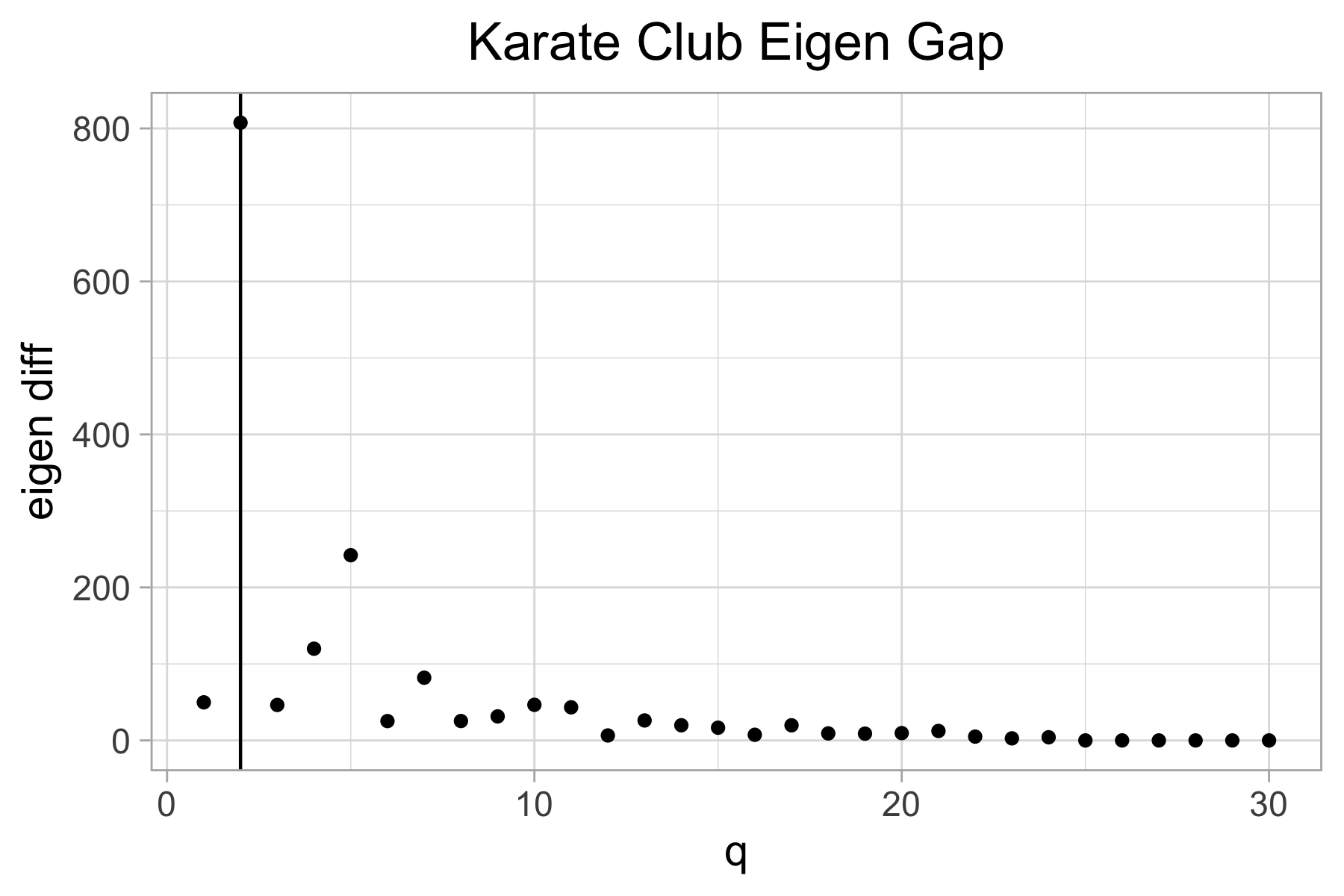}\hfill
\includegraphics[width=.49\textwidth]{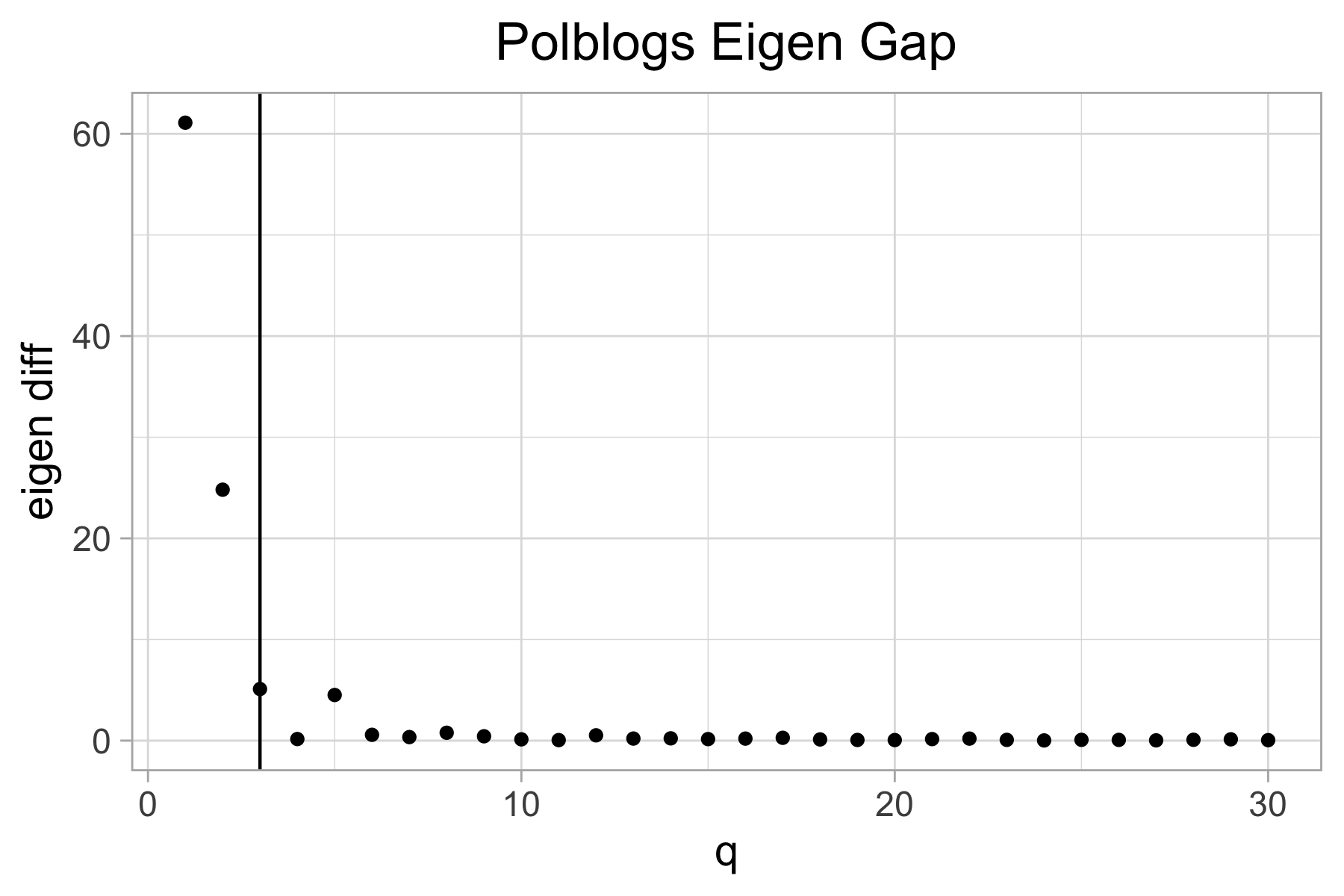}
\caption{\textbf{Network factorization eigengap}.}
\label{fig:network:rank}
\end{figure}
\hfill
\begin{figure}[h]
\centering
\includegraphics[width=.33\textwidth,height=1.6in]{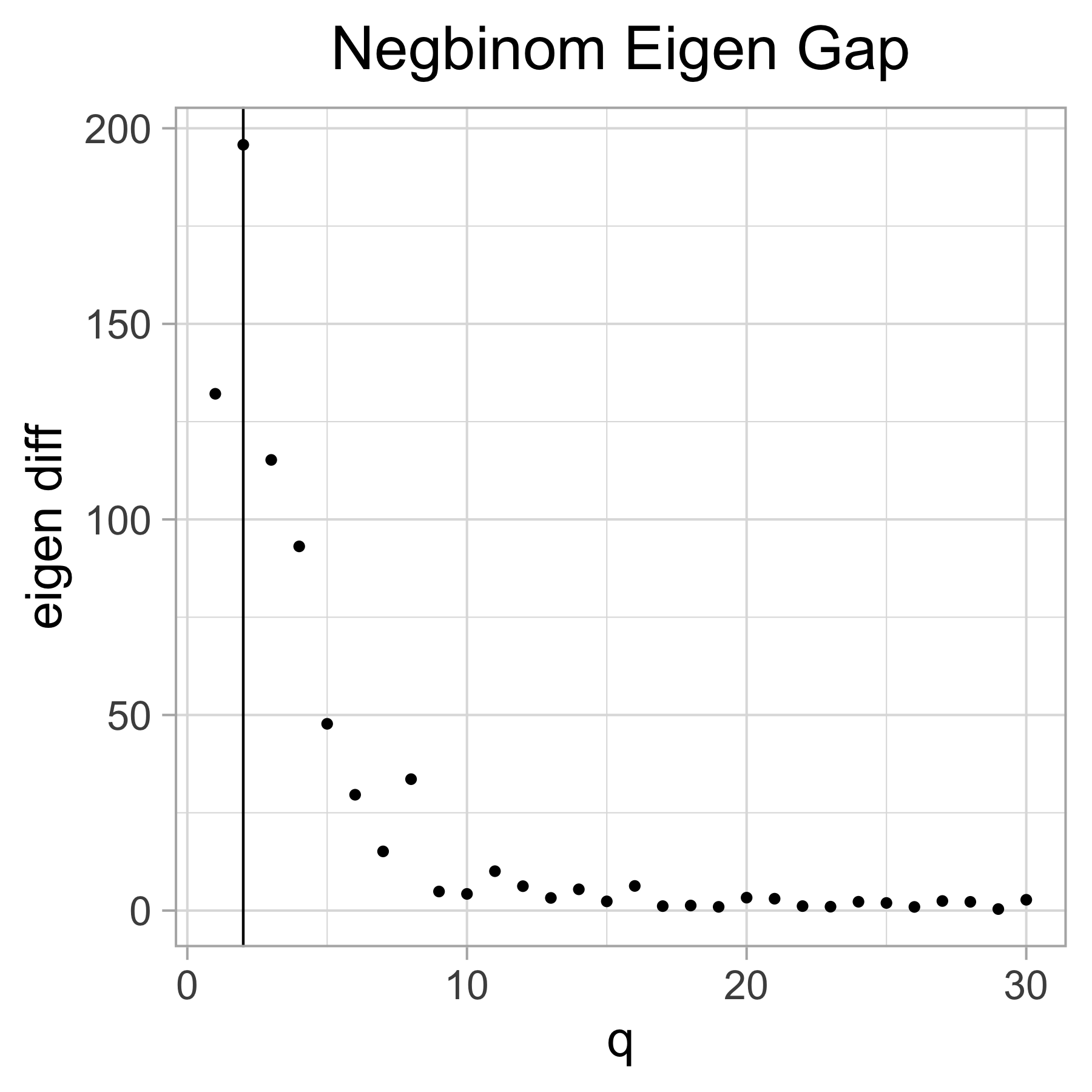}\hfill
\includegraphics[width=.33\textwidth,height=1.6in]{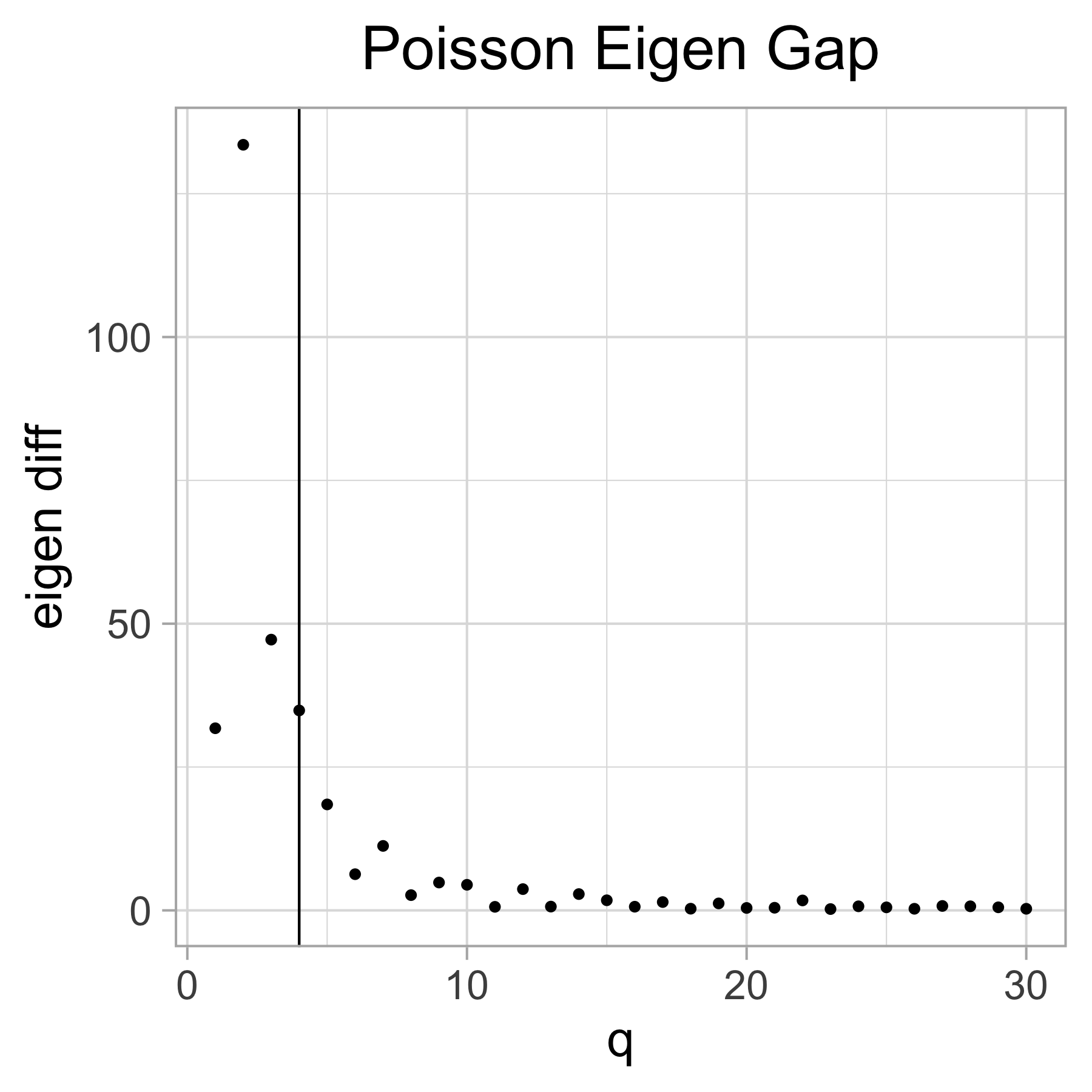}\hfill
\includegraphics[width=.33\textwidth,height=1.6in]{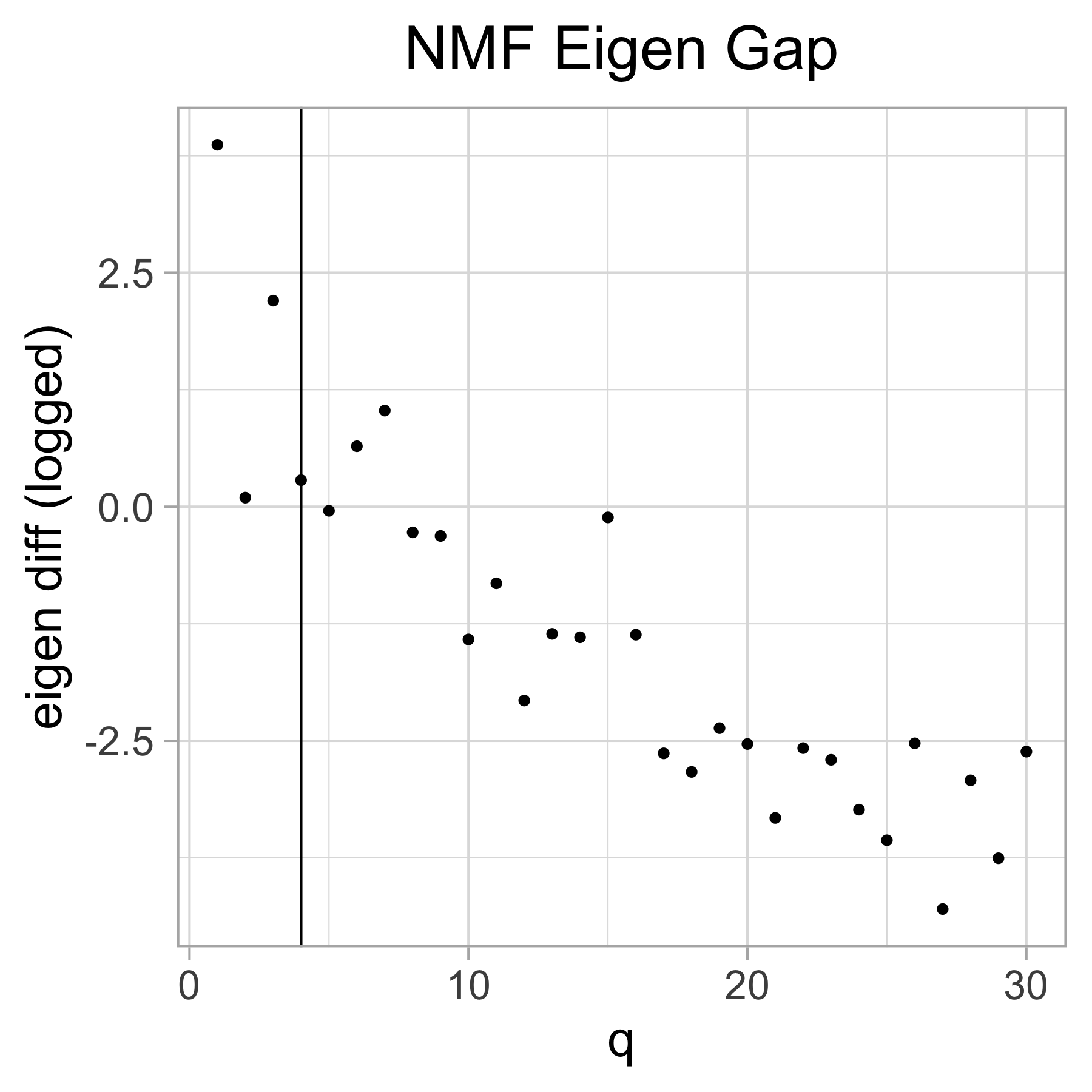}
\caption{\textbf{Leukemia factorization eigengap}.}
\label{fig:leukemia_rank}
\end{figure}
\hfill

%
\if1\online
{
\section{Theoretical Guarantees}

\subsection{Proof of Theorem~\ref{thm:consistency}}

The consistency proof is based upon Theorem~1 in~\citep{1999glmconsistency}
with some additional inequalities being imposed. We start with the following
lemma:
\begin{lemma} \label{le:inverse}
For smooth injection mapping $H(\cdot): \mathbb{R}^q \rightarrow \mathbb{R}^q$
with $H(x_0) = y_0$.
Define $B_{\delta}(x_0)= \{x \in \mathbb{R}^q, \|x-x_0\|_2\leq \delta\}$ and
$S_{\delta}(x_0) = \{x \in \mathbb{R}^q, \|x-x_0\|_2= \delta\}$,
then $\inf_{x\in S_{\delta}(x_0)}\|H(x) - y_0\|_2 \geq r$ implies:
\begin{itemize}
\item $B_r(y_0) \subseteq H(B_{\delta}(x_0)) $
\item $H^{-1}(B_r(y_0)) \subseteq  B_{\delta}(x_0)$
\end{itemize}
\end{lemma}
\begin{proof}
The proof is detailed in the appendix of \citep{1999glmconsistency}.
\end{proof}
\subsubsection{The first order conditions}
Denote $\eta_{ij} = \Lambda_i^\top V_j$ where $\Lambda_i$ is the i-th row
of factorized matrix $\Lambda$, $V_j^\top$ is the j-th column of
factorized matrix $V^\top$. As a notational convention, all vectors are
defined as column vectors with their inner products being scalar. Under an
exponential GLM setup, the log likelihood with parameter of interest
$\theta_{ij} = \theta(\eta_{ij})$,
dispersion parameter $a(\phi)$, mean function $\mu_{ij} = m(\eta_{ij})$ and
cumulant function $b_{ij} = b(\eta_{ij})$ could be written in the following
form:
\begin{equation}
\begin{aligned}
    l(\Lambda,V) &= \sum_{i=1}^n \sum_{j=1}^p l (\Lambda_i^\top V_j)\\
     & = \sum_{i=1}^n \sum_{j=1}^p \frac{1}{a(\phi)}  (X_{ij} \theta_{ij} - b(\theta_{ij}) + c(X_{ij}, \phi))
\end{aligned}
\end{equation}
After convergence of Algorithm~\ref{algo:dmf}, denote $l_{ij} = l(\eta_{ij})$.
We will have the following first order condition to be satisfied:
\begin{equation}
\label{eq:foc1}
\begin{aligned}
\frac{\partial l(\Lambda, V)}{\partial V_j^\top} &= \sum_{i=1}^n \frac{\partial l_{ij}}{\partial \theta_{ij}} \frac{\partial \theta_{ij}}{\partial \mu_{ij}} \frac{\partial \mu_{ij}}{\partial \eta_{ij}} \frac{\partial \eta_{ij}}{\partial V_{j}}=\mathbf{0}_q\\
\frac{\partial l(\Lambda, V)}{\partial \Lambda_i} &= \sum_{j=1}^p \frac{\partial l_{ij}}{\partial \theta_{ij}} \frac{\partial \theta_{ij}}{\partial \mu_{ij}} \frac{\partial \mu_{ij}}{\partial \eta_{ij}} \frac{\partial \eta_{ij}}{\partial \Lambda_{i}}=\mathbf{0}_q
\end{aligned}
\end{equation}
Specifically for the components, with $\mathcal{V}(\mu_{ij})$ being the variance function
with mean $\mu_{ij}$ as the scalar input, we have:
\begin{equation}
\begin{aligned}
\frac{\partial l_{ij}}{ \partial \theta_{ij}} &= \frac{1}{a(\phi)}(X_{ij} - b'(\theta_{ij}))\\
 \frac{\partial \theta_{ij}}{\partial \mu_{ij}} &= (\frac{\partial \mu_{ij}}{\partial \theta_{ij}})^{-1} = \frac{1}{b''(\theta_{ij})} = \frac{1}{\mathcal{V}(\mu_{ij})} \\
\frac{\partial \mu_{ij}}{\partial \eta_{ij}} &= (\frac{\partial \eta_{ij}}{\partial \mu_{ij}})^{-1} = \frac{1}{g'(\eta_{ij})} \\
\frac{\partial \eta_{ij}}{\partial V_{j}^\top} &= \Lambda_i^\top, \quad
\frac{\partial \eta_{ij}}{\partial \Lambda_{j}} = V_j^\top
\end{aligned}
\end{equation}
Substituting into~\eqref{eq:foc1}, we have:
\begin{equation}
\begin{aligned}
\frac{\partial l(\Lambda, V)}{\partial V_j} &= \sum_{i=1}^n
\frac{1}{a(\phi)}(X_{ij} - b'(\theta_{ij}))\frac{1}{\mathcal{V}(\mu_{ij}) g'(\mu_{ij})}
\Lambda_i^\top=\mathbf{0}_q \\
\frac{\partial l(\Lambda, V)}{\partial \Lambda_i} &= \sum_{i=j}^p
\frac{1}{a(\phi)}(X_{ij} - b'(\theta_{ij}))\frac{1}{\mathcal{V}(\mu_{ij}) g'(\mu_{ij})}
V_j^\top=\mathbf{0}_q
\end{aligned}
\end{equation}
Observe that $\theta_{ij} = \eta_{ij}$ under canonical link. Thus
$g'(\mu_{ij}) = \frac{1}{\mathcal{V}(\mu_{ij})}$, which further gives us the normal equation
from the least square estimation for fixed $\Lambda_{i}$ or $V_{j}, \forall i = 1...n, j = 1...p$:

\begin{equation}
\label{eq:foc}
\begin{aligned}
\frac{\partial l(\Lambda, V)}{\partial V_j} &= \sum_{i=1}^n (X_{ij} -
m(\Lambda_i^\top V_j))\Lambda_i^\top=\mathbf{0}_q \\
\frac{\partial l(\Lambda, V)}{\partial \Lambda_i} &= \sum_{j=1}^p (X_{ij} -
m(\Lambda_i^\top V_j))V_j^\top=\mathbf{0}_q
\end{aligned}
\end{equation}
As we assume that there exist $\eta$ of rank $q$ that generates the data and maximize the likelihood, our DMF algorithm uses the data initialization to find the corresponding estimate $\hat{\eta}$ as the global maximum of the likelihood.
However, the solution of $(\Lambda, V)$ is still not unique up to a rotational matrix until we impose an identifiability constraints.
In together with our identifiability constraints and the convexity of the exponential log-likelihood, we thus observe that $\hat{\Lambda}_i$ is unique conditional on $\hat{V}$ and that $\hat{V}_{j}$ is unique conditional on $\hat{\Lambda}$. At final algorithm convergence, our estimators are thus designed to maximize the likelihood and to satisfy the F.O.C in Eq~\eqref{eq:foc}:
\begin{equation}
\label{eq:foc2}
    \begin{aligned}
    (\hat{\Lambda}^\top \hat{\Lambda})^{-1} \sum_{i=1}^n \hat{\Lambda}_i^\top(X_{ij} -
    m(\hat{\Lambda}_i^\top \hat{V}_j)) &=\mathbf{0}_q \\
    (\hat{V}^{\top}\hat{V})^{-1} \sum_{j=1}^p \hat{V}_j^\top(X_{ij} -
    m(\hat{\Lambda}_i^\top \hat{V}_j)) &=\mathbf{0}_q
    \end{aligned}
\end{equation}
\subsubsection{Sketch of the proof}
Intuitively speaking, if either $\hat{\Lambda}$ or $\hat{V}$ is known in Equation (\ref{eq:foc2}), the problem reduces to a multiple generalized linear regression model~\citep{1999glmconsistency}. When both of the estimators are unknown, it is thus hopeful to put some constraints on the space of our estimators $\hat{\Lambda}, \hat{V}$ such that the consistency is still attained. In this proof,  we first demonstrate that by restricting the $L_2$ norm of  estimator $V$ with the orthogonal design ($\hat{V}^\top \hat{V} = I_q$), the estimator of $\hat{\Lambda}_i, \forall i = 1, \cdots, n$ is strongly consistent. After which, we can derive the consistency property of estimator $\hat{V}_j, \forall j = 1,\cdots,p$ in a similar manner.
\\
We start by defining an auxiliary  $q$ by $1$ vector $\tilde{G}(\cdot): \mathbb{R}^{q} \rightarrow \mathbb{R}^{q}$
\begin{equation}
    \begin{aligned}
    \label{eq:gtildedmf}
\tilde{G}_{\hat{\Lambda}_n}( \hat{V}_{j}) &=   (\hat{\Lambda}^\top \hat{\Lambda})^{-1} \sum_{i=1}^n \hat{\Lambda}_i^\top [
    m(\hat{\Lambda}_i^\top \hat{V}_j)- m(\Lambda_i^\top V_j)]  \\
\tilde{G}_{\hat{V}_p}(\hat{\Lambda}_{i}) &= (\hat{V}^{\top}\hat{V})^{-1} \frac{1}{\sqrt{p}}\sum_{j=1}^p \hat{V}_{j}[
 m(\hat{\Lambda}_i^\top \hat{V}_j) - m(\Lambda_i^\top V_j) ]\\
    \end{aligned}
\end{equation}
Notice that the $\tilde{G}_{\hat{\Lambda}_n}( \hat{V}_{j}), \tilde{G}_{\hat{V}_p}(\hat{\Lambda}_{i}) $ vectors are defined in a similar manner to the normal equation in Eq~\eqref{eq:foc2} with $X_{ij}$ changed to the expectation $m(\Lambda_i^\top V_j)$. In a comparison to GLM setup that has either $\hat{\Lambda} = \Lambda_n \in \mathbb{R}^{n\times q} $ or $ \hat{V} = V_p \in \mathbb{R}^{p\times q}$ fixed as observed data, \citep{1999glmconsistency} justified the consistency through some inequality relationship based upon the following vectors:
\begin{equation}
\label{eq:gtilde}
    \begin{aligned}
    \tilde{G}_{\textcolor{red}{\Lambda_n}}( \hat{V}_{j}) &=   (\hat{\Lambda}^\top \hat{\Lambda})^{-1} \sum_{i=1}^n \hat{\Lambda}_i^\top [m(\textcolor{red}{\Lambda_i^\top} \hat{V}_j) - m(\Lambda_i^\top V_j)] \\
\tilde{G}_{\textcolor{red}{V_p}}( \hat{\Lambda}_{i}) &=   (\hat{V}^\top \hat{V})^{-1} \frac{1}{\sqrt{p}} \sum_{j=1}^p \hat{V}_j [m(\hat{\Lambda}_i^\top \textcolor{red}{ V_j}) - m(\Lambda_i^\top V_j)
    ]
    \end{aligned}
\end{equation}
For the precise consistency inequality with $\tilde{G}_{\textcolor{red}{V_p}}( \hat{\Lambda}_{i})$ as an example, 
$$\|\tilde{G}_{\textcolor{red}{V_p}}( \hat{\Lambda}_{i})\|_2^2  \underbrace{\geq}_{\textcolor{blue}{ \text{Inequality}\ \ast}}  \tilde{Z}^2_{V_\delta}(\Lambda_{i})\|\Lambda_i -\hat{\Lambda}_i\|_2^2$$
implies the following from Lemma \ref{le:inverse}
$$\hat{\Lambda}_i = \tilde{G}^{-1}_{\hat{V}_p}(\{b: \|b\|_2 \leq Z^2_{V_\delta}(\Lambda_{i}) \})
\subseteq \{b: \|b - \Lambda_{i}\|_2 \leq \delta_{\Lambda_{i}} \}$$
Our detailed proof then follows from building inequality relationship between Eq~\eqref{eq:gtildedmf} and Eq~\eqref{eq:gtilde}. After which, consistency inequality in \citep{1999glmconsistency} can be readily applied.  To facilitate the understanding, we focus on a summary on the proof in this section while postponing the inequality proof in Section~\ref{sec:inequality}.

\subsubsection*{1). The Inequalities.}
With $\|\cdot\|_2$ denote the spectral norm for matrix and euclidean norm(inner product) for vector, we will show that for large $p$, the following inequality holds true $\forall\ i = 1, \cdots n$:
$$\|\tilde{G}_{\hat{V}_p}( \hat{\Lambda}_{i})\|_2^2 \underbrace{\geq}_{\textcolor{blue}{\text{Inequality}\ 	\dagger }} \|\tilde{G}_{\textcolor{red}{V_p}}( \hat{\Lambda}_{i})\|_2^2  \underbrace{\geq}_{\textcolor{blue}{ \text{Inequality}\ \ast}}  \tilde{Z}^2_{V_\delta}(\Lambda_{i})\|\Lambda_i -\hat{\Lambda}_i\|_2^2$$
where $0<\tilde{Z}_{V_\delta}(\Lambda_{i})< \infty$ is a positive and finite constant that introduced later in our detailed inequality proof. 
\noindent If the above inequality holds, we can readily apply Lemma \ref{le:inverse} to obtain:
\begin{equation}
\label{eq:cvg_order2}
\begin{aligned}
\hat{\Lambda}_i \in \tilde{G}^{-1}_{\hat{V}_p}(\{b: \|b\|_2 \leq Z^2_{V_\delta}(\Lambda_{i}) \})
\subseteq \{b: \|b - \Lambda_{i}\|_2 \leq \delta_{\Lambda_{i}} \}
\end{aligned}
\end{equation}
where $\delta_{\Lambda_i}>0$ is an  arbitrary small positive scalar. That is, we can find estimator $\hat{\Lambda}_i$ as the pre-image of injection $\tilde{G}_{\hat{V}_p}(\cdot)$. The obtained estimator is sufficiently close to the true parameter $\Lambda_i$ with euclidean distance controlled by $\delta_{\Lambda_i}$.
\subsubsection*{ 2). The existence of $b \in \mathbb{R}^{q}$.}
\label{subsubsec:exist}
Based upon the definition of the fitted residual $e_{ij} = X_{ij} - m(\Lambda^\top_{i}V_{j})$,
we can construct an $\mathbb{R}^q$ vector  $ a_{i}(\hat{V}_p) =\sum_{j=1}^p (\hat{V}_{j} \hat{V}_{j}^\top)^{-1} \frac{1}{\sqrt{p}}\sum_{j=1}^p \hat{V}_{j} e_{ij} $. By the result of our condition C2, we have 
$$\lim\limits_{p\rightarrow \infty} a_{i}(\hat{V}_p)  = \mathbf{0}_q \text{ a.s. } \forall i =1,\ldots, n $$
The result holds for arbitrary $n$. Consequently, for large $p$ and sufficiently small $\delta_{\Lambda_i}$
\begin{itemize}
    \item $\hat{\Lambda}_i = \tilde{G}^{-1}_{\hat{V}_p}(a_{i}(\hat{V}_p))$ exists.
    \item $\hat{\Lambda}_i = \tilde{G}^{-1}_{\hat{V}_p}(a_{i}(\hat{V}_p))$ belong to $\{\Lambda_{i}: \|\hat{\Lambda}_{i} - \Lambda_{i}\|_2 \leq \delta_{\Lambda_{i}}\}$.
\end{itemize}
That is, the pre-image of $\tilde{G}^{-1}_{\hat{V}_p} (\cdot)$ with argument $a_{i}(\hat{V}_p) =\sum_{j=1}^p (\hat{V}_{j} \hat{V}_{j}^\top)^{-1} \frac{1}{\sqrt{p}} \sum_{j=1}^p \hat{V}_{j} e_{ij}$ is a consistent estimator as $p \rightarrow \infty$. The result again holds for all $i\in 1,\ldots, n$. 
\subsubsection*{3). The relationship to the $\hat{\Lambda}_i$ estimator}
\label{subsubsec:unique}
To justify that $\hat{\Lambda}_i$ is uniquely our DMF estimator, we plug $\hat{\Lambda}_i = \textcolor{red}{\tilde{G}^{-1}_{\hat{V}_p}(a_i(\hat{V}_p))}$ into the convergent condition of our estimator in Eq~\eqref{eq:foc2} with a harmless multiplication of $\frac{1}{\sqrt{p}}$.
\begin{equation}
\label{eq:unique}
\begin{aligned}
        &(\hat{V}^\top \hat{V})^{-1} \sum_{j=1}^p \hat{V}_j^\top(
    m(\textcolor{red}{\hat{\Lambda}_i^\top} \hat{V}_j)- X_{ij}) \\
    &=  \sum_{j=1}^p \hat{V}_j^\top(
    m(\textcolor{red}{\tilde{G}^{-1}_{\hat{V}_p}(a_i(\hat{V}_p))^\top}\hat{V}_j  ) - X_{ij})\\
        &=  \sum_{j=1}^p \hat{V}_j^\top( m( \textcolor{red}{\tilde{G}^{-1}_{\hat{V}_p}(a_i(\hat{V}_p))^\top} \hat{V}_j   )-
        m(\Lambda_i^\top V_j) + m(\Lambda_i^\top V_j)- X_{ij}
    )\\
        &=   \sum_{j=1}^p \hat{V}_j^\top(
        m(\Lambda_i^\top V_j)- X_{ij}) +    \sum_{j=1}^p \hat{V}_j^\top \big(
   m( \textcolor{red}{\tilde{G}^{-1}_{\hat{V}_p}(a_i(\hat{V}_p))^\top}\hat{V}_j)  - m(\Lambda_i^\top V_j) \big) \\
        &=- \sqrt{p}a_{i}(\hat{V}_p) + \sqrt{p}\tilde{G}_{\hat{V}_p} \big(\textcolor{red}{\tilde{G}^{-1}_{\hat{V}_p}(a_i(\hat{V}_p))}\big) = \mathbf{0}_q
\end{aligned}
\end{equation}
where we start the substitution of $\hat{\Lambda}_i$ from the first equality. The second equality added and subtracted $m(\Lambda_i^\top V_j)$. The third equality used the definition of $\tilde{G}^{-1}_{\hat{V}_p}(\cdot)$ as defined in Eq~\eqref{eq:gtildedmf}. 

As shown in the derivation of Eq \eqref{eq:unique}, the strong consistent estimator $\hat{\Lambda}_i = \tilde{G}^{-1}_{\hat{V}_p}(a_i(\hat{V}_p))$ satisfies the F.O.C. equation. Additionally, the estimator is unique conditional on $\hat{V}_p$ satisfying the orthogonality constraint  as we have shown in the argument below Eq \eqref{eq:foc}. As a result, the F.O.C uniquely identifies the almost surely converged estimator $\hat{\Lambda_i} = \tilde{G}^{-1}_{\hat{V}_p}(a_i(\hat{V}_p))$, which is the output of our Algorithm \ref{algo:dmf}:
$$\hat{\Lambda}_i \xrightarrow[p\rightarrow \infty]{a.s} \Lambda_i, \forall i = 1\cdots n$$
We then continue the proof with $\hat{V}_j$ consistency for all $j = 1,\ldots, p$.
\subsubsection*{4). $\hat{V}_j$ convergence}
\noindent By a similar argument, we will show that under large $n, p$, we have $\forall j = 1, \cdots, p $
$$\|\tilde{G}_{\hat{\Lambda}_n}( \hat{V}_{j})\|_2^2 \underbrace{\geq}_{\textcolor{blue}{\text{Inequality}\ 	\Diamond}} \|\tilde{G}_{\textcolor{red}{\Lambda_n}}( \hat{V}_{j})\|_2^2 \underbrace{\geq}_{\textcolor{blue}{\text{Inequality}\ \star}} Z^2_{\Lambda_\delta}(V_{j})\|V_j -\hat{V}_j\|_2^2$$
where $0<Z_{\Lambda_\delta}(V_{j}) <\infty$ is a positive finite constant introduced in the detailed inequality proof. For an arbitrary small positive scalar $\delta_{V_j}>0$, Lemma~\ref{le:inverse} gives us:
\begin{equation}
\label{eq:cvg_order2v}
\begin{aligned}
\tilde{G}^{-1}_{\hat{\Lambda}_n}(\{b: \|b\|_2 \leq Z^2_{\Lambda_\delta}(V_{j}) \})
\subseteq \{b: \|b - V_j\|_2 \leq \delta_{V_{j}} \}
\end{aligned}
\end{equation}
Notice that under conditions~C1 and~C2, with large $n$, we have the vector $a_{j}(\Lambda_n) = (\Lambda^\top \Lambda)^{-1} \Lambda^\top e_j$ converge almost surely by \citep{1999glmconsistency}. Precisely, the result holds $\forall j= 1,\cdots,p $ with its order given by:
\begin{equation}
\label{eq:cvg_order3}
\begin{aligned}
a_{j}(\Lambda_n)  &
= o\Bigg( \frac{(\log \lambda_{\text{min}} ((\sum_{i=1}^n
\Lambda_{i}{\Lambda_{i}}^\top)^{-1}))^{\frac{1+\delta}{2}}}{\lambda_{\text{max}}
((\sum_{i=1}^n \Lambda_{i}{\Lambda_{i}}^\top)^{-1})^{1/2}} \Bigg)\quad a.s.
\end{aligned}
\end{equation}
From the previous result, we obtained $\hat{\Lambda}_i \xrightarrow[p\rightarrow \infty]{a.s} \Lambda_i, \forall i = 1\cdots n$. As a sum of converged vector outer product, $a_{j} (\hat{\Lambda}_{n,p})= (\hat{\Lambda}^\top \hat{\Lambda})^{-1} \hat{\Lambda}^\top e_j $ thus
converge almost surely to $a_{j}(\Lambda_n) = (\Lambda^\top \Lambda)^{-1} \Lambda^\top e_j$ through continuous mapping theorem. Then by taking the limit of $\lim\limits_{n \rightarrow \infty} a_{j} (\Lambda_n)$, we justified that the $\mathbb{R}^q$ vector $a_{j}(\hat{\Lambda}_{n,p})$ converges a.s. for large $n,p$ and the results hold for arbitrary $j$.
\\
Repeat argument of \textit{existence}~\ref{subsubsec:exist} and \textit{uniqueness}~\ref{subsubsec:unique} for estimator
\begin{equation}
\label{eq:uniqueV}
\begin{aligned}
\hat{V}_j = \tilde{G}^{-1}_{\hat{\Lambda}_n}( a_{j} (\hat{\Lambda}_{n,p}))
\end{aligned}
\end{equation}
we can obtain for arbitrary j:
$$\hat{V}_j \xrightarrow[]{a.s} V_j, \text{ as } n.p \rightarrow\infty \text{ given arbitrary j}$$
\subsubsection*{5). Some analysis on the rate of the convergence}
Observe we have Lipschitz continuous function $\tilde{G}_{\Lambda_n}(\cdot)$ such that $\tilde{G}_{\hat{\Lambda}_n}(\hat{V}_{j}) =  a_{j} (\hat{\Lambda}_{n,p})$ is implied in~\eqref{eq:uniqueV}. The order of convergence
is thus established to be the same as the order of convergence for
$ a_{j} (\hat{\Lambda}_{n,p})$, which converges to $a_j(\Lambda_n)$. Deriving the convergence rate for $a_{j}(\hat{\Lambda}_{n,p})$
is however challenging as the convergence rate for $\hat{\Lambda}_{n,p} \rightarrow \Lambda_n$ is challenging to derive
without assumptions on the convergence rate of $e_{ij}$ (Condition C2).
The GLM regression problem however provides us an intuitive lower bound of the convergence rate:

Conditional on $\Lambda_i$ being known, a theoretical rate of convergence by treating $\hat{V}_j$ as a regression problem is
shown in \citep{1979lmconsistency}. Specifically, one can choose $\delta_{V_{j}} = c
\mathop{\text{max}}\limits_{j=1,\ldots,p} \{t_{jj}^{n} |\log
t_{jj}^n|^{1+\delta} \}^{1/2}$
in~\eqref{eq:cvg_order2} where
\begin{itemize}
    \item $c>0$ is a fixed positive number
    \item $\delta>0$ can be made arbitrarily small
    \item  $t_{ij}^n$ is the $j$-th diagonal element of matrix $(\sum_{i=1}^n \Lambda_i \Lambda_i^\top)^{-1}$.
\end{itemize}
Conditional on $\Lambda_i$ being known, the setup give us the following order on $a_{\Lambda_n}$:
\begin{equation}
\label{eq:cvg_order4}
\begin{aligned}
a_{\Lambda_n}  &= \Bigg(\sum_{i=1}^n \Lambda_i \Lambda_i^\top\Bigg)^{-1}
\sum_{i=1}^n \Lambda_i e_i\\
&= o\Bigg( \frac{(\log \lambda_{\text{min}} ((\sum_{i=1}^n
\Lambda_{i}{\Lambda_{i}}^\top)^{-1}))^{\frac{1+\delta}{2}}}{\lambda_{\text{max}}
((\sum_{i=1}^n \Lambda_{i}{\Lambda_{i}}^\top)^{-1})^{1/2}} \Bigg)\quad a.s.\\
&= o\Big(\mathop{\text{max}}\limits_{j=1,\ldots,p} \{t_{jj}^{n} |\log
t_{jj}^n|^{1+\delta} \}^{1/2}\Big)\quad a.s.\\
\end{aligned}
\end{equation}
In practice, since we do not have access to
$\Lambda$ and $V$, the rate of convergence should be larger than we obtained
in~\eqref{eq:cvg_order4}. Obtaining is beyond the focus of this article especially
considering the the missing of any consistency justifications for EPCA from the literature.
We thus encourage interested
readers to proceed with theoretical refinements from here.
\subsubsection{Detailed proof on the Inequalities}
\label{sec:inequality}
\subsubsection*{1). Detailed proof on \textcolor{blue}{\text{Inequality $\ast$}} and \textcolor{blue}{\text{Inequality $\star$}}.}
\begin{proof}
The lower bound is based on~\citep{1999glmconsistency}.
For any $\delta > 0$, denote
$Z_{\Lambda_\delta}(V_{j}) = \inf_{i, \|V_{j} -
\hat{V}_{j}\| \leq \delta} m' ( \Lambda_{i}^\top \tilde{V}_{j})
$. $Z_{\Lambda_\delta}(V_{j})$ is great than 0 due to the boundness of $\Lambda_i$ and continuity of  function $m(\cdot)$, then
observe that by the mean value theorem:
for arbitrary vectors $V^1_{j}, V^2_{j}$ and $\Lambda^1_{i},
\Lambda^2_{i}$, there exist $\tilde{\Lambda}_{i} \in (\Lambda^1_{i},
\Lambda^2_{i})$ and $\tilde{V}_{j} \in (V^1_{j}, V^2_{j})$ that satisfy:
\begin{equation}
\label{eq:mean_value}
\begin{aligned}
G_{\Lambda_n}(V^1_{j}) - G_{\Lambda_n}(V^2_{j})&= \sum_{i=1}^n  m' ( \Lambda_{i}^\top \tilde{V}_{j})
\Lambda_{i}\Lambda_{i}^\top (V^1_{j}- V^2_{j})\\
G_{V_p}(\Lambda^1_{i}) - G_{V_p}(\Lambda^2_{i})&= \sum_{i=1}^n  m' ( \tilde{\Lambda}_{i}^\top V_{j})
V_{j}V_{j}^\top (\Lambda^1_{i}- \Lambda^2_{i})
\end{aligned}
\end{equation}
Now, expanding the Euclidean norm of $\tilde{G}_{\Lambda_{n}}(\hat{V}_{j})$:
\begin{equation}
\begin{aligned}
\|\tilde{G}_{\Lambda_n}(\hat{V}_{j})\|_2^2 &=
\Bigg\| \Bigg(\sum_{i=1}^n \hat{\Lambda}_{i}{ \hat{\Lambda}_{i}}^\top\Bigg)^{-1}
[G_{\Lambda_n}(V_{j}) - G_{\Lambda_n}(\hat{V}_{j})]\Bigg\|_2^2\\
& = (V_{j} - \hat{V}_{j})^\top \sum_{i=1}^n  m' ( \Lambda_{i}^\top \tilde{V}_{j}) \Lambda_{i}{\Lambda_{i}}^\top
\Bigg(\sum_{i=1}^n \Lambda_{i}{\Lambda_{i}}^\top\Bigg)^{-2}\\
& \quad \times \sum_{i=1}^n  m' ( \Lambda_{i}^\top \tilde{V}_{j})  \Lambda_{i}{\Lambda_{i}}^\top (V_{j} - \hat{V}_{j})\\
&\geq (V_{j} - \hat{V}_{j})^\top \sum_{i=1}^n m' ( \Lambda_{i}^\top \tilde{V}_{j}) \Lambda_{i}{\Lambda_{i}}^\top
\Bigg(\sum_{i=1}^n \frac{ m' ( \Lambda_{i}^\top \tilde{V}_{j})}{Z_{\Lambda_{\delta}(V_{j})}}\Lambda_{i}
{\Lambda_{i}}^\top\Bigg)^{-2}\\
& \quad \times \sum_{i=1}^n  m' ( \Lambda_{i}^\top \tilde{V}_{j}) \Lambda_{i}{\Lambda_{i}}^\top
(V_{j} - \hat{V}_{j})\\
&= Z^2_{\Lambda_{\delta}}(V_{j}) \|V_{j} - \hat{V}_{j}\|^2
\end{aligned}
\end{equation}
where the first and only equality holds due to the mean value theorem
in~\eqref{eq:mean_value} and due to the expansion the $L_2$ norm. The last and only inequality
arises from the definition of $Z_{\Lambda_{\delta}}$.
For the other lower bound, we could simply adopt similar strategy to derive:
 \begin{equation}
     \|\tilde{G}_{V_p}(\hat{\Lambda}_{i})\|_2^2  \geq \inf_{j, \|\Lambda_{i} -
\hat{\Lambda}_{i}\| \leq \delta}  \frac{1}{\sqrt{p}} m' ( \tilde{\Lambda}_{i}^\top V_{j}) \|\Lambda_i -\hat{\Lambda}_i \|_2
\label{eq:glm_L_consis}
 \end{equation}
\end{proof}
\subsubsection*{2). Detailed proof on the \textcolor{blue}{$\text{Inequality}\ 	\dagger $} }
\begin{proof}
First, we denote
\begin{itemize}
    \item $W_{V_p}^\top =
    m(\hat{\Lambda}_i V^\top)- m(\Lambda_i V^\top)  \in \mathbb{R}^p$
    \item $W_{\hat{V}_p}^\top = m(\hat{\Lambda}_i \hat{V}^\top)-m(\Lambda_i V^\top) \in \mathbb{R}^p$
    \item $R = (R_1, R_2, \cdots, R_p) \in \mathbb{R}^p$ with $R_j = \frac{
    m(\hat{\Lambda}_i^\top V_j) - m(\Lambda_i^\top V_j) }{
    m(\hat{\Lambda}_i^\top \hat{V}_j)-  m(\Lambda_i^\top V_j)}$
\end{itemize}
With $V^\top V = I_q$, the desired inequality becomes:
$$ \|\tilde{G}_{\textcolor{red}{V_p}}( \hat{\Lambda}_{i})\|_2^2  = \frac{1}{p}\|\hat{V}^{\top} \text{Diag}(W_{\hat{V}_{p}}) R\|_2^2\leq \frac{1}{p}\|R\|_2^2 \|\hat{V}^{\top} W_{\hat{V}_{p}}\|_2^2
=\|R\|_2^2\|\tilde{G}_{\textcolor{red}{\hat{V}_p}}( \hat{\Lambda}_{i})\|_2^2
$$
Specifically:
\begin{equation}
    \begin{aligned}
    \|\tilde{G}_{\textcolor{red}{V_p}}( \hat{\Lambda}_{i})\|_2^2 &=  \frac{1}{p}\|(\hat{V}^\top \hat{V})^{-1}  \hat{V}^{\top} W_{V_p}\|_2^2 \\
    & = \frac{1}{p}\|\hat{V}^{\top} W_{V_p}\|_2^2 \\
    &= \frac{1}{p}\|\hat{V}^{\top} \text{Diag}(W_{\hat{V}_{p}}) R\|_2^2\\
    &\leq \frac{1}{p} \|\hat{V}^{\top} \text{Diag}(W_{\hat{V}_{p}}) \|^2_F \| R\|^2_2\\
    & =\frac{1}{p}\|\hat{V}^{\top} W_{\hat{V}_{p}}\|_2^2  \|R\|_2^2 \\
    & = \|R\|_2^2\|\tilde{G}_{\textcolor{red}{\hat{V}_p}}( \hat{\Lambda}_{i})\|_2^2
    \end{aligned} 
\end{equation}
Now combine with the previous inequality result:
\begin{equation}
    \begin{aligned}
     \|\tilde{G}_{\textcolor{black}{\hat{V}_p}}( \hat{\Lambda}_{i})\|_2^2 &\geq \frac{1}{\|R\|_2^2}Z^2_{V_\delta}(\Lambda_{i})\|\Lambda_i -\hat{\Lambda}_i\|_2^2 \\
&\geq (\text{inf}_{j, \|\Lambda_i - \hat{\Lambda}_i|\leq \delta} \frac{m'(\hat{\Lambda}_{i}^\top \tilde{V}_{j})}{\|R\|_2})^2  \|\Lambda_i - \hat{\Lambda}_i\|_2^2
    \end{aligned}
\end{equation}
Then with $\tilde{Z}_{V_{\delta}} = \inf_{j, |\Lambda_i - \hat{\Lambda}_i| \leq \delta} \frac{m'(\hat{\Lambda}_{i}^\top \tilde{V}_{j})}{\|R\|_2} $
$$\|\tilde{G}_{\textcolor{red}{\hat{V}_p}}( \hat{\Lambda}_{i})\|_2^2 \geq \tilde{Z}^2_{V_\delta}(\Lambda_{i})\|\Lambda_i -\hat{\Lambda}_i\|_2^2$$
Denote $\delta_{\Lambda_i}$ as an arbitrary small number, lemma  \ref{le:inverse} gives us:
\begin{equation}
\begin{aligned}
\tilde{G}^{-1}_{\hat{V}_p}(\{b: \|b\|_2 \leq \tilde{Z}_{V_\delta}(\Lambda_{i}) \})
\subseteq \{\Lambda_{i}: \|\hat{\Lambda}_{i} - \Lambda_{i}\|_2 \leq \delta_{\Lambda_{i}} \}
\end{aligned}
\end{equation}
As a result, we can obtain estimator $\hat{\Lambda}_i $ by solving the first
order equation conditional on $\hat{V}_p$. The resulting estimator
$\hat{\Lambda}_i $ satisfies:
$\|\hat{\Lambda}_{i} - \Lambda_{i}\|_2 \leq \delta_{\Lambda_{i}}$, for
arbitrary $\delta_{\Lambda_{i}}>0$.
\end{proof}
\subsubsection*{3). Detailed proof on \textcolor{blue}{Inequality $\diamond$}}
\begin{proof}
    We can now study the convergence of estimator $\hat{V}_j$ by deriving a similar upper bound:
\begin{equation}
\begin{aligned}
\|\tilde{G}_{\Lambda_n}(\hat{V}_j)\|_2 &= \Bigg\| \Bigg(\sum_{i=1}^n
\hat{\Lambda}_{i}{\hat{\Lambda}_{i}}^\top\Bigg)^{-1} \sum_{i=1}^n \hat{\Lambda}_{i}
(m(\Lambda_i^{\top} \hat{V}_j)- m(\Lambda_i^{\top}  V_j)) \Bigg\|^2_2\\
&\leq \Bigg\|\Bigg(\sum_{i=1}^n \hat{\Lambda}_{i}{\hat{\Lambda}_{i}}^\top\Bigg)^{-1}
\sum_{i=1}^n \hat{\Lambda}_{i} (m(\Lambda_i^{\top} \hat{V}_j) - m(\hat{\Lambda}_i^{\top}\hat{V}_j))\Bigg\|^2_2 + \\
&\quad \Bigg\|\Bigg(\sum_{i=1}^n \hat{\Lambda}_{i}{\hat{\Lambda}_{i}}^\top\Bigg)^{-1}
\sum_{i=1}^n \hat{\Lambda}_{i} (m(\hat{\Lambda}_i^{\top}
\hat{V}_j) - m(\Lambda_i^{\top} V_j) ) \Bigg\|^2_2 \\
& =  \Bigg\| \Big( \hat{\Lambda}^\top \hat{\Lambda}\Big)^{-1} \sum_{i=1}^n  \hat{\Lambda}_i m'(\tilde{\eta}_{ij})
(\hat{\Lambda}_i - \Lambda_i) V_j \Bigg\|^2_2
 + \big\| \tilde{G}_{\hat{\Lambda}_n} (\hat{V}_j) \big\|^2_2
\end{aligned}
\end{equation}
The first inequality is an application of Minkowski inequality after adding and
subtracting $m(\Lambda_i^\top V_j)$;
the last equality used the mean value theorem with 
$$\tilde{\eta}_{ij} \in (\min(\Lambda_i^\top \hat{V}_j), \max(\hat{\Lambda}_i^\top \hat{V}_j))$$
Finally, as established previously, $\hat{\Lambda}_i \xrightarrow[p \rightarrow \infty]{a.s} \Lambda_i, \forall i$, we have the first term bounded with small positive scalar $\delta_{\Lambda_i}>0$. Applying the submultiplicative of the spectral norm,  we thus justified that for large $p$:
$$\|\tilde{G}_{\Lambda_n}(\hat{V}_j)\|_2^2 \leq \|\tilde{G}_{\hat{\Lambda}_n}(\hat{V}_j)\|_2^2, \forall j = 1, \cdots, p$$
\end{proof}
\subsection{Proof of Theorem~\ref{thm:family}}
Under conditions F1--F4 and using Taylor expansion, $\forall k = 1, \ldots, G$,
$$
\begin{aligned}
R_{n,p}^1(u_k) &=  \frac{1}{\sqrt{np}}  \sum_{i=1}^n \sum_{j=1}^p 1(\eta_{ij} \leq u_k) [X_{ij} - g^{-1}(\hat{\eta}_{ij})]\\
&= \frac{1}{\sqrt{np}} \sum_{i=1}^n \sum_{j=1}^p 1(\eta_{ij} \leq u_k)[X_{ij}
- g^{-1}(\eta_{ij})] \\
& \qquad + \frac{1}{\sqrt{np}} \sum_{i=1}^n \sum_{j=1}^p 1(\eta_{ij} \leq
u_k)[g^{-1}(\eta_{ij}) - g^{-1}(\hat{\eta}_{ij})]\\
& = \frac{1}{\sqrt{np}}  \sum_{i=1}^n \sum_{j=1}^p 1(\eta_{ij} \leq
u_k)[X_{ij} - g^{-1}(\eta_{ij})] \\
& \quad + \frac{1}{\sqrt{np}}  \sum_{i=1}^n \sum_{j=1}^p 1(\eta_{ij} \leq
u_k)[m'(\hat{\eta}_{ij})(\hat{\eta}_{ij} - \eta_{ij})]
+ O\big((\hat{\eta}_{ij} - \eta_{ij})^2\big).
\end{aligned}
$$
Now note that
\begin{enumerate}
\item Under conditions~C1 -~C4, after convergence of
Algorithm~\ref{algo:dmf}, the consistency result of
Theorem~\ref{thm:consistency} guarantees that the second term is absorbed
for large $n$ and $p$.

\item Under $H_0$,
$T_k = \{1(\eta_{ij} \leq u_k)[X_{ij} - m(\eta_{ij})]\}_{k=2}^G$ has zero mean
and variance given by
$$
\begin{aligned}
\sigma^2_{T_k} &= \text{Var}[R_{n,p}(u_k)] + \text{Var}[R_{n,p}(u_{k-1})] - 2
\text{Cov}[R_{n,p}(u_k), R_{n,p}(u_{k-1})] \\
&= \sum_{l=k-1}^k \Exp\big\{1(\eta_{ij} \leq u_l)
[X_{ij} - m(\eta_{ij})]^2 \big\} \\
& \qquad - 2 \Exp\big\{1(\eta_{ij} \leq u_k)1(\eta_{ij} \leq u_{k-1})
[X_{ij} - m(\eta_{ij})]^2 \big\} \\
&= \phi_{ij} \int_{-\infty}^{u_k}\mathcal{V}(X_{ij} | \eta_{ij} = t) F(dt) - \phi_{ij} \int_{-\infty}^{u_{k-1}}\mathcal{V}(X_{ij} | \eta_{ij} = t) F(dt)\\
& = \phi_{ij} \int_{u_{k-1}}^{u_{k}} \mathcal{V}(X_{ij} | \eta_{ij} = t) F(dt)\\
\end{aligned}
$$

\item Under condition~F3, we have from the SLLN that element-wisely
$$D^1_{n,p} \xrightarrow[n,p\rightarrow \infty]{a.s} D_{n,p}$$

\item Under condition F2, applying Slutsky's theorem, we have from the
Lindeberg CLT that
$$\hat{T}(n,p) = (T_{n,p}^1)^\top (D^1_{n,p})^{-1} T_{n,p}^1 \xrightarrow[n,p \rightarrow \infty]{d} \chi^2_{G-1}.$$
\end{enumerate}
}
\subsection{Proof of Proposition~\ref{prop:rank}}
\label{appendix:rankproof}
\begin{proof}
\noindent Under our DMF setup, we assume $X_{ij} \sim F(\mu_{ij} = m(\eta_{ij}))$
where $m(\cdot)$ is the inverse link $g^{-1}(\cdot)$ and $F$ is a member of the exponential
family that specifies the relationship between $m(\eta_{ij})$ and
$\mathcal{V}(m(\eta_{ij}))$. From Theorem~\ref{thm:consistency}, we defined
$e_{ij} = X_{ij} - m(\eta_{ij})$ with the following properties:
\begin{itemize}
    \item $\eta_{ij}= \Lambda_i^\top V_j$ is deterministic;
    \item $e_{ij}$ is random with $\Exp[e_{ij}] = 0$ and
    $\text{VAR}(e_{ij}) = \mathcal{V} (m(\eta_{ij}))$.
\end{itemize}
\noindent Based upon a full rank DMF fit $q = p$,  we can obtain estimator $\hat{\eta} = \hat{\Lambda} \hat{V}$.
From the F.O.C, we can have:
\begin{equation}
\begin{aligned}
\frac{\partial l(\Lambda, V)}{\partial V_j} &= \sum_{i=1}^n (X_{ij} -
m(\hat{\Lambda}_i^\top \hat{V}_j)) \hat{\Lambda}_i^\top= \mathbf{0}_q \\
\frac{\partial l(\Lambda, V)}{\partial \Lambda_i} &= \sum_{j=1}^p (X_{ij} -
m(\hat{\Lambda}_i^\top \hat{V}_j))\hat{V}_j^\top= \mathbf{0}_q
\end{aligned}
\end{equation}
The $\Lambda_i$ F.O.C is:
\begin{equation}
\begin{aligned}
\mathbf{0}_q&=\sum_{j=1}^p (X_{ij} -
m(\hat{\Lambda}_i^\top \hat{V}_j)) \hat{V}_j^\top  \\
&=\sum_{j=1}^p (X_{ij} -
m(\hat{\Lambda}_i^\top \hat{V}_j)
+m(\Lambda_i^\top V_j) - m(\Lambda_i^\top V_j)) \hat{V}_j^\top \\
&= \sum_{j=1}^p ( \underbrace{X_{ij} -m(\Lambda_i^\top V_j)}_{e_{ij}}
+m(\Lambda_i^\top V_j) - m(\hat{\Lambda}_i^\top \hat{V}_j)) \hat{V}_j^\top \\
&= \sum_{j=1}^p \big( e_{ij} - m'(\eta_{ij}) (\hat{\eta}_{ij} -\eta_{ij})  - O(\frac{1}{2}m''(\eta_{ij}) (\hat{\eta}_{ij}-\eta_{ij})^2) \big) \hat{V}_j^{\top}
\label{eq:rank_foc}
\end{aligned}
\end{equation}
where the fourth equality used Taylor expansion:
$$m(\hat{\Lambda}_i^\top \hat{V}_j) \equiv  m(\hat{\eta}_{ij}) = m(\eta_{ij}) + m'(\eta_{ij})(\hat{\eta}_{ij} - \eta_{ij}) + O(\frac{1}{2}m''(\eta_{ij}) (\hat{\eta}_{ij} - \eta_{ij})^2)$$
Now keep $\sum_{j=1}^p \hat{V}_j^\top e_{ij} = \hat{V}^\top e_i$ to the right side of the equation and denote $a_{ij} = m'(\tilde{\eta}_{ij}) (\eta_{ij} -\hat{\eta}_{ij}) + O(\frac{1}{2}m''(\eta_{ij})(\hat{\eta}_{ij} - \eta_{ij})^2)$, Eq~\eqref{eq:rank_foc} reduces to 
\begin{equation}
\begin{aligned}
    \sum_{j=1}^p  a_{ij} \hat{V}_j^{\top}= \hat{V}^\top a_i  &= \hat{V}^{\top} e_i \\
\end{aligned}
\end{equation}
Since $\hat{V}$ is a $p\times p $ matrix full rank with $q=p$, we can multiply $\hat{V}^{-\top}$ on both sides and obtain:
\begin{equation*}
    \begin{aligned}
        \hat{V}^{-\top} \hat{V}^\top a_i &= a_i = \hat{V}^{-\top} \hat{V}^\top e_i = e_i\\
    \end{aligned}
\end{equation*}
$a_i, e_i$ are vector of $\mathbb{R}^p$ and we have element-wisely $\forall i = 1,\ldots,n, \forall j= 1,\ldots, p$:
\begin{equation}
    \begin{aligned}
     &m'(\eta_{ij}) (\hat{\eta}_{ij} -\eta_{ij}) + O(\frac{1}{2}m''(\eta_{ij})(\hat{\eta}_{ij} - \eta_{ij})^2) =a_{ij} =  e_{ij}\\
     \Rightarrow & \hat{\eta}_{ij} = \eta_{ij} + \frac{1}{m'(\eta_{ij})} e_{ij} + O(\frac{1}{2} \frac{m''(\eta_{ij})}{m'(\eta_{ij})}(\hat{\eta}_{ij} - \eta_{ij})^2) 
    \end{aligned}
    \label{eq:rank_eijorder}
\end{equation}
In a rank $p$ fit, we have saturated estimation and thus $m(\hat{\eta}_{ij}) = X_{ij}$, which implies $g \circ m(\hat{\eta}_{ij}) = \hat{\eta}_{ij}=  g(X_{ij})$. Taylor expanding $g(X_{ij})$ at $\mu_{ij}$, we have:
\begin{equation*}
    \begin{aligned}
            \hat{\eta}_{ij} = g(X_{ij}) &= g(\mu_{ij}) + O(g'(\mu_{ij})(X_{ij} - \mu_{ij}))\\
            & = \eta_{ij} + O(g'(\mu_{ij})(X_{ij} - \mu_{ij}))
    \end{aligned}
\end{equation*}
Use the fact that $\frac{1}{m'(\eta_{ij})} = g'(\mu_{ij})$, we can simplify the order term as:
$$\hat{\eta}_{ij}= \eta_{ij} + \frac{1}{m'(\eta_{ij})} e_{ij} + O(m''(\eta_{ij}) g'(\mu_{ij})e_{ij}^2)$$
Now observe that
\begin{itemize}
    \item From condition~C4 and Cauchy inequality, we have $|\eta_{ij}|\leq \|\Lambda_i\|_2 \|V_j^\top\|_2< \infty$.
    \item The function  $m''(\eta_{ij}) (g' \circ m(\eta_{ij}))$ is finite by the continuity of $m''(\cdot)$ and $g' \circ m(\cdot)$ plus the finiteness of $\eta_{ij}$.
    \item $e_{ij}^2$ is random and converges to 0 a.s. In fact due to the square operation, it converges to 0 faster than $e_{ij}$ for large $n,p$.
\end{itemize}
When we have large $n, p$ such that condition C2 implies the convergence result, we directly have $e_{ij} \xrightarrow[n,p\rightarrow \infty]{a.s.} 0$.
When we have moderate $n, p$,  we should expect the order term $m''(\eta_{ij}) g'(\mu_{ij}) e_{ij}^2$ to converge to 0 faster than $e_{ij}$, which leaves us with:
\begin{equation}
\hat{\eta}_{ij}= \eta_{ij} +  \frac{1}{m'(\eta_{ij})} e_{ij}
= \eta_{ij} + \frac{\sqrt{\mathcal{V} (\mu_{ij})}}{m'(\eta_{ij})} \epsilon_{ij}
\label{eq:rankresid}
\end{equation}
where $\Exp[\epsilon_{ij}] = 0$, $\Exp[\epsilon_{ij}^2] = 1$ and
$\Exp[\epsilon_{ij}^4] < \infty$ by condition~R2.

To adapt \citep{2010onrank}'s rank theorem for the case of moderate $n,p$, we
thus only need to justify the existence of $A$ and $B$ such that
$A_i \epsilon B_j = (\sqrt{\mathcal{V}(\mu_{ij})}/m'(\eta_{ij})) \epsilon_{ij}$
where $A, B$ and $\epsilon$ satisfy the setup and Assumption~2 of~\citep{2010onrank}.
Now observe that condition (i) in~\citep{2010onrank} is essentially our
condition~R1 with the definition of the Pearson residual $\epsilon_{ij}$;
conditions~(ii) and (iii) in~\citep{2010onrank} are extremely mild in the sense
that condition (ii) requires $B^\top B$ (as well as $AA^{\top}$) to be the
auto-covariance matrix of any covariance-stationary process with a bounded
spectral density; condition (iii) demands the prevention of a few linear
combinations of idiosyncratic terms to have unusually large variation.

The existence of $A,B, \epsilon$ satisfying the above assumption can then be
justified as follows. Let $h$ be a variance stabilizing link function, that is,
$h(\mu) \propto \int^\mu dt / \sqrt{\mathcal{V}(t)}$, e.g.
$h(\mu) = \arcsin\sqrt{\mu}$ for binomial data and
$h(\mu) = \sqrt{\mu}$ for Poisson data. If $h$ is the true link, then
$h(\mu_{ij}) = \eta_{ij}$ and
$m'(\eta_{ij}) = \sqrt{\mathcal{V}(\mu_{ij})}$, leaving $\epsilon_{ij} =
e_{ij}$ and thus the assumptions are fully satisfied with $A = I_n$ and
$B = I_p$.
Under a general link, similar to \citep{2010onrank}, we assume that the
residual covariance
$\tilde{e}_{ij} =  (\sqrt{\mathcal{V}(\mu_{ij})}/m'(\eta_{ij})) \epsilon_{ij}$
can be approximated with either $A$ or $B$ being diagonal and the other one
being unrestricted such that
\[
\Exp[\text{vec}\ \tilde{e} (\text{vec}\  \tilde{e})^{\top} ] =
\text{vec}\ \frac{\sqrt{V(\mu)}}{m'(\eta)}
\Big(\text{vec}\  \frac{\sqrt{V(\mu)}}{m'(\eta)}\Big)^{\top} \approx B^\top B \otimes AA^{\top}
\]
The quality of the approximation depends on how close the general link is to a
variance stabilizing link; schemes for this Kronecker approximation are
discussed in~\citep{van1993approximation}. For our rank determination
simulation described in Section \ref{sec:studies}, we empirically verified
that the application of eigengap theorem performs reasonably well.
\end{proof}

\fi

\bibliographystyle{abbrv} 
\bibliography{dmf}
\end{document}